\documentclass{article}
\usepackage[utf8]{inputenc}
\usepackage{geometry}
\usepackage{natbib}
\newif\ifarxiv
\arxivtrue
\usepackage{commands}
\renewcommand\cite{\citep}

\newcommand*\samethanks[1][\value{footnote}]{\footnotemark[#1]}
\newcommand{\rows}{R}
\newcommand{\cols}{C}
\newcommand{\pads}{P}
\newcommand{\pmean}{\mu}
\newcommand{\emean}{\hat{\mu}}
\newcommand{\omean}{\bar \mu}
\newif\ifdraft
\drafttrue

\ifdraft
\newcommand{\cc}[1]{\textcolor{cyan}{CC: #1}}

\else
\newcommand{\cc}[1]{}

\fi

\usepackage{verbatim}
\usepackage{fixltx2e}
\usepackage{fancyvrb}
\RecustomVerbatimCommand{\VerbatimInput}{VerbatimInput}%
{fontsize=\footnotesize,
 framesep=2em, 
 rulecolor=\color{gray},
 commandchars=\|\(\), 
 commentchar=*        
}

\ifarxiv
\else
\Crefname{theorem}{Thm.}{Theorems}
\Crefname{thm}{Thm.}{Theorems}
\Crefname{cor}{Cor.}{Corollaries}
\Crefname{lem}{Lem.}{Lemmas}
\Crefname{prop}{Prop.}{Propositions}
\Crefname{assumption}{Assumption}{Assumptions}
\Crefname{definition}{Defn.}{Definitions}
\Crefname{claim}{Claim}{Claims}
\Crefname{section}{Sec.}{Sections}
\Crefname{appendix}{App.}{Appendices}
\Crefname{algorithm}{Alg.}{Algorithms}
\Crefname{equation}{Eq.}{Equations}
\fi

\title{Private Federated Learning with Autotuned Compression}
\author{
\makebox[1in]{ \hfill Enayat Ullah \thanks{The Johns Hopkins University.  Work completed while on internship at Google, \href{mailto:enayat@jhu.edu}{enayat@jhu.edu}}
\thanks{Equal contribution}
}
\makebox[2.8in]{ \hfill Christopher A. Choquette-Choo \thanks{Google Research, Brain team; Now at Google DeepMind \href{mailto:cchoquette@google.com}{cchoquette@google.com.}}
\samethanks[2]
}
\and
\makebox[1.2in]{ \hfill Peter Kairouz \thanks{Google Research, \href{mailto:kairouz@google.com}{kairouz@google.com}}
}
\makebox[1.2in]{ \hfill Sewoong Oh \thanks{University of Washington and Google Research, \href{mailto:sewoongo@google.com}{sewoongo@google.com}}}
}

\begin{document}
\date{}
\maketitle
\begin{abstract}
We propose new techniques for reducing communication in private federated learning without the need for setting or tuning compression rates. Our on-the-fly methods automatically adjust the compression rate based on the error induced during training, while maintaining provable privacy guarantees through the use of secure aggregation and differential privacy. Our techniques are provably instance-optimal for mean estimation, meaning that they can adapt to the ``hardness of the problem'' with minimal interactivity. We demonstrate the effectiveness of our approach on real-world datasets by achieving favorable compression rates without the need for tuning.

\end{abstract}
\section{Introduction}
\label{sec:intro}

Federated Learning (FL) is a form of distributed learning whereby a shared global model is trained collaboratively by many clients under the coordination of a central service provider. Often, clients are entities like mobile devices which may contain sensitive or personal user data. FL has a favorable construction for privacy-preserving machine learning, since user data never leaves the device. Building on top of this can provide strong trust models with rigorous user-level differential privacy (DP) guarantees~\cite{dwork2010differential}, which has been studied extensively in the literature~\cite{dwork2010pan, mcmahan2017learning, mcmahan2022private,kairouz2021practical}.

More recently, it has become evident that secure aggregation (SecAgg) techniques~\cite{bonawitz2016practical, bell2020secure} are required to prevent honest-but-curious servers from breaching user privacy~\cite{fowl2022robbing,hatamizadeh2022gradient,suliman2022two}.
Indeed, SecAgg and DP give a strong trust model for privacy-preserving FL~\cite{kairouz2021distributed,chen2022fundamental, agarwal2021skellam, chen2022poisson,xu2023federated}. However, SecAgg can introduce significant communication burdens, especially in the large cohort setting which is preferred for DP. In the extreme, this can significantly limit the scalability of DP-FL with SecAgg. This motivated the study of privacy-utility-communication tradeoffs by~\citet{chen2022fundamental}, where they found that significant communication reductions could be attained essentially ``for free'' (see Figure \ref{fig:compression-differ-tasks}) by using a variant of the Count Sketch linear data structure.

However, the approach of~\citet{chen2022fundamental}
suffers from a major drawback: it is unclear how to set the sketch size (i.e., compression rate) a priori. Indeed, using a small sketch may lead to a steep degradation in 
{utility} whereas a large sketch may hurt bandwidth. Further, as demonstrated in Figure \ref{fig:compression-differ-tasks}, the ``optimal'' compression rate can differ substantially for different datasets, tasks, model architectures, and noise multipliers (altogether, DP configurations).   
This motivates the following question which we tackle in this paper: 
\begin{center}
    \textit{Can we devise adaptive (autotuning) compression techniques which can find favorable compression rates in an instance-specific manner?}
\end{center}

\begin{figure}
    \centering
    \vspace{-1em}
    \includegraphics[width=0.5\linewidth]{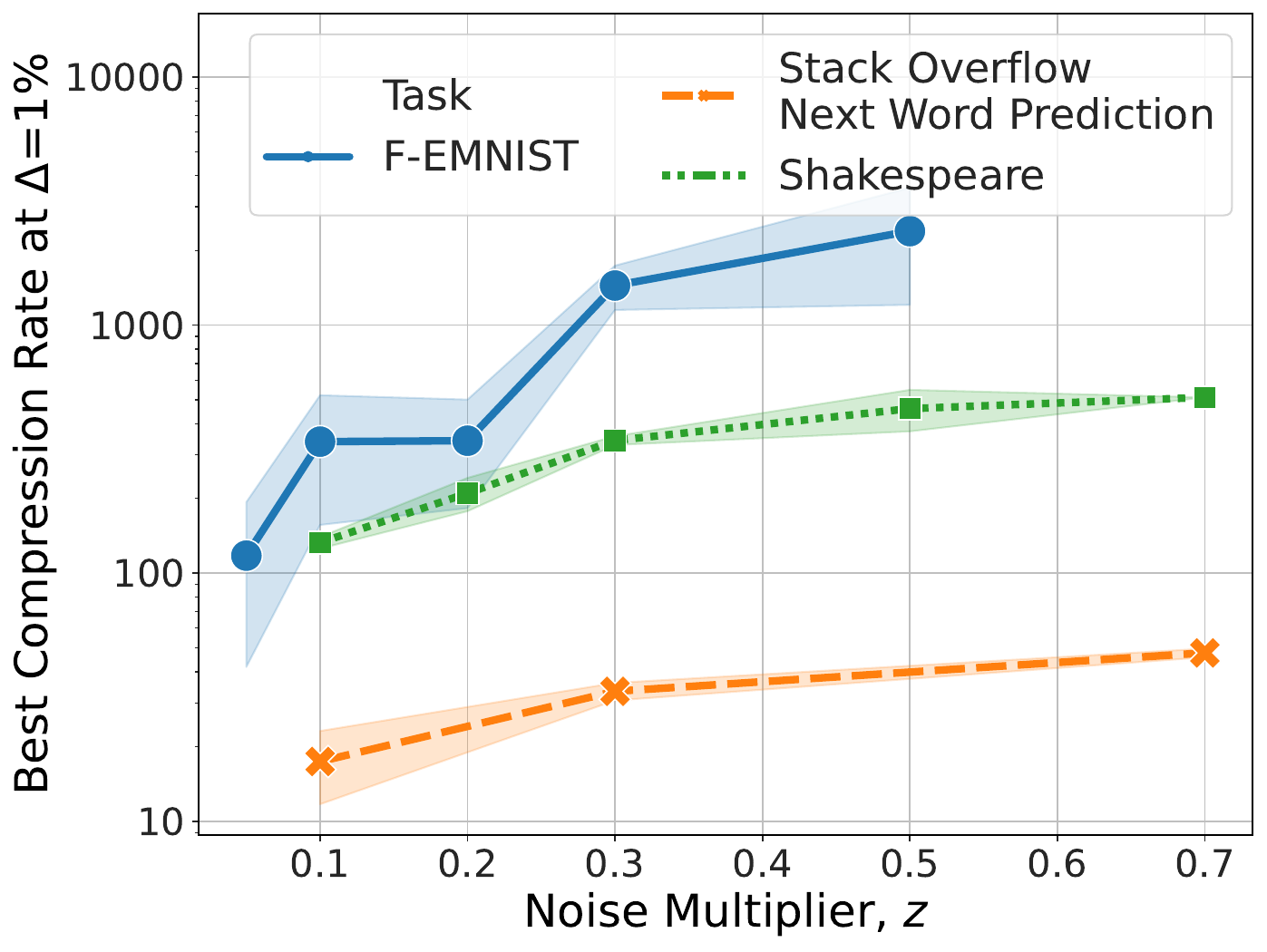}
    \vspace{-1em}
    \caption{\textbf{The optimal compression rates differ highly between different tasks.} 
    For each noise multiplier, we run compression rates of $2^i$, for $i\in[1,13]$, and report the highest compression rate with a maximum of $\Delta=1\%$ relative drop in accuracy from the same model without compression.
    We interpolate accuracies for fine-grained results. See \cref{sec:experiments} for task descriptions.}
    \label{fig:compression-differ-tasks}
\end{figure}

\ifarxiv
\paragraph{Naive approaches to tuning the compression rate.}
\else
{\bf Naive approaches to tuning the compression rate.} 
\fi
Before introducing our proposed approaches, we discuss the challenges of using approaches based on standard hyper-parameter tuning. The two most common approaches are:

\ifarxiv
\begin{enumerate}
    \item \textbf{Grid-search {(``Genie'')}}: Here, we a priori fix a set of compression rates, try all of them and choose the one which has the \textit{best} compression-utility (accuracy) trade-off. The problem is that this requires trying 
    {most or all}
    compression rates, and {would likely lead to even more communication in totality.}
    \item \textbf{Doubling-trick}: 
    This involves starting with an initial \textit{optimistic} guess for the compression rate, using it, and evaluating the 
    {utility}
    for this choice. The process is then repeated, with the compression rate halved each time, until a predetermined \textit{stopping criterion} is met. While this method  ensures that the required communication is no more than twice the \textit{optimal} amount, it can be difficult to determine a \textit{principled} stopping criterion (see \cref{fig:compression-differ-tasks} where we show that optimal compression rates depend on the task and model used). It may be desirable to maintain a high level of 
    {utility}
    close to that of the uncompressed approach, but 
    this requires knowledge of the 
    {utility}
    of the uncompressed method, which is unavailable.
\end{enumerate}
\else{
\begin{CompactEnumerate}
    \item \textbf{Grid-search {(``Genie'')}}: Here, we a priori fix a set of compression rates, try all of them and choose the one which has the \textit{best} compression-utility (accuracy) trade-off. The problem is that this requires trying 
    {most or all}
    compression rates, and {would likely lead to even more communication in totality.}
    \item \textbf{Doubling-trick}: 
    This involves starting with an initial \textit{optimistic} guess for the compression rate, using it, and evaluating the 
    {utility}
    for this choice. The process is then repeated, with the compression rate halved each time, until a predetermined \textit{stopping criterion} is met. While this method  ensures that the required communication is no more than twice the \textit{optimal} amount, it can be difficult to determine a \textit{principled} stopping criterion (see \cref{fig:compression-differ-tasks} where we show that optimal compression rates depend on the task and model used). It may be desirable to maintain a high level of 
    {utility}
    close to that of the uncompressed approach, but 
    this requires knowledge of the 
    {utility}
    of the uncompressed method, which is unavailable.
\end{CompactEnumerate}
}
\fi

In our approach, instead of treating the optimization procedure as a closed-box that takes the compression rate as input, we open it up to identify the core component, \textit{mean estimation}, which is more amenable to adaptive tuning.

\ifarxiv
\paragraph{The proxy of mean estimation.}
\else
{\bf The proxy of mean estimation.}
\fi
A core component of first-order optimization methods 
is estimating the mean of gradients  at every step.
From an algorithmic perspective, 
this \textit{mean-estimation view} has been fruitful in incorporating
\textit{modern constraints}, such as privacy \cite{bassily2014private,abadi2016deep,choquette2022multi}, robustness \cite{diakonikolas2019sever,prasad2018robust} and compression \cite{alistarh2017qsgd,chen2022fundamental,suresh2022correlated}, by simply designing appropriate mean estimation procedures, while reusing the rest of the optimization method as is. 
In
FL, the popular Federated Averaging (FedAvg) algorithm \cite{mcmahan2017communication}, which we will use, computes a mean of client updates at every round.

\ifarxiv
\paragraph{Our theoretical results.}
\else
\textbf{Our theoretical results.}
\fi
Previously, \cite{chen2022fundamental} studied federated mean estimation  under the constraints of SecAgg and $(\epsilon,\delta)$-DP,
and showed that with $n$ clients (data points) in $d$ dimensions,
 the optimal (worst-case) communication cost is $\min\br{d,n^2\epsilon^2}$ bits such that the mean squared error is of the same order as optimal error under (central) DP, without any compression or SecAgg.
However, in practice, the instances are rarely worst-case, and one would desire a method with more \textit{instance-specific} performance.
Towards this, we introduce \textbf{two} fine-grained descriptions of instance classes, parameterized by {the} (a) \textbf{norm of mean} (\cref{sec:adaptivity_norm}), which is motivated by the fact that in FL, the mean vector is the average per-round gradient and from classical optimization wisdom, its norm goes down as training progresses, and (b) \textbf{tail-norm of mean} (\cref{sec:adaptivity_tail_norm}), which captures the setting where the mean is \textit{approximately sparse}, motivated by empirical observations {of this phenomena} on gradients in deep learning \cite{micikevicius2017mixed,shi2019understanding}.

For both of these settings, we design \textit{adaptive} procedures which are oblivious to the values of norm and tail-norm, yet yield 
    performance competitive to a method with complete knowledge of these. Specifically, for the norm of mean setting,
    our proposed procedure has a  per-client communication complexity of roughly $\tilde O(n^2\epsilon^2 M^2)$ bits where $M\leq 1$ is the ratio of norm of mean to the worst-case norm bound on data points. For the tail-norm setting, we get an improved communication complexity given by the \textit{generalized sparsity} of the mean
-- see \cref{sec:adaptivity_tail_norm} for details. Further, for both  settings, we show that our proposed procedures achieve
(a) optimal error under DP, without communication constraints, and
(b) optimal communication complexity, under the SecAgg constraint,  up to poly-logarithmic factors.

We note that adaptivity to tail-norm implies adaptivity to norm, but
this comes at a price of $\log{d}$, rather than \textit{two}, rounds of communication, which may be less favorable, especially in FL settings. We also show that interaction is necessary for achieving (nearly) optimal communication complexity, adaptively.

Finally, even without the need for adaptivity, e.g. in centralized settings, as by-products, our results yield optimal rates for DP mean estimation for the above fine-grained instance classes, parameterized by norm or tail-norm of mean, which could be of independent interest.

\ifarxiv
\paragraph{Our techniques.}
\else
\textbf{Our techniques.} 
\fi
Our compression technique, \textit{count-median-of-means} sketching, is based on linear sketching, and generalizes the count-mean sketch used in \cite{chen2022fundamental}.
Our proposed protocol for \textit{federated mean estimation} (FME) comprises of multiple rounds of communication in which each participating client sends \textbf{two} (as opposed to one) sketches.
The first sketch is used to estimate the mean, as in prior work,
whereas the second sketch, which is much smaller, is used to track certain statistics for adaptivity. The unifying guiding principle behind all the proposed methods is to set compression rate such that the total compression error does not overwhelm the DP error. 

\ifarxiv
\paragraph{Experimental evaluation.}
\else
{\bf Experimental evaluation.}
\fi
We map our mean estimation technique to the FedAvg algorithm and test it on three standard FL benchmark tasks: character/digit recognition task on the F-EMNIST dataset and next word prediction on Shakespeare and Stackoverflow datasets (see \cref{sec:experiments} for details).
{
We find that our proposed technique can obtain \textit{favorable} compression rates without tuning. In particular, we find that our one-shot approach tracks the potentially unachievable \textbf{Genie} baseline (shown in Figure \ref{fig:compression-differ-tasks}), with no harm to model utility beyond a small slack ($\Delta\approx1\%$) which we allow for any compression method, including the Genie.
Our code is at: \url{https://github.com/google-research/federated/tree/master/private_adaptive_linear_compression}.
}

\ifarxiv
\paragraph{Related work.} 
\else
{\bf Related work.}
\fi
Mean estimation is one of most widely studied problems in the DP literature.
A standard setting deals with bounded data points \cite{steinke2015between, kamath2020primer}, which is what we focus on in our work. However, mean estimation has also been studied under probabilistic assumptions on the data generating process \cite{karwa2017finite,bun2019average,kamath2020private,biswas2020coinpress,liu2021robust,liu2022differential}.
The work most related to ours is that of \cite{chen2022fundamental}, which showed that in the worst-case, $O(\min(d,n^2\epsilon^2))$ bits of per-client communication is sufficient and necessary for achieving optimal error rate for SecAgg-compatible distributed DP mean estimation. Besides this, \cite{agarwal2018cpsgd,agarwal2021skellam,kairouz2021distributed} also study compression under DP in FL settings, but rely on quantization and thus incur $\Omega(d)$ per-client communication.
Finally, many works, such as \cite{feldman2021lossless,chen2020breaking, asi2022optimal,duchi2018minimax, girgis2021shuffled}, study mean estimation, with and without compression, under local DP \cite{warner1965randomized, kasiviswanathan2011can}. However, we focus on SecAgg-compatible (distributed) DP. 

There has been a significant amount of research on optimization with compression in distributed and federated settings. 
The most common compression techniques are quantization \cite{alistarh2017qsgd,wen2017terngrad} and 
{sparsification}
\cite{aji2017sparse,stich2018sparsified}.
Moreover, the works of \cite{makarenko2022adaptive,jhunjhunwala2021adaptive, chen2018lag} consider adaptive compression wherein the compression rate is adjusted across rounds.
However, the compression techniques used in the aforementioned works
are non-linear, and thus are not SecAgg compatible.
The compression techniques most related to ours  are of \cite{ivkin2019communication,rothchild2020fetchsgd,haddadpour2020fedsketch} using linear sketching.
However, they do not consider adaptive compression or privacy. Finally, the works of \cite{arora2022differentially,arora2022faster,bu2021fast} use random projections akin to sketching in DP optimization, but in a centralized setting, for improved utility or run-time.

\ifarxiv
\paragraph{Organization.}
\else
\textbf{Organization.}
\fi
We consider two tasks, Federated Mean Estimation (FME) and Federated Optimization (FO). The former primarily serves as a subroutine for the latter by considering the vector $z$ at a client to be the model update at that round of FedAvg. 
In \cref{sec:fme}, we propose two algorithms for
FME
: \textbf{Adapt Norm}, in \cref{alg:dp-mean-estimation-adapt-server_adapt_mean} (with the formal claim of adaptivity in \cref{thm:dp-mean-estimation-upper-bound-1}) and \textbf{Adapt Tail}, in  \cref{alg:dp-mean-estimation-adapt-server_adapt_tail} (with the formal claim of adaptivity in \cref{thm:dp-mean-estimation-upper-bound-2}), which adapt to the norm of the mean and the tail-norm of the mean, respectively.
In \cref{sec:fme-to-fo-norm}, we show how to extend our
Adapt Norm FME protocol to the FO setting in 
\cref{alg:dp-mean-estimation-adapt-server_adapt_experiments}.
Finally, in \cref{sec:experiments}, we evaluate the performance of the Adapt Norm approach
for FO on benchmark tasks in FL.

\section{Preliminaries}
\label{sec:prelims}
\begin{definition}[$(\epsilon, \delta)$-Differential Privacy]
An algorithm $\cA$ satisfies $(\epsilon,\delta)$-differential privacy if for all datasets $D$ and $D'$ differing in one data point and all events $\cE$ in the range of the $\cA$, we have, $\mathbb{P}\br{\cA(D)\in \cE} \leq     e^\varepsilon \mathbb{P}\br{\cA(D')\in \cE}  +\delta$. 
\end{definition}

\ifarxiv
\paragraph{Secure Aggregation (SecAgg).}
\else
{\bf Secure Aggregation (SecAgg).}
\fi
SecAgg is a cryptographic technique that allows multiple parties to compute an aggregate value, such as a sum or average, without revealing their individual contributions to the computation.
In the context of FL, the works of \cite{bonawitz2016practical,bell2020secure} proposed practical SecAgg schemes. We assume SecAgg as default, as is the case in typical FL systems.

\ifarxiv
\paragraph{Count-mean sketching.}
\else
{\bf Count-mean sketching.}
\fi
We  describe the compression technique of \cite{chen2022fundamental}, which is based on the sparse Johnson-Lindenstrauss (JL) random matrix/count-mean sketch data structure \cite{kane2014sparser}. The sketching operation is a linear map, denoted as $S:\bbR^d\rightarrow\bbR^{PC}$, where $P,C\in \bbN$ are parameters. The corresponding \textit{unsketching}  operation is denoted as $U:\bbR^{PC}\rightarrow \bbR^d$.
To explain the sketching operation, we begin by introducing the count-sketch data structure.
A count-sketch is a linear map, which for $j\in [P]$, is denoted as
$S_j:\bbR^{d} \rightarrow \bbR^{C}$. It is described using two hash functions: bucketing hash $h_j:[d]\rightarrow [C]$ and sign hash: $s_j:[d]\rightarrow \bc{-1,1}$, mapping the $q$-th co-ordinate $z_q$ to $\sum_{i=1}^d s_j(i) \mathbbm{1}\br{h_j(i)=h_j(q)}z_i$.

The 
count-mean sketch construction pads $P$ count sketches to get $S: \bbR^d \rightarrow \bbR^{PC}$, mapping $z$ as follows,
\begin{align*}
    S(z) = ({1}/{\sqrt{P}})
    \begin{bmatrix}
{S_{1}}^\top
    {S_{2}}^\top &
    \cdots &
    {S_{P}}^\top 
    \end{bmatrix}^\top z\;. 
\end{align*}
The above, being a JL matrix, approximately preserves norms of $z$ i.e. $\norm{S(z)}\approx \norm{z}$, which is useful in controlling sensitivity, thus enabling application of DP techniques. The unsketching operation is simply $U(S(z)) = S^\top S(z)$. This gives an unbiased estimate, ${\mathbb E}[S^TS(z)]=z$, whose variance scales as $d\|z\|^2/(PC)$. This captures the trade-off between compression rate, $d/PC$, and error.
\section{Instance-Optimal Federated Mean Estimation (FME)}
\label{sec:fme}
A central operation in standard federated learning (FL) algorithms is averaging the client model updates in a distributed manner. 
This can be posed as a standard distributed mean estimation (DME) problem with $n$ clients, each with a vector $z_i \in \bbR^d$ sampled i.i.d. from an unknown distribution $\cD$ with population mean $\mu$. The goal of the server is to estimate $\mu$ while communicating only a small number of bits with the clients at each round. Once we have a communication efficient scheme, this can be readily integrated into the learning algorithm of choice, such as FedAvg. 

 In order to provide privacy and security of the clients' data, mean estimation for FL has additional requirements: we can only access the clients via SecAgg~\cite{bonawitz2016practical} and we need to satisfy DP~\cite{dwork2014algorithmic}. We refer to this problem as {\em Federated Mean Estimation} (FME).  To bound the sensitivity of the empirical mean, $\hat \mu(\{z_i\}_{i=1}^n) :=(1/n)\sum_{i=1}^n z_i$, the data is assumed to be bounded by  $\norm{z_i}\leq G$.  Since gradient norm clipping is 
 {almost always used in the DP}
 FL setting, we assume $G$ is known. The first result characterizing the communication cost of FME is by \citet{chen2022fundamental}, who propose an unbiased estimator based on count-mean sketching that satisfy $(\epsilon,\delta)$-differential privacy and (order) optimal error of 
\begin{align}
    \label{eqn:mean-estimation-optimal-rate}
     \mathbb{E}[\norm{\overline \mu - \mu}^2]\; =\; \tilde \Theta \br{G^2\br{\frac{d}{n^2\epsilon^2}+\frac{1}{n}}}, 
\end{align}
with an (order) optimal communication complexity of $\tilde O(n^2\epsilon^2)$. 
The error rate in \cref{eqn:mean-estimation-optimal-rate} is optimal as it matches the information theoretic lower bound of \cite{steinke2015between} that holds even without any communication constraints. 
The communication cost of $\tilde O(n^2\epsilon^2)$ cannot be improved for the worst-case data as it matches the lower bound in \cite{chen2022fundamental}. However, it might be possible to improve  on communication for specific instances of $\mu$.
We thus ask the fundamental question of how much communication is necessary and sufficient to achieve the optimal error rate as a function of the ``hardness'' of the problem.

\subsection{Adapting to the Instance's Norm}
\label{sec:adaptivity_norm}
Our more fine-grained error analysis of the scheme in \cite{chen2022fundamental} shows that, with a sketch size of $O(PC)$ that costs $O(PC)$ in per client communication, one can achieve 
\ifarxiv
\begin{align}
    \mathbb{E}[\norm{\overline \mu - \mu}^2] =    \tilde O\Big(G^2\,\Big( \frac{d(M^2 + 1/n)}{PC} + 
    \frac{d}{n^2\epsilon^2} +\frac{1}{n} \Big) \Big)\,, 
    \label{eqn:norm_adptive_bound}
\end{align}
\else
{\small
\begin{align}
    \mathbb{E}[\norm{\overline \mu - \mu}^2] =    \tilde O\Big(G^2\,\Big( \frac{d(M^2 + 1/n)}{PC} + 
    \frac{d}{n^2\epsilon^2} +\frac{1}{n} \Big) \Big)\,, 
    \label{eqn:norm_adptive_bound}
\end{align}
}
\fi
where we define the normalized  norm of the mean,
$M\;\;=\;\;\frac{\norm{\mu}}{G}=\;\;\frac{\norm{\mathbb E \left[z\right]}}{G}\,\in\,[0,1]\;
$ and $G$ is the maximum $\lVert z \rVert$ as chosen for a sensitivity bound under DP FL.
The first error term captures how the sketching error scales as the norm of the mean, $MG$. 
When $M$ is significantly smaller than one, as motivated by our application to FL in \cref{sec:intro}, a significantly smaller choice of the communication cost,  $PC=O(\min\br{n^2\epsilon^2,nd}\br{M^2 + 1/n})$, is sufficient to achieve the optimal error rate of \cref{eqn:mean-estimation-optimal-rate}.
The dominant term is smaller than the standard sketch size of $PC=O(n^2\epsilon^2)$ by a factor of $M^2 \in[0,1]$. However, selecting this sketch size requires knowledge of $M$. This necessitates a scheme that adapts to the current instance by privately estimating $M$ with a small communication cost. This leads to the design of our proposed {\em interactive} 
\cref{alg:dp-mean-estimation-adapt-server_adapt_mean}.  

Precisely, we call an FME algorithm {\em instance-optimal with respect to $M$} if it achieves the optimal error rate of \cref{eqn:mean-estimation-optimal-rate} with communication complexity of $O(M^2 n^2\epsilon^2)$ for every instance whose norm of the mean is bounded by $GM$.
We propose a novel adaptive compression scheme in
Section~\ref{sec:norm_algorithm} and show instance optimality in Section~\ref{sec:adaptivity_norm_theoretical_guarantees}.

\subsubsection{Instance-optimal FME for norm \texorpdfstring{$M$}{M}}
\label{sec:norm_algorithm}

We present a two-round procedure that achieves the optimal error in  \cref{eqn:mean-estimation-optimal-rate} with instance-optimal communication complexity of
$O\br{\min\br{d, M^2n^2\epsilon^2/\log{1/\delta}, M^2nd}}$,
without prior knowledge of $M$. The main idea is to use count-mean sketch and in the first round, to construct a private yet accurate estimate of $M$.
This is enabled by the fact the count-mean sketch approximately preserves norms. 
In the second round, we set the sketch size based on this estimate. 
Such interactivity is $(i)$ provably necessary for instance-optimal FME as we show in \cref{thm:lower-bound-communication-bounded-mean-2}, and $(ii)$ needed in other problems that target instance-optimal bounds with DP \cite{berrett2020locally}. 

The  client protocol (line 3 in \cref{alg:dp-mean-estimation-adapt-server_adapt_mean})
is similar to that in \citet{chen2022fundamental}; each client computes a sketch, clips it,
and sends it to the server. 
However,
it
crucially differs in that each client sends two sketches.
The second sketch is used for estimating the statistic
to ensure that we can achieve instance-optimal compression-utility tradeoffs as outlined in $(i)$ and $(ii)$ above. 
We remark that even though only one sketch is needed for the Adapt Norm approach, since the statistic 
can be directly estimated from the (first) sketch $S$ without a second sketch, allowing for additional minor communication optimization, we standardize our presentation on this two-sketch approach to make it inline with the Adapt Tail approach presented in \cref{sec:adaptivity_tail_norm} which requires both sketches.

\begin{algorithm}[t]

\caption{
Adapt Norm FME
}
\label{alg:dp-mean-estimation-adapt-server_adapt_mean}
\begin{algorithmic}[1]
\REQUIRE Sketch sizes $
\cols_1, \pads,
\tilde \cols, \tilde \pads$, 
noise variance
$\tilde \sigma$, $\sigma$, 
$\overline \gamma$,
{ client data $\bc{z_c}_c$}
\FOR{$j=1$ to $2$}
\STATE Select $n$ random clients {$\{z_c^{(j)}\}_{c=1}^n$}
and 
broadcast sketching operators $S_j$ and $\tilde S$ of sizes $(\cols_j, \pads)$ and  $(\tilde \cols, \tilde \pads)$, respectively.
\STATE \textbf{SecAgg}: 
{\small $\nu_j=\! \text{SecAgg}(\{Q^{(c)}_{j}\}_{c=1}^n)+\cN\left(0,\frac{\sigma^2}{n} \bbI_{PC}\right)$}, \\ $\tilde \nu_j \!= \!
 \text{SecAgg}(\{\tilde Q^{(c)}_{j}\}_{c=1}^n)${, where  $Q^{(c)}_{j} \gets \mathsf{clip}_{B}(S_{j}(z^{(j)}_c))$, 
$\tilde Q^{(c)}_j \gets \mathsf{clip}_{B}(\tilde S_{j}(z^{(j)}_{c}))$ } 
\STATE   \textbf{Unsketch DP mean}:
Compute $\omean_j = U_{j}(\nu_j)$
        \STATE \textbf{Estimate DP norm}: {\small $\hat n_j\!\!=\!
         \text{clip}_B(\norm{\tilde \nu_j})\!+\! \text{Laplace}(\tilde \sigma) $}
        \STATE  $C_{j+1} =
        \max\br{\left \lceil\min\br{\frac{n^2\epsilon^2}{\log{1/\delta}},nd}\frac{\br{\hat n_j + \overline \gamma}^2}{G^2P}\right \rceil, 2}$
    \ENDFOR
\end{algorithmic}
\end{algorithm}

The server protocol, in \cref{alg:dp-mean-estimation-adapt-server_adapt_mean}, aggregates the sketches using SecAgg (line 3). In the first round, 
 $Q_j$'s are not used, i.e., $\cols_1=0$.
 Only $\tilde Q_j$'s are used to construct a private estimate of the norm of the mean $\hat n_1$ (line 5). We do not need to unsketch $\tilde Q_j$'s as the norm is preserved under our sketching operation.
 We use this estimate to set the next round's sketch size, $C_{j+1}$ (line 6), so as to ensure the compression error of the same order as privacy and statistical error; this is the same choice we made when we assumed oracle knowledge of $M$---see discussion after \cref{eqn:norm_adptive_bound}. In the second round and onwards, we use the first sketch, which is aggregated using SecAgg and privatized by adding appropriately scaled Gaussian noise (line 3).
Finally, the $Q_j$'s are unsketched to estimate the mean.
Note that the clients need not send $\tilde Q_j$'s after the first round. 

\subsubsection{Theoretical analysis }
\label{sec:adaptivity_norm_theoretical_guarantees}

 We show that the Adapt Norm approach (\cref{alg:dp-mean-estimation-adapt-server_adapt_mean})
 achieves instance-optimality with respect to $M$; an optimal error rate of \cref{eqn:mean-estimation-optimal-rate} is achieved for every instance of the problem with an efficient communication complexity of  $O(M^2n^2\epsilon^2)$ (Theorem~\ref{thm:dp-mean-estimation-upper-bound-1}). The optimality follows from a matching lower bound  in  Corollary~\ref{thm:lower-bound-communication-bounded-mean-1}. 
 We further establish that interactivity, which is a key feature of Algorithm~\ref{alg:dp-mean-estimation-adapt-server_adapt_mean}, is critical in achieving instance-optimality. This follows from a lower bound in Theorem~\ref{thm:lower-bound-communication-bounded-mean-2} that proves a fundamental gap between interactive and non-interactive algorithms. 

\begin{theorem}
\label{thm:dp-mean-estimation-upper-bound-1}
For any choice of the failure probability $0 < \beta < \min\br{\frac{\log{1/\delta}}{n^2\epsilon^2},1}$, 
\cref{alg:dp-mean-estimation-adapt-server_adapt_mean}
with 
$B=2G,
\cols_1  = 0,
\pads = \tilde \pads = \lceil \Theta\br{\log{4/\beta}} 
    \rceil,\tilde \cols  = 2,\sigma^2 =\frac{256B^2\log{1/\delta}}{\epsilon^2},
    \tilde \sigma = \frac{4B}{n\epsilon} \text{ and }
    \overline \gamma =  \frac{2B\log{8/\beta}}{n\epsilon} + \frac{B\sqrt{\log{16/\beta}}}{\sqrt{n}}$
satisfies $(\epsilon,\delta)$-DP,
 the output $\omean_2$ is an unbiased estimate of mean, and its error
is bounded as,
\begin{align*}
    &\mathbb{E}[\norm{\omean_2 - \pmean}]^2 \leq  
    O\br{G^2\br{\frac{d\log{1/\delta}}{n^2\epsilon^2}+ \frac{1}{n}}}\;.
\end{align*}
Finally, the number of rounds is two, and
with probability at least $1-\beta$,
the total per-client communication complexity is 
    $\tilde O\br{\min\br{\frac{n^2\epsilon^2}{\log{1/\delta}},nd} 
    \br{M^2+\frac{1}{n^2\epsilon^2} + \frac{1}{n}}}$.
\end{theorem}
We provide a proof  in \cref{app:proofs_upper_bounds}. The error rate matches \cref{eqn:mean-estimation-optimal-rate} and cannot be improved in general. 
Compared to the target communication complexity of $O((M^2+1/n)n^2\epsilon^2)$, the above 
communication complexity has an additional term 
$1/(n\epsilon)^2$,
 which stems from the fact that 
the norm can only be accessed privately.
In the interesting regime where the error is strictly less than  the trivial $M^2G^2$, the extra
$1/(n\epsilon)^2+1/n $ is smaller than $M^2$. The resulting communication  complexity is $O(M^2n^2\epsilon^2/ \log{1/\delta})$. This nearly matches the oracle communication complexity that has the knowledge of $M$. In the following, we make this precise.

\paragraph{\cref{alg:dp-mean-estimation-adapt-server_adapt_mean}
is instance-optimal with respect to~${\bm M}$.} 
The next theorem shows that even under the knowledge of $M$, no unbiased procedure under a SecAgg constraint with optimal error can have smaller communication complexity than the above procedure. 
We provide a proof in \cref{app:proofs_lower-bounds_communication}.

\begin{corollary}[{\cite{chen2022fundamental} Theorem 5.3}]
\label{thm:lower-bound-communication-bounded-mean-1}
\sloppy 
Let $K,d,n \in \bbN$, $M,G,\epsilon,\delta \geq 0$,
and $\cP_1(d,G,M):= \bc{\cD \text{ over }  \bbR^d : \norm{z}\leq G \text{ for } z\sim \cD \text{ and }
    \norm{\mu(\cD)}
    \leq MG}$.
For any $K$, any $K$-round unbiased SecAgg-compatible protocol $\cA$
(see \cref{sec:multi-round-protocols} for details)
such that
$\mathbb{E}_{D\sim \cD^n}[\norm{\cA(D) -  \mu}^2] = O\br{\min\br{M^2G^2,\frac{G^2}{n}+\frac{G^2d\log{1/\delta}}{n^2\epsilon^2}}} \ \forall \ \cD \in \cP_1(d,G,M)$, there exists a distribution $\cD \in  \cP_1(d,G,M)$, such that on dataset $D\sim \cD^n$, w.p. 1, the total per-client communication complexity is $\Omega\br{\min\br{d,\frac{M^2n^2\epsilon^2}{\log{1/\delta}},  M^2nd }}$ bits. 
\end{corollary}

\paragraph{Interaction is necessary.}
A key feature of our algorithm is interactivity: the norm of the mean estimated in the prior round is used to determine the sketch size in the next round. We show that at least two rounds of communication are necessary for any algorithm with instance-optimal communication complexity. 
This proves a fundamental gap between interactive and non-interactive approaches in solving FME. We provide a proof in \cref{app:proofs_lower-bounds_communication}.

\begin{theorem}
\label{thm:lower-bound-communication-bounded-mean-2}
\sloppy 
Let $K,d,n \in \bbN, n\geq 2$, $G,\epsilon,\delta >0$, and $\cP_2(d,G):= \bc{\cD \text{ over }  \bbR^d : \norm{z}\leq G \text{ for } z\sim \cD}$.
Let $\cA$ be a $K$-round unbiased SecAgg-compatible protocol with $\underset{D\sim \cD^n}{\mathbb{E}}[\norm{\cA(D) -  \mu(\cD)}^2] =O\Big(\min\Big(\norm{\mu(\cD)}^2,\frac{G^2}{n}+\frac{G^2d\log{1/\delta}}{n^2\epsilon^2}\Big)\Big)$ and total per-client communication complexity of $O\Big(\min\Big(d,\frac{\norm{\mu(\cD)}^2n^2\epsilon^2}{\log{1/\delta}},  \norm{\mu(\cD)}^2 nd \Big)\Big)$ bits with probability $>\frac{1}{n}$, point-wise $\forall \ \cD \in \cP_2(d,G)$. Then $K\geq 2$.
\end{theorem}

\subsection{Adapting to the Instance's Tail-norm} 
\label{sec:adaptivity_tail_norm} 

The key idea in norm-adaptive compression is interactivity. On top of interactivity, another  key idea in tail-norm-adaptive compression is count median-of-means sketching.

\ifarxiv
\paragraph{Count median-of-means sketching.}
\else
{\bf Count median-of-means sketching.}
\fi
Our new sketching technique takes $R \in \bbN$ independent 
count-mean sketches of \citet{chen2022fundamental} (see \cref{sec:prelims}).
Let $S^{(i)}: \bbR^d \rightarrow \bbR^{PC}$ denote the $i$-th count-mean sketch. Our  sketching operation, $S: \bbR^{d}\rightarrow\bbR^{R\times PC}$,  is defined as the concatenation: 
$   S(z) = Sz = \begin{bmatrix}
    (S^{(1)}z)^\top 
    (S^{(2)}z)^\top
    \ldots
    (S^{(\rows)}z)^\top
    \end{bmatrix}^\top$.  
Our median-of-means unsketching takes the median over $R$ unsketched estimates:
$
      U(S(z))_j = \text{Median}(\{({S^{(i)}}^\top S^{(i)} z)_j\}_{i=1}^R)
$. 
This \textit{median trick}  boosts  the confidence
to get a high probability of success (as opposed to a guarantee in expectation) and to get tail-norm (as opposed to norm) based bounds
\cite{charikar2002finding}.
However, in non-private settings, it suffices to take multiple copies of a count-sketch, which in our notation, corresponds to setting $P=1$. On the other hand, we set $P>1$ to get bounded sensitivity which is useful for DP, yielding our (novel) count-median-of-means sketching technique. We remark that in some places, we augment the unsketching operation with ``Top$_k$", which returns top-$k$ coordinates, of a vector, in magnitude.

\begin{algorithm}[t]
\caption{ Adapt Tail FME
}
\label{alg:dp-mean-estimation-adapt-server_adapt_tail}
\begin{algorithmic}[1]

\REQUIRE Sketch sizes $
\rows,\cols_1, \pads,
\tilde \rows, \tilde \cols, \tilde \pads$, 
noise 
$\sigma$,
$\tilde \sigma$,
$\bc{{k_j}}_j$, $\bc{\overline \gamma_j}_j$,  rounds $K$, {client data $\bc{z_c}_c$}
\FOR{$j=1$ to $K$}
\STATE Select $n$ random clients
{$\{z_c^{(j)}\}_{c=1}^n$}
and 
broadcast sketching operators $S_j$ and $\tilde S_j$ of sizes $(\rows,\cols_j, \pads)$ and  $(\tilde \rows,\tilde \cols, \tilde \pads)$ respectively..
\STATE \textbf{SecAgg}: 
{\small $\nu_j=\! \text{SecAgg}(\{Q^{(c)}_{j}\}_{c=1}^n)+\cN\left(0,\frac{\sigma^2}{n} \bbI_{PC}\right)$}, \\ $\tilde \nu_j \!= \!
             \text{SecAgg}(\{\tilde Q^{(c)}_{j}\}_{c=1}^n)$
 {, where  $Q^{(c)}_{j} \!\!\gets\!\! [\mathsf{clip}_{B}(S^{(i)}_{j}(z^{(j)}_c))]_{i=1}^R$, 
$\tilde Q^{(c)}_j\!\! \gets [\mathsf{clip}_{B}(\tilde S^{(i)}_{j}(z^{(j)}_{c}))]_{i=1}^{\tilde R}$ } 
\STATE   \textbf{Unsketch DP mean}:
Compute $\omean_j = \text{Top}_{{k_j}}(U_{j}(\nu_j))$
        \STATE \textbf{Estimate error}:
        $\bar e_j = \norm{\tilde S_j(\omean_j) - \tilde \nu_j}$
         \STATE $ \tilde \gamma_j = \bar \gamma_j + \text{Laplace}(\tilde \sigma)$ and $\tilde e_j = \bar e_j + \text{Laplace}(2\tilde \sigma)$
        \STATE  \textbf{If} $\tilde e_j \leq \tilde \gamma_j$, set  $ \omean = \omean_j$, \textbf{break}
        \STATE $C_{j+1} = 2C_j$
    \ENDFOR
    
\end{algorithmic}
\end{algorithm}

\paragraph{Approximately-sparse setting.}
We expect improved guarantees when $\mu$ is \textit{approximately sparse}. This is  captured by the tail-norm.
Given $z \in \bbR^d, \norm{z}\leq G$, 
and $k \in [d]$, the \textit{normalized} tail-norm is $
\norm{z_{\ntail{k}}}_2
    := \frac{1}{G}\min_{\tilde z: \norm{\tilde z}_0\leq k
    }\norm{z-\tilde z}_2\;
$. 
This measures the  error in the best $k$-sparse approximation. 
We show that the mean estimate via count median-of-means sketch is still {\em unbiased} (\cref{lem:unbiased-count-median-of-means}), and 
for $C=\Omega\br{Pk}$,
has error  bounded as,
{\small
\begin{align}
    \norm{\bar \mu - \mu}_2^2
    \!=\!O\Big(\frac{dG^2}{PC}\Big(\norm{\mu_{\ntail{k}}}^2+\frac{1}{n}\Big) +   \frac{G^2}{n} + \frac{G^2d}{n^2\epsilon^2}\Big)
    \label{eqn:tail_adaptive_bound}
\end{align}
}
If {\em all} tail-norms are known, then we can set the sketch size
\ifarxiv
\sloppy
$PC=\min_{k}\max\br{k,\min\br{nd,n^2\epsilon^2}\br{\norm{\mu_{\ntail{k}}}^2+\frac{1}{n}}}$,
\else
$PC\!\!\!=\!\!\!\min_{k}\max\br{k,\min\br{nd,n^2\epsilon^2}\br{\norm{\mu_{\ntail{k}}}^2\!+\!\frac{1}{n}}}$,
\fi
and achieve the error rate in \cref{eqn:mean-estimation-optimal-rate}.
When $k=0$, this recovers the previous  communication complexity of  count-mean sketch in \cref{sec:adaptivity_norm}. For optimal choice of $k$, this communication complexity can be smaller. 
We aim to achieve it adaptively, without the (unreasonable) assumption of knowledge of all the tail-norms. 

\subsubsection{Instance-optimal FME for tail-norm}
Previously, we estimated the  norm of the mean in one round. Now, we need multiple tail-norms. 
 We propose a doubling-trick based scheme. Starting from an optimistic guess of the sketch size, we progressively double
it, until an appropriate stopping criterion on the error is
met.
The main challenge is in estimating the error of the first sketch $S$ (for the current choice of sketch size), as naively this would require the uncompressed true vector $z$ to be sent to the server.
To this end, we show how to use the second sketch $\tilde S$ so as to obtain a reasonable estimate of this error, while using significantly less communication than transmitting the original $z$. 

We re-sketch the unsketched estimate, $U(S(z))$, of $z$ with $\tilde S$. 
The reason is that 
this can now be compared with the second sketch $\tilde S(z)$ to get an estimate of the error: 
$
\norm{\tilde S(U(S(z))) - \tilde S(z)}^2 = 
    \norm {\tilde S (U(S(z)) -z)}^2 \approx \norm{U(S(z))) - z}^2
$, where we used the linearity and norm-preservation property of the count-sketch.

\cref{alg:dp-mean-estimation-adapt-server_adapt_tail} presents the pseudo-code for \textbf{Adapt Tail}.
The client protocol
remains the same; each participating client sends {\em two} independent (count median-of-mean) sketches to the server.
The server, starting from an optimistic choice of initial sketch size $PC_1$, obtains the aggregated sketches from the clients via SecAgg and adds noise to the first sketch for DP (line 4).
It then unsketches the first sketch to get an estimate of mean, $\omean_j$ (line 4)---we note that $\text{Top}_{k_j}$ is not applied (i.e., $k_j=d$) for the upcoming result but will be useful later. We then sketch $\omean_j$ with the second sketch $\tilde S$ and compute the norm of difference with the aggregated second sketch from clients (line 5)---this gives an estimate of the error of $\omean_j$. Finally, we want to stop if the error is sufficiently small, dictated by the threshold $\overline \gamma_j$, which will be set as target error in \cref{eqn:mean-estimation-optimal-rate}. To preserve privacy at this step, we use the well-known \textbf{AboveThreshold} algorithm \cite{dwork2014algorithmic}, which adds noise to both the error and threshold (line 6) and stops if the noisy error is smaller than the noisy threshold (line 7). If this criterion is not met, then we double the first sketch size (line 8) and repeat.
 
\subsubsection{Theoretical analysis }
Given a non-decreasing 
function $g:[d]\rightarrow \bbR$,  we define the generalized sparsity as 
    $k_{\text{tail}}(g(k);\mu) := \min_{k \in [d]}\bc{k:\norm{\mu_{\text{\ntail{k}}}}^2\leq g(k)}$.
In the special case when $g(k)\equiv 0$, this recovers the sparsity of $\mu$.

We now present our result for the proposed Adapt Tail approach.
\begin{theorem}
\label{thm:dp-mean-estimation-upper-bound-2}
For any choice of failure probability $0\!\!<\!\!\beta\!\!<\!\!1$,
\cref{alg:dp-mean-estimation-adapt-server_adapt_tail}
with
$B \!\!\!=\!\!\! 2G,
{k_j}\!=\!d,
K\!\!\!\!=\!\!\!\!\lfloor \log{d} \rfloor,
\tilde \sigma\!=\!\frac{4B}{n\epsilon}, 
\sigma^2\! =\!\frac{256\rows\br{\log{d}}^2B^2\log{1/\delta}}{\epsilon^2},
\rows\!\!\!\!=\!\!\!\!  \left\lceil 2\log{8d\log{d}/\beta}\right\rceil,
\cols_1\!\!\!=\!\!\!8\pads,
 \pads\!\!\!\!\!\!=\!\!\!\!\!\!\lceil \Theta(2\log{8\rows\log{d}/\beta}) \rceil,
  \tilde \cols\!\!\!\!\!=\!\!\!\!\!2\tilde \pads ,
 \tilde \rows\!\!\!\!\!\!=\!\!\!\!\!\!\!1, 
 \tilde \pads \!\!=\!\! \lceil \Theta\br{2\log{4d\log{d}/\beta}}\rceil,
 \overline \gamma_j =  \overline \gamma =15\max\br{\frac{G\sqrt{\log{8\log{d}/\beta}}}{\sqrt{n}} ,\sqrt{d}\frac{\sigma}{n}} + \tilde \alpha
  \qquad\text{and }
 \tilde \alpha = \frac{32B\br{\log{\lfloor \log{d} \rfloor} + \log{8/\beta}}}{n\epsilon},
$
satisfies $(\epsilon,\delta)$-DP, and outputs an unbiased estimate of the mean.
With probability at least $1-\beta$, the error is bounded as,
\begin{align*}
    &\norm{\omean - \pmean}^2 
    = \tilde O\br{\frac{G^2}{n} + \frac{G^2 d\log{1/\delta}}{n^2\epsilon^2}},
\end{align*}
 the total per-client communication complexity is 
 $\tilde O\br{k_{\text{tail}}(g(k));\mu)}$,
 where $g(k) = \max\br{\frac{1}{nd},\frac{\log{1/\delta}}{n^2\epsilon^2}} k-\frac{4\log{8\log{d}/\beta}}{n}$,
and number of rounds is $\lfloor \log{d} \rfloor$.
\end{theorem}
We provide the proof of \cref{thm:dp-mean-estimation-upper-bound-2} in \cref{app:proofs_upper_bounds}.
First, we argue that the communication complexity is of the same order as of this method with prior knowledge of all tail norms---plugging in $g(k)$ in the definition of $k_{\text{tail}}(g(k);\mu)$ gives us that $k_{\text{tail}}(g(k);\mu)$ is the smallest $k$ such that $k\geq \min\br{nd,n^2\epsilon^2}\br{\norm{\mu_{\ntail{k}}}^2 + \frac{1}{n}}$. This is what we obtained before, which further is no larger than the result on adapting to norm of mean (\cref{thm:dp-mean-estimation-upper-bound-1})---see discussion after \cref{eqn:tail_adaptive_bound}.
However, this algorithm requires  $\lfloor \log{d} \rfloor$ rounds of interaction, as opposed to two in \cref{thm:dp-mean-estimation-upper-bound-1}. It therefore depends on the use-case which of the two, total communication or number of rounds, is more important.

The error rate in \cref{thm:dp-mean-estimation-upper-bound-2} matches that in \cref{eqn:mean-estimation-optimal-rate}, known to be optimal in the worst case.
However, it may be possible to do better 
for specific values of tail-norms.
Under the setting where the algorithm designer is given $k<d$ and $\gamma$ and is promised that $\norm{\mu_{\ntail{k}}}\leq \gamma$, we give a lower bound of $\Omega\br{G^2\min \bc{1, \gamma^2 + \frac{k}{n^2\epsilon^2} + \frac{1}{n}, \frac{d}{n^2\epsilon^2} + \frac{1}{n}}}$ on error of (central) DP mean estimation (see \cref{thm:dp-mean-estimation-lower-bound-2}).
The second term here is new, and
we show that a simple tweak to our procedure underlying \cref{thm:dp-mean-estimation-upper-bound-2}---adding a Top$_k$ operation, with exponentially increasing $k$, to the unsketched estimate---suffices to achieve this error (see \cref{thm:dp-mean-estimation-upper-bound-3}).
The resulting procedure is \textbf{biased} and has a communication complexity of $\tilde O(k_{\text{tail}}(\gamma^2;\mu))$, which we show to the optimal under SecAgg constraint among all, potentially biased, multi-round procedures (see \cref{thm:lower-bound-communication-bounded-mean-3}). Due to space constraints, we defer formal descriptions of these results to \cref{app:missing-tail-norm}.
\section{Federated Optimization/Learning}

\label{sec:fme-to-fo-norm}
\begin{algorithm}[h]
\caption{Adapt Norm FL
}

\begin{algorithmic}[1]

\REQUIRE Sketch sizes $\ssize_1 = RPC_1$ and $\sssize = \tilde R \tilde P \tilde C$
noise multiplier $\sigma$, model dimension $d$,
adaptation method $\mathrm{adapt}$, a constant $c_0$,
clipping threshold $B$,
rounds $K$.
\FOR{$j=1$ to $K$}
\STATE Sample $n$ clients;
broadcast current model and sketching operators $S_j,\tilde S_j$ of sizes 
$\ssize_j = RPC_j$ and $\sssize$.
\STATE \textbf{SecAgg}:
{\small $\nu_j=\! \text{SecAgg}(\{Q^{(c)}_{j}\}_{c=1}^n)+\cN\left(0,\frac{\sigma^2B^2}{0.9n} \bbI_{PC}\right)$}, \\ $\tilde \nu_j \!= \!
             \text{SecAgg}(\{\tilde Q^{(c)}_{j}\}_{c=1}^n)$
              {, where  $Q^{(c)}_{j} \!\!\gets\!\! [\mathsf{clip}_{B}(S^{(i)}_{j}(z^{(j)}_c))]_{i=1}^R$, 
$\tilde Q^{(c)}_j\!\! \gets [\mathsf{clip}_{B}(\tilde S^{(i)}_{j}(z^{(j)}_{c}))]_{i=1}^{\tilde R}$ } 
\STATE   \textbf{Unsketch DP mean}:$\omean_j = 
U_{j}(\nu_j)$
        \STATE \textbf{Second sketch}:
        $\bar n_j = \norm{\tilde v_j}$
        \STATE $\overline \gamma^2 = 20 \sigma^2$, $\tilde \sigma = \sigma/\sqrt{0.1}$
\STATE  $\hat n_j = \text{clip}_B(\bar n_j) + \cN(0,\tilde \sigma^2B^2) $
\STATE $C_{j+1} =
\frac{c_0\br{\hat n_j+\overline \gamma}^2}{B^2P\sigma^2}$
    \ENDFOR
\end{algorithmic}

\label{alg:dp-mean-estimation-adapt-server_adapt_experiments}
\end{algorithm}
\begin{algorithm}[t]
\caption{Two Stage FL
}
\begin{algorithmic}[1]
\REQUIRE Sketch sizes $\ssize_1 = RPC_1$ and $\sssize = \tilde R \tilde P \tilde C$
noise multiplier $\sigma$, model dimension $d$,
adaptation method $\mathrm{adapt}$, a constant $c_0$,
clipping threshold $B$,
rounds $K$.
\FOR{$j=1$ to $W$}
\STATE \textbf{SecAgg}: {\small $\nu_j=\! \text{SecAgg}(\{Q^{(c)}_{j}\}_{c=1}^n)+\cN\left(0,\frac{\sigma^2B^2}{0.9n} \bbI_{PC}\right)$},
{ where  $Q^{(c)}_{j} \!\!\gets\!\! [\mathsf{clip}_{B}(S^{(i)}_{j}(z^{(j)}_c))]_{i=1}^R$} 
\STATE $\bar n_j = \norm{\text{SecAgg}(\{Q^{(c)}_{j}\}_{c=1}^n)}$
\STATE   \textbf{Unsketch DP mean}:$\omean_j = 
U_{j}(\nu_j)$
\STATE  $\hat n_j = \text{clip}_B(\bar n_j) + \cN(0, \sigma^2B^2/0.1) $
    \ENDFOR
\STATE $\hat n = \frac{1}{W}\sum_{j=1}^W \hat n_j, C =
\frac{c_0\br{\hat n+ \sqrt{20}\sigma}^2}{B^2P\sigma^2}$
\FOR{$j=W+1$ to $K$}
\STATE Select $n$ random clients
and 
broadcast current model and sketching operator $S_j$
of size
$\ssize = RPC$.
\STATE Receive $\{Q^{(c)}_{j}\}_{c=1}^n$
from clients.
\STATE {\small $\nu_j=\! \text{SecAgg}(\{Q^{(c)}_{j}\}_{c=1}^n)+\cN\left(0,\frac{\sigma^2B^2}{n} \bbI_{PC}\right)$}, 
\STATE   \textbf{Unsketch DP mean}:$\omean_j = 
U_{j}(\nu_j)$
    \ENDFOR
\end{algorithmic}
\label{alg:two-stage}
\end{algorithm}
In the previous section, we showed how to adapt the compression rate to the problem instance for FME. In this section, we will use our proposed FME procedure for auto-tuning the compression rate in Federated Optimization (FO).
We use the ubiquitous FedAvg algorithm~\cite{mcmahan2017communication}, which is an iterative procedure, computing
an average of the client updates at every round. It is this averaging step that we replace with our corresponding
 FME procedures
(\cref{alg:dp-mean-estimation-adapt-server_adapt_mean} or \cref{alg:dp-mean-estimation-adapt-server_adapt_tail})
We remind the reader that when moving from FME to FO, that the client data $z_c$ is now the (difference of) model(s) updated via local training at the client.

However, recall that our proposed procedures for FME require multiple rounds of communication. While this can be achieved in FO by \textit{freezing} the model for the interactive FME rounds, it is undesirable in a real-world application of FL as it would increase the total wall-clock time of the training process.
Thus, we 
propose an additional heuristic for our
Adapt Norm and Adapt Tail algorithms when used in FO/FL: to use a \textit{stale estimate} of the privately estimated statistic using the second sketch.
We discuss this for the Adapt Norm procedure and defer details and challenges surrounding Adapt Tail to \cref{app:adapt-tail}.

{
\ifarxiv
\paragraph{\textbf{Two Stage method (\cref{alg:two-stage}).}}
\else
\textbf{Two Stage method (\cref{alg:two-stage}).}
\fi
First, we describe a simple approach based on the observation from the Genie that a single fixed compression rate works well. We assume that the norm of the updates remains relatively stable throughout training. To estimate it, we run $W$ warm-up rounds as the first stage. Then, using this estimate, we compute a fixed compression rate, by balancing the errors incurred due to compression and privacy, akin to Adapt Norm in FME, which we then use for the second stage. Because the first stage is run without compression, it is important that we can minimize $W$, which may be possible through prior knowledge of the statistic, e.g., proxy data distributions or other hyperparameter tuning runs.}

\ifarxiv
\paragraph{Adapt Norm (\cref{alg:dp-mean-estimation-adapt-server_adapt_experiments}).}
\else
\textbf{Adapt Norm (\cref{alg:dp-mean-estimation-adapt-server_adapt_experiments}).}
\fi
Our main algorithm
at every round uses two sketches: one to estimate the mean for FL and the other to compute an estimate of its norm, which is used to set the (first) sketch size for the next round. 
This is akin to the corresponding FME \cref{alg:dp-mean-estimation-adapt-server_adapt_mean} with the exception that it uses stale estimate of norm, from the prior round, to set the sketch size in the current round---in our experiments, we find that this heuristic still provides accurate estimates of the norm at the current round.
Further, we split the privacy budget between the mean and norm estimation parts heuristically in the ratio $9:1$, and set sketch size parameters $\tilde R=1, R=P=\tilde P = \lceil\log{d} \rceil$. Finally, the constant $c_0$ is set such that the total error in FME, at every round, is at most $1.1$ times the DP error.

We also note some minor changes between \cref{alg:dp-mean-estimation-adapt-server_adapt_mean} (FME) and
its adaptation to \cref{alg:dp-mean-estimation-adapt-server_adapt_experiments} (FO), made only for simplification, explained below.
We replace the Laplace noise added to norm by Gaussian noise for ease of privacy accounting in practice.
Further, the expression for the sketch size (line 9) may look different, however, for all \textit{practically relevant} regime of parameters, it is the same as line 7 in \cref{alg:dp-mean-estimation-adapt-server_adapt_mean}.
\section{Empirical Analysis on Federated Learning}
\label{sec:experiments}
\begin{figure*}[t]
    \centering
    \ifarxiv
    \else
    \vspace{-1.5em}
    \fi
    \subfloat[F-EMNIST]{\includegraphics[width=.33\linewidth]{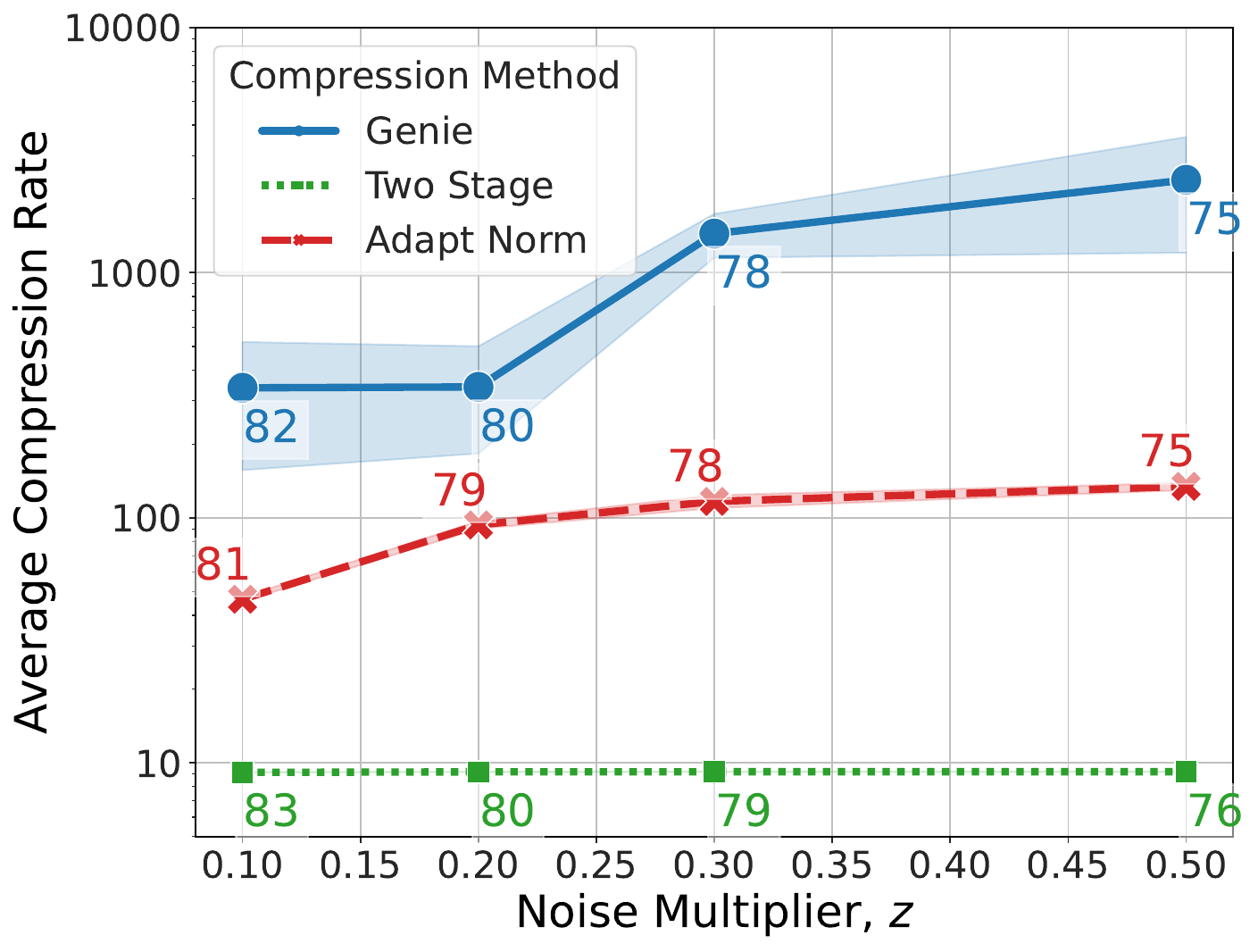}}
    \subfloat[SONWP]{\includegraphics[width=.33\linewidth]{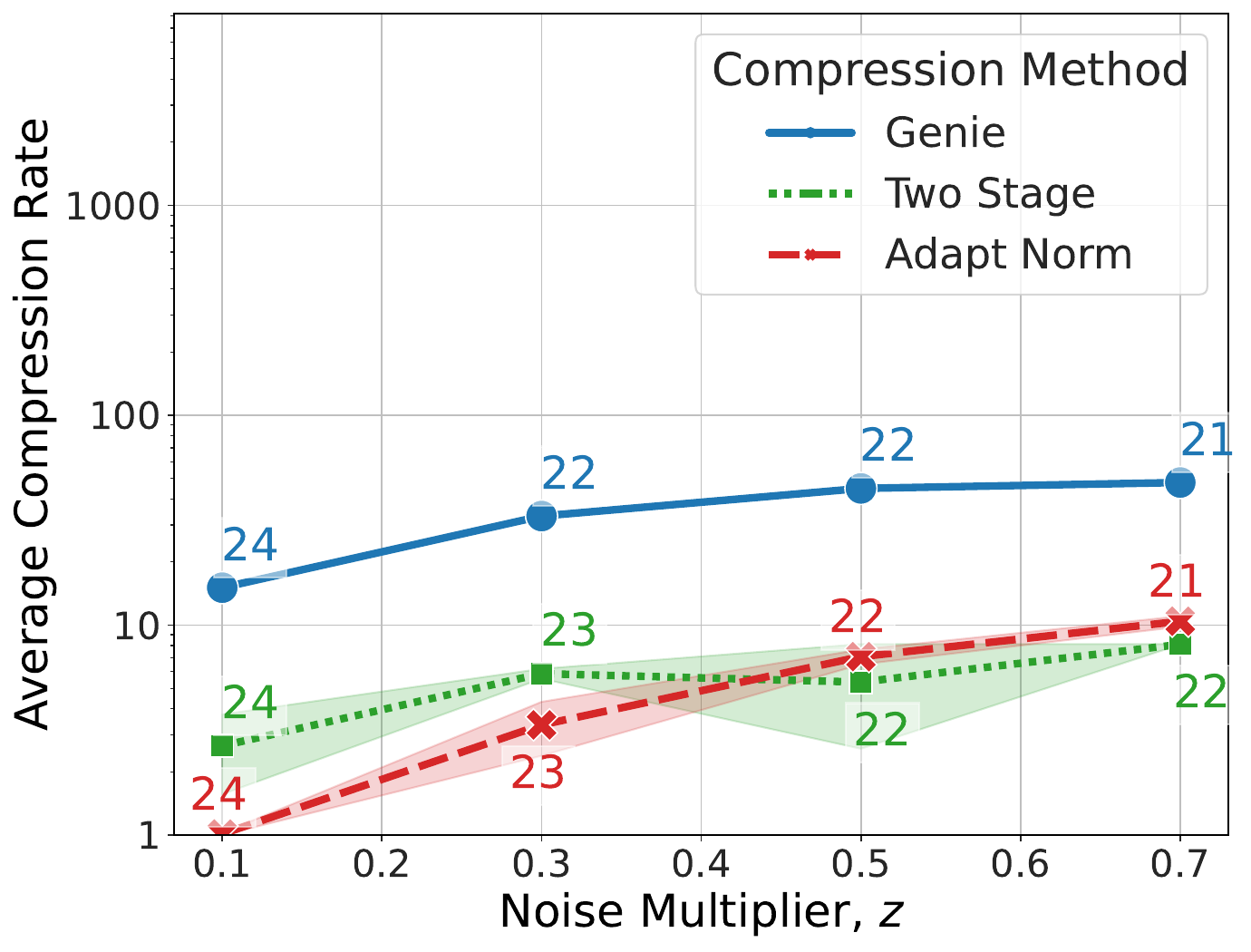}}
    \subfloat[Shakespeare]{\includegraphics[width=.33\linewidth]{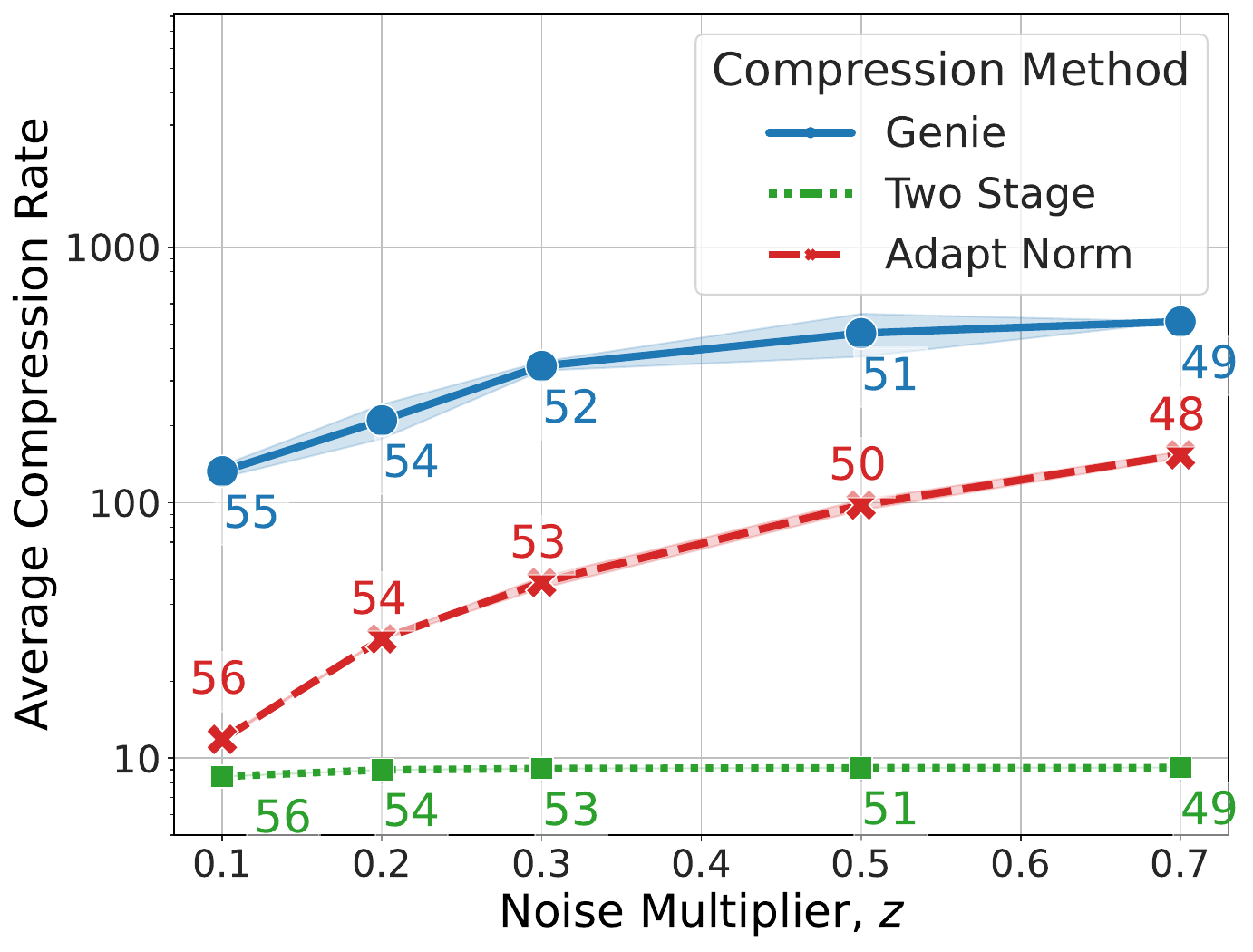}}
    \ifarxiv
    \else
    \vspace{-1em}
    \fi
    \caption{\textbf{Our adaptive approaches achieve favorable compression rates without tuning.} The `Genie' represents the optimal fixed compression rate obtained via significant hyper parameter tuning. Our adaptive approaches are all $1$-shot using a fixed constant of $0.1$. 
    Text represents the final validation accuracy, with the `Genie' at $\Delta=1\%$.}
    \label{fig:adaptive-compare}
    \vspace{-1em}
\end{figure*}

In this section, we present experimental evaluation of the methods proposed in \cref{sec:fme-to-fo-norm} for federated optimization, in standard FL benchmarks.
We define the average compression rate of an $T$-round procedure as be relative decrease in the average bits communicated, i.e., 
$\frac{dT}{\sum_{t=1}^TL_t+\tilde L_t}$ where $d$ is the model dimensionality, $L_t$ and $\tilde L_t$ are sizes of first and second sketches at round $t$. This is equivalently the harmonic average of the per-round compression rates.

\paragraph{Setup.}
We focus on three common FL benchmarks: Federated EMNIST (F-EMNIST) and Shakespeare represent two relatively easier tasks, and Stack Overflow Next Word Prediction (SONWP) represents a relatively harder task. F-EMNIST is an image classification task, whereas Shakespeare and SONWP are language modelling. 
We follow the exact same setup (model architectures, hyper parameters) as~\citet{chen2022fundamental} except where noted below. Our full description can be found in \cref{sec:app-setup}.

Unlike~\citet{chen2022fundamental}, we use fixed clipping (instead of adaptive clipping) and tune the server's learning rate for each noise multiplier, $z$. 
As noted by~\citet{choquette2022multi}, we also zero out high $\ell_1$ norm updates $\geq 100$ for improved utility and use their hyper parameters for SONWP.
\ifarxiv
\paragraph{Defining Feasible Algorithms.}
\else
{\bf Defining Feasible Algorithms.}
\fi
Consider that introducing compression into a DP FL algorithm introduces more noise into its optimization procedure. Thus, we may expect that our algorithms will perform worse than the baseline. We define $\Delta$ to be the max allowed relative drop in utility (validation accuracy) when compared to their baseline without compression. Then, our set of feasible algorithms are those that achieve at least $1-\Delta$ of the baseline performance. This lets us study the privacy-utility tradeoff under a fixed utility and closely matches real-world constraints---often a practitioner will desire an approach that does not significantly impact the model utility, where $\Delta$ captures their tolerance.

\ifarxiv
\paragraph{Algorithms.}
\else
{\bf Algorithms.}
\fi
Our baseline is the \textbf{Genie}, which is the method of~\citet{chen2022fundamental}, run on the grid of exponentially increasing compression rates $2^{b},b\in[0, 13]$. This requires computation that is logarithmic in the interval size and significantly more communication bandwidth in totality---thus, this is unachievable in practice but serves as our best-case upper bound.
Our proposed adaptive algorithms are \textbf{Adapt Norm} and \textbf{Two Stage method}, described in \cref{sec:fme-to-fo-norm}.

{
\ifarxiv
\paragraph{Achieving Favorable Compression Without Tuning.}
\else
{\bf Achieving Favorable Compression Without Tuning.}
\fi
Because the optimal \textbf{Genie} compression rates may be unattainable without significantly increased communication and computation, we instead desire to achieve nontrivial compression rates, as close this genie as possible, but without any (significant) additional computation, while also ensuring that the impact on utility remains negligible. Thus, to ensure minimal impact on utility, we run our adaptive protocols so that the error introduced by compression is much smaller than that which is introduced by DP. We heuristically choose that error from compression can be at most 10\% of DP: 
this choice of this threshold is such that the error is small enough to be dominated by DP and yet large enough for substantial compression rates.
 We emphasize that we chose this value heuristically and do not tune it. 
 
 In extensive experiments across all three datasets (\cref{fig:adaptive-compare}), we find that our methods achieve essentially the same utility models (as seen by the colored text). We further find that our methods achieve nontrivial and favorable communication reduction---though still far from the genie, they already recover a significant fraction of the compression rate.
This represents a significant gain in computation: to calculate this genie we ran tens of jobs per noise multiplier---for our proposed methods, we run but one. 

Comparing our methods, we find that the Adapt Norm approach performs best. It achieves favorable compression rates consistently across all datasets with compression rates that can well adapt to the specific error introduced by DP at a given noise multiplier. For example, we find on the Shakespeare dataset in \cref{fig:adaptive-compare} (c) that the compression-privacy curve follows that of the genie.
In contrast, the Two Stage approach typically performs much worse. This is in part due to construction: we require a significant number of warmup rounds (we use $W=75$) to get an estimate of the gradient norm. Because these warmup rounds use no compression, it drastically lowers the (harmonic) average compression rate. We remark that in scenarios where strong priors exist, e.g., when tuning a model on a similar proxy dataset locally prior to deploying in production, it may be possible to significantly reduce or even eliminate $W$ so as to make this approach more competitive. However, our fully-adaptive Adapt Norm approach is able to perform well even without any prior knowledge of the such statistics.
}

\ifarxiv
\paragraph{Benchmarking computation overhead.}
\else
\textbf{Benchmarking computation overhead.}
\fi
We remark that sketching is a computationally efficient algorithm requiring computation similar to clipping for encoding and decoding the update. Though our adaptive approaches do introduce minor additional computation (e.g., a second sketch), these do not significantly impact the runtime. In benchmark experiments on F-EMINST, we found that standard DP-Fed Avg takes 3.63s/round, DP-Fed Avg with non-adaptive sketching takes 3.63s/round, and our Adapt Norm approach take 3.69s/round. Noting that our methods may impact convergence, so we also fix the total number of rounds a-priori in all experiments; this provides a fair comparison for all methods where any slow-down in convergence is captured by the utility decrease. 

\ifarxiv
\paragraph{Choosing the relative error from compression ($c_0$).}
\else
\textbf{Choosing the relative error from compression ($c_0$).}
\fi
Though we chose $c_0=10\%$ heuristically and without tuning, we next run experiments selecting other values of $c_0$ to show that our intuition of ensuring the relative error is much smaller leads to many viable values of this hyperparameter. We sweep values representing $\{25,5,1\}\%$ error and compare the utility-compression tradeoff in \cref{sec:app-additional-figures}. We find that $c_0$ well captures the utility-compression tradeoff, where smaller values consistently lead to higher utility models with less compression, and vice versa. We find that the range of suitable values is quite large, e.g., even as low as $c_0=1\%$, non-trivial compression can be attained.

\section{Conclusion and Discussion}
We design autotuned compression techniques for private federated learning that are compatible with secure aggregation. We accomplish this by creating provably optimal autotuning procedures for federated mean estimation and then mapping them to the federated optimization (learning) setting. 
We analyze the proposed mean estimation schemes and show that they achieve order optimal error rates with order optimal communication complexity, adapting to the norm of the true mean (Section~\ref{sec:adaptivity_norm}) and adapting to the tail-norm of the true mean (Section~\ref{sec:adaptivity_tail_norm}). Our results show that we can attain favorable compression rates that recover much of the optimal Genie, in one-shot without any additional computation.
Although the $\ell_2$ error in mean estimation is a \textit{tractable proxy} for autotuning compression rate in federated learning, we found that it may not always correlate well with the downstream model accuracy.
In particular, in our adaptation of Adapt Tail to federated learning (in \cref{app:adapt-tail}), we found that that the procedure is able to attain very high compression rates for the federated mean estimation problem, with little overhead in $\ell_2$ error, relative to the DP error.
However, these compression rates are too high to result in useful model accuracy.
So a natural direction for future work is to design procedures which improve upon this proxy in federated learning settings.
\bibliographystyle{plainnat}
\newpage
\bibliography{main}

\begin{thebibliography}{66}
\providecommand{\natexlab}[1]{#1}
\providecommand{\url}[1]{\texttt{#1}}
\expandafter\ifx\csname urlstyle\endcsname\relax
  \providecommand{\doi}[1]{doi: #1}\else
  \providecommand{\doi}{doi: \begingroup \urlstyle{rm}\Url}\fi

\bibitem[Abadi et~al.(2016)Abadi, Chu, Goodfellow, McMahan, Mironov, Talwar,
  and Zhang]{abadi2016deep}
Martin Abadi, Andy Chu, Ian Goodfellow, H~Brendan McMahan, Ilya Mironov, Kunal
  Talwar, and Li~Zhang.
\newblock Deep learning with differential privacy.
\newblock In \emph{Proceedings of the 2016 ACM SIGSAC conference on computer
  and communications security}, pages 308--318, 2016.

\bibitem[Agarwal et~al.(2018)Agarwal, Suresh, Yu, Kumar, and
  McMahan]{agarwal2018cpsgd}
Naman Agarwal, Ananda~Theertha Suresh, Felix Xinnan~X Yu, Sanjiv Kumar, and
  Brendan McMahan.
\newblock cpsgd: Communication-efficient and differentially-private distributed
  sgd.
\newblock \emph{Advances in Neural Information Processing Systems}, 31, 2018.

\bibitem[Agarwal et~al.(2021)Agarwal, Kairouz, and Liu]{agarwal2021skellam}
Naman Agarwal, Peter Kairouz, and Ziyu Liu.
\newblock The skellam mechanism for differentially private federated learning.
\newblock \emph{Advances in Neural Information Processing Systems},
  34:\penalty0 5052--5064, 2021.

\bibitem[Aji and Heafield(2017)]{aji2017sparse}
Alham~Fikri Aji and Kenneth Heafield.
\newblock Sparse communication for distributed gradient descent.
\newblock \emph{arXiv preprint arXiv:1704.05021}, 2017.

\bibitem[Alistarh et~al.(2017)Alistarh, Grubic, Li, Tomioka, and
  Vojnovic]{alistarh2017qsgd}
Dan Alistarh, Demjan Grubic, Jerry Li, Ryota Tomioka, and Milan Vojnovic.
\newblock Qsgd: Communication-efficient sgd via gradient quantization and
  encoding.
\newblock \emph{Advances in neural information processing systems}, 30, 2017.

\bibitem[Arora et~al.(2022)Arora, Bassily, Guzm{\'a}n, Menart, and
  Ullah]{arora2022differentially}
Raman Arora, Raef Bassily, Crist{\'o}bal~A Guzm{\'a}n, Michael Menart, and
  Enayat Ullah.
\newblock Differentially private generalized linear models revisited.
\newblock In \emph{Advances in Neural Information Processing Systems}, 2022.

\bibitem[Arora et~al.(2023)Arora, Bassily, Gonz{\'a}lez, Guzm{\'a}n, Menart,
  and Ullah]{arora2022faster}
Raman Arora, Raef Bassily, Tom{\'a}s Gonz{\'a}lez, Crist{\'o}bal Guzm{\'a}n,
  Michael Menart, and Enayat Ullah.
\newblock Faster rates of convergence to stationary points in differentially
  private optimization.
\newblock In \emph{International Conference on Machine Learning}, 2023.

\bibitem[Asi et~al.(2022)Asi, Feldman, and Talwar]{asi2022optimal}
Hilal Asi, Vitaly Feldman, and Kunal Talwar.
\newblock Optimal algorithms for mean estimation under local differential
  privacy.
\newblock \emph{arXiv preprint arXiv:2205.02466}, 2022.

\bibitem[Bassily et~al.(2014)Bassily, Smith, and Thakurta]{bassily2014private}
Raef Bassily, Adam Smith, and Abhradeep Thakurta.
\newblock Private empirical risk minimization: Efficient algorithms and tight
  error bounds.
\newblock In \emph{2014 IEEE 55th annual symposium on foundations of computer
  science}, pages 464--473. IEEE, 2014.

\bibitem[Bassily et~al.(2019)Bassily, Feldman, Talwar, and
  Guha~Thakurta]{bassily2019private}
Raef Bassily, Vitaly Feldman, Kunal Talwar, and Abhradeep Guha~Thakurta.
\newblock Private stochastic convex optimization with optimal rates.
\newblock \emph{Advances in neural information processing systems}, 32, 2019.

\bibitem[Bell et~al.(2020)Bell, Bonawitz, Gasc{\'o}n, Lepoint, and
  Raykova]{bell2020secure}
James~Henry Bell, Kallista~A Bonawitz, Adri{\`a} Gasc{\'o}n, Tancr{\`e}de
  Lepoint, and Mariana Raykova.
\newblock Secure single-server aggregation with (poly) logarithmic overhead.
\newblock In \emph{Proceedings of the 2020 ACM SIGSAC Conference on Computer
  and Communications Security}, pages 1253--1269, 2020.

\bibitem[Berrett and Butucea(2020)]{berrett2020locally}
Thomas Berrett and Cristina Butucea.
\newblock Locally private non-asymptotic testing of discrete distributions is
  faster using interactive mechanisms.
\newblock \emph{Advances in Neural Information Processing Systems},
  33:\penalty0 3164--3173, 2020.

\bibitem[Biswas et~al.(2020)Biswas, Dong, Kamath, and
  Ullman]{biswas2020coinpress}
Sourav Biswas, Yihe Dong, Gautam Kamath, and Jonathan Ullman.
\newblock Coinpress: Practical private mean and covariance estimation.
\newblock \emph{Advances in Neural Information Processing Systems},
  33:\penalty0 14475--14485, 2020.

\bibitem[Bonawitz et~al.(2016)Bonawitz, Ivanov, Kreuter, Marcedone, McMahan,
  Patel, Ramage, Segal, and Seth]{bonawitz2016practical}
Keith Bonawitz, Vladimir Ivanov, Ben Kreuter, Antonio Marcedone, H~Brendan
  McMahan, Sarvar Patel, Daniel Ramage, Aaron Segal, and Karn Seth.
\newblock Practical secure aggregation for federated learning on user-held
  data.
\newblock \emph{arXiv preprint arXiv:1611.04482}, 2016.

\bibitem[Bretagnolle and Huber(1978)]{bretagnolle1978estimation}
Jean Bretagnolle and Catherine Huber.
\newblock Estimation des densit{\'e}s: risque minimax.
\newblock \emph{S{\'e}minaire de probabilit{\'e}s de Strasbourg}, 12:\penalty0
  342--363, 1978.

\bibitem[Bu et~al.(2021)Bu, Gopi, Kulkarni, Lee, Shen, and
  Tantipongpipat]{bu2021fast}
Zhiqi Bu, Sivakanth Gopi, Janardhan Kulkarni, Yin~Tat Lee, Hanwen Shen, and
  Uthaipon Tantipongpipat.
\newblock Fast and memory efficient differentially private-sgd via jl
  projections.
\newblock \emph{Advances in Neural Information Processing Systems},
  34:\penalty0 19680--19691, 2021.

\bibitem[Bun and Steinke(2019)]{bun2019average}
Mark Bun and Thomas Steinke.
\newblock Average-case averages: Private algorithms for smooth sensitivity and
  mean estimation.
\newblock \emph{Advances in Neural Information Processing Systems}, 32, 2019.

\bibitem[Charikar et~al.(2002)Charikar, Chen, and
  Farach-Colton]{charikar2002finding}
Moses Charikar, Kevin Chen, and Martin Farach-Colton.
\newblock Finding frequent items in data streams.
\newblock In \emph{International Colloquium on Automata, Languages, and
  Programming}, pages 693--703. Springer, 2002.

\bibitem[Chen et~al.(2018)Chen, Giannakis, Sun, and Yin]{chen2018lag}
Tianyi Chen, Georgios Giannakis, Tao Sun, and Wotao Yin.
\newblock Lag: Lazily aggregated gradient for communication-efficient
  distributed learning.
\newblock \emph{Advances in Neural Information Processing Systems}, 31, 2018.

\bibitem[Chen et~al.(2020)Chen, Kairouz, and Ozgur]{chen2020breaking}
Wei-Ning Chen, Peter Kairouz, and Ayfer Ozgur.
\newblock Breaking the communication-privacy-accuracy trilemma.
\newblock \emph{Advances in Neural Information Processing Systems},
  33:\penalty0 3312--3324, 2020.

\bibitem[Chen et~al.(2021)Chen, Choquette-Choo, and
  Kairouz]{chen2021communication}
Wei-Ning Chen, Christopher~A Choquette-Choo, and Peter Kairouz.
\newblock Communication efficient federated learning with secure aggregation
  and differential privacy.
\newblock In \emph{NeurIPS 2021 Workshop Privacy in Machine Learning}, 2021.

\bibitem[Chen et~al.(2022{\natexlab{a}})Chen, Choo, Kairouz, and
  Suresh]{chen2022fundamental}
Wei-Ning Chen, Christopher A~Choquette Choo, Peter Kairouz, and Ananda~Theertha
  Suresh.
\newblock The fundamental price of secure aggregation in differentially private
  federated learning.
\newblock In \emph{International Conference on Machine Learning}, pages
  3056--3089. PMLR, 2022{\natexlab{a}}.

\bibitem[Chen et~al.(2022{\natexlab{b}})Chen, Ozgur, and
  Kairouz]{chen2022poisson}
Wei-Ning Chen, Ayfer Ozgur, and Peter Kairouz.
\newblock The poisson binomial mechanism for unbiased federated learning with
  secure aggregation.
\newblock In \emph{International Conference on Machine Learning}, pages
  3490--3506. PMLR, 2022{\natexlab{b}}.

\bibitem[Choquette-Choo et~al.(2022)Choquette-Choo, McMahan, Rush, and
  Thakurta]{choquette2022multi}
Christopher~A Choquette-Choo, H~Brendan McMahan, Keith Rush, and Abhradeep
  Thakurta.
\newblock Multi-epoch matrix factorization mechanisms for private machine
  learning.
\newblock \emph{arXiv preprint arXiv:2211.06530}, 2022.

\bibitem[Diakonikolas et~al.(2019)Diakonikolas, Kamath, Kane, Li, Steinhardt,
  and Stewart]{diakonikolas2019sever}
Ilias Diakonikolas, Gautam Kamath, Daniel Kane, Jerry Li, Jacob Steinhardt, and
  Alistair Stewart.
\newblock Sever: A robust meta-algorithm for stochastic optimization.
\newblock In \emph{International Conference on Machine Learning}, pages
  1596--1606. PMLR, 2019.

\bibitem[Duchi(2016)]{duchi2016lecture}
John Duchi.
\newblock Lecture notes for statistics 311/electrical engineering 377.
\newblock \emph{URL: https://stanford. edu/class/stats311/Lectures/full notes.
  pdf. Last visited on}, 2:\penalty0 23, 2016.

\bibitem[Duchi et~al.(2018)Duchi, Jordan, and Wainwright]{duchi2018minimax}
John~C Duchi, Michael~I Jordan, and Martin~J Wainwright.
\newblock Minimax optimal procedures for locally private estimation.
\newblock \emph{Journal of the American Statistical Association}, 113\penalty0
  (521):\penalty0 182--201, 2018.

\bibitem[Dwork et~al.(2010{\natexlab{a}})Dwork, Naor, Pitassi, and
  Rothblum]{dwork2010differential}
Cynthia Dwork, Moni Naor, Toniann Pitassi, and Guy~N Rothblum.
\newblock Differential privacy under continual observation.
\newblock In \emph{Proceedings of the forty-second ACM symposium on Theory of
  computing}, pages 715--724, 2010{\natexlab{a}}.

\bibitem[Dwork et~al.(2010{\natexlab{b}})Dwork, Naor, Pitassi, Rothblum, and
  Yekhanin]{dwork2010pan}
Cynthia Dwork, Moni Naor, Toniann Pitassi, Guy~N Rothblum, and Sergey Yekhanin.
\newblock Pan-private streaming algorithms.
\newblock In \emph{ics}, pages 66--80, 2010{\natexlab{b}}.

\bibitem[Dwork et~al.(2014)Dwork, Roth, et~al.]{dwork2014algorithmic}
Cynthia Dwork, Aaron Roth, et~al.
\newblock The algorithmic foundations of differential privacy.
\newblock \emph{Found. Trends Theor. Comput. Sci.}, 9\penalty0 (3-4):\penalty0
  211--407, 2014.

\bibitem[Feldman and Talwar(2021)]{feldman2021lossless}
Vitaly Feldman and Kunal Talwar.
\newblock Lossless compression of efficient private local randomizers.
\newblock In \emph{International Conference on Machine Learning}, pages
  3208--3219. PMLR, 2021.

\bibitem[Fowl et~al.(2022)Fowl, Geiping, Czaja, Goldblum, and
  Goldstein]{fowl2022robbing}
Liam~H Fowl, Jonas Geiping, Wojciech Czaja, Micah Goldblum, and Tom Goldstein.
\newblock Robbing the fed: Directly obtaining private data in federated
  learning with modified models.
\newblock In \emph{International Conference on Learning Representations}, 2022.
\newblock URL \url{https://openreview.net/forum?id=fwzUgo0FM9v}.

\bibitem[Girgis et~al.(2021)Girgis, Data, Diggavi, Kairouz, and
  Suresh]{girgis2021shuffled}
Antonious~M Girgis, Deepesh Data, Suhas Diggavi, Peter Kairouz, and
  Ananda~Theertha Suresh.
\newblock Shuffled model of federated learning: Privacy, accuracy and
  communication trade-offs.
\newblock \emph{IEEE journal on selected areas in information theory},
  2\penalty0 (1):\penalty0 464--478, 2021.

\bibitem[Haddadpour et~al.(2020)Haddadpour, Karimi, Li, and
  Li]{haddadpour2020fedsketch}
Farzin Haddadpour, Belhal Karimi, Ping Li, and Xiaoyun Li.
\newblock Fedsketch: Communication-efficient and private federated learning via
  sketching.
\newblock \emph{arXiv preprint arXiv:2008.04975}, 2020.

\bibitem[Hatamizadeh et~al.(2022)Hatamizadeh, Yin, Molchanov, Myronenko, Li,
  Dogra, Feng, Flores, Kautz, Xu, et~al.]{hatamizadeh2022gradient}
Ali Hatamizadeh, Hongxu Yin, Pavlo Molchanov, Andriy Myronenko, Wenqi Li,
  Prerna Dogra, Andrew Feng, Mona~G Flores, Jan Kautz, Daguang Xu, et~al.
\newblock Do gradient inversion attacks make federated learning unsafe?
\newblock \emph{arXiv preprint arXiv:2202.06924}, 2022.

\bibitem[Ivkin et~al.(2019)Ivkin, Rothchild, Ullah, Stoica, Arora,
  et~al.]{ivkin2019communication}
Nikita Ivkin, Daniel Rothchild, Enayat Ullah, Ion Stoica, Raman Arora, et~al.
\newblock Communication-efficient distributed sgd with sketching.
\newblock \emph{Advances in Neural Information Processing Systems}, 32, 2019.

\bibitem[Jhunjhunwala et~al.(2021)Jhunjhunwala, Gadhikar, Joshi, and
  Eldar]{jhunjhunwala2021adaptive}
Divyansh Jhunjhunwala, Advait Gadhikar, Gauri Joshi, and Yonina~C Eldar.
\newblock Adaptive quantization of model updates for communication-efficient
  federated learning.
\newblock In \emph{ICASSP 2021-2021 IEEE International Conference on Acoustics,
  Speech and Signal Processing (ICASSP)}, pages 3110--3114. IEEE, 2021.

\bibitem[Jin et~al.(2019)Jin, Netrapalli, Ge, Kakade, and Jordan]{jin2019short}
Chi Jin, Praneeth Netrapalli, Rong Ge, Sham~M Kakade, and Michael~I Jordan.
\newblock A short note on concentration inequalities for random vectors with
  subgaussian norm.
\newblock \emph{arXiv preprint arXiv:1902.03736}, 2019.

\bibitem[Kairouz et~al.(2021{\natexlab{a}})Kairouz, Liu, and
  Steinke]{kairouz2021distributed}
Peter Kairouz, Ziyu Liu, and Thomas Steinke.
\newblock The distributed discrete gaussian mechanism for federated learning
  with secure aggregation.
\newblock In \emph{International Conference on Machine Learning}, pages
  5201--5212. PMLR, 2021{\natexlab{a}}.

\bibitem[Kairouz et~al.(2021{\natexlab{b}})Kairouz, McMahan, Song, Thakkar,
  Thakurta, and Xu]{kairouz2021practical}
Peter Kairouz, Brendan McMahan, Shuang Song, Om~Thakkar, Abhradeep Thakurta,
  and Zheng Xu.
\newblock Practical and private (deep) learning without sampling or shuffling.
\newblock In \emph{International Conference on Machine Learning}, pages
  5213--5225. PMLR, 2021{\natexlab{b}}.

\bibitem[Kamath and Ullman(2020)]{kamath2020primer}
Gautam Kamath and Jonathan Ullman.
\newblock A primer on private statistics.
\newblock \emph{arXiv preprint arXiv:2005.00010}, 2020.

\bibitem[Kamath et~al.(2020)Kamath, Singhal, and Ullman]{kamath2020private}
Gautam Kamath, Vikrant Singhal, and Jonathan Ullman.
\newblock Private mean estimation of heavy-tailed distributions.
\newblock In \emph{Conference on Learning Theory}, pages 2204--2235. PMLR,
  2020.

\bibitem[Kane and Nelson(2014)]{kane2014sparser}
Daniel~M Kane and Jelani Nelson.
\newblock Sparser johnson-lindenstrauss transforms.
\newblock \emph{Journal of the ACM (JACM)}, 61\penalty0 (1):\penalty0 1--23,
  2014.

\bibitem[Karwa and Vadhan(2017)]{karwa2017finite}
Vishesh Karwa and Salil Vadhan.
\newblock Finite sample differentially private confidence intervals.
\newblock \emph{arXiv preprint arXiv:1711.03908}, 2017.

\bibitem[Kasiviswanathan et~al.(2011)Kasiviswanathan, Lee, Nissim,
  Raskhodnikova, and Smith]{kasiviswanathan2011can}
Shiva~Prasad Kasiviswanathan, Homin~K Lee, Kobbi Nissim, Sofya Raskhodnikova,
  and Adam Smith.
\newblock What can we learn privately?
\newblock \emph{SIAM Journal on Computing}, 40\penalty0 (3):\penalty0 793--826,
  2011.

\bibitem[Liu et~al.(2021)Liu, Kong, Kakade, and Oh]{liu2021robust}
Xiyang Liu, Weihao Kong, Sham Kakade, and Sewoong Oh.
\newblock Robust and differentially private mean estimation.
\newblock \emph{Advances in neural information processing systems},
  34:\penalty0 3887--3901, 2021.

\bibitem[Liu et~al.(2022)Liu, Kong, and Oh]{liu2022differential}
Xiyang Liu, Weihao Kong, and Sewoong Oh.
\newblock Differential privacy and robust statistics in high dimensions.
\newblock In \emph{Conference on Learning Theory}, pages 1167--1246. PMLR,
  2022.

\bibitem[Makarenko et~al.(2022)Makarenko, Gasanov, Islamov, Sadiev, and
  Richt{\'a}rik]{makarenko2022adaptive}
Maksim Makarenko, Elnur Gasanov, Rustem Islamov, Abdurakhmon Sadiev, and Peter
  Richt{\'a}rik.
\newblock Adaptive compression for communication-efficient distributed
  training.
\newblock \emph{arXiv preprint arXiv:2211.00188}, 2022.

\bibitem[McMahan et~al.(2017{\natexlab{a}})McMahan, Moore, Ramage, Hampson, and
  y~Arcas]{mcmahan2017communication}
Brendan McMahan, Eider Moore, Daniel Ramage, Seth Hampson, and Blaise~Aguera
  y~Arcas.
\newblock Communication-efficient learning of deep networks from decentralized
  data.
\newblock In \emph{Artificial intelligence and statistics}, pages 1273--1282.
  PMLR, 2017{\natexlab{a}}.

\bibitem[McMahan et~al.(2022)McMahan, Rush, and Thakurta]{mcmahan2022private}
Brendan McMahan, Keith Rush, and Abhradeep~Guha Thakurta.
\newblock Private online prefix sums via optimal matrix factorizations.
\newblock \emph{arXiv preprint arXiv:2202.08312}, 2022.

\bibitem[McMahan et~al.(2017{\natexlab{b}})McMahan, Ramage, Talwar, and
  Zhang]{mcmahan2017learning}
H~Brendan McMahan, Daniel Ramage, Kunal Talwar, and Li~Zhang.
\newblock Learning differentially private recurrent language models.
\newblock \emph{arXiv preprint arXiv:1710.06963}, 2017{\natexlab{b}}.

\bibitem[Micikevicius et~al.(2017)Micikevicius, Narang, Alben, Diamos, Elsen,
  Garcia, Ginsburg, Houston, Kuchaiev, Venkatesh,
  et~al.]{micikevicius2017mixed}
Paulius Micikevicius, Sharan Narang, Jonah Alben, Gregory Diamos, Erich Elsen,
  David Garcia, Boris Ginsburg, Michael Houston, Oleksii Kuchaiev, Ganesh
  Venkatesh, et~al.
\newblock Mixed precision training.
\newblock \emph{arXiv preprint arXiv:1710.03740}, 2017.

\bibitem[Mironov(2017)]{mironov2017renyi}
Ilya Mironov.
\newblock R{\'e}nyi differential privacy.
\newblock In \emph{2017 IEEE 30th computer security foundations symposium
  (CSF)}, pages 263--275. IEEE, 2017.

\bibitem[Pagh and Thorup(2022)]{pagh2022improved}
Rasmus Pagh and Mikkel Thorup.
\newblock Improved utility analysis of private countsketch.
\newblock \emph{arXiv preprint arXiv:2205.08397}, 2022.

\bibitem[Prasad et~al.(2018)Prasad, Suggala, Balakrishnan, and
  Ravikumar]{prasad2018robust}
Adarsh Prasad, Arun~Sai Suggala, Sivaraman Balakrishnan, and Pradeep Ravikumar.
\newblock Robust estimation via robust gradient estimation.
\newblock \emph{arXiv preprint arXiv:1802.06485}, 2018.

\bibitem[Price and Woodruff(2011)]{price20111+}
Eric Price and David~P Woodruff.
\newblock (1+ eps)-approximate sparse recovery.
\newblock In \emph{2011 IEEE 52nd Annual Symposium on Foundations of Computer
  Science}, pages 295--304. IEEE, 2011.

\bibitem[Rothchild et~al.(2020)Rothchild, Panda, Ullah, Ivkin, Stoica,
  Braverman, Gonzalez, and Arora]{rothchild2020fetchsgd}
Daniel Rothchild, Ashwinee Panda, Enayat Ullah, Nikita Ivkin, Ion Stoica,
  Vladimir Braverman, Joseph Gonzalez, and Raman Arora.
\newblock Fetchsgd: Communication-efficient federated learning with sketching.
\newblock In \emph{International Conference on Machine Learning}, pages
  8253--8265. PMLR, 2020.

\bibitem[Shi et~al.(2019)Shi, Chu, Cheung, and See]{shi2019understanding}
Shaohuai Shi, Xiaowen Chu, Ka~Chun Cheung, and Simon See.
\newblock Understanding top-k sparsification in distributed deep learning.
\newblock \emph{arXiv preprint arXiv:1911.08772}, 2019.

\bibitem[Steinke and Ullman(2015)]{steinke2015between}
Thomas Steinke and Jonathan Ullman.
\newblock Between pure and approximate differential privacy.
\newblock \emph{arXiv preprint arXiv:1501.06095}, 2015.

\bibitem[Stich et~al.(2018)Stich, Cordonnier, and Jaggi]{stich2018sparsified}
Sebastian~U Stich, Jean-Baptiste Cordonnier, and Martin Jaggi.
\newblock Sparsified sgd with memory.
\newblock \emph{Advances in Neural Information Processing Systems}, 31, 2018.

\bibitem[Suliman and Leith(2022)]{suliman2022two}
Mohamed Suliman and Douglas Leith.
\newblock Two models are better than one: Federated learning is not private for
  google gboard next word prediction.
\newblock \emph{arXiv preprint arXiv:2210.16947}, 2022.

\bibitem[Suresh et~al.(2022)Suresh, Sun, Ro, and Yu]{suresh2022correlated}
Ananda~Theertha Suresh, Ziteng Sun, Jae~Hun Ro, and Felix Yu.
\newblock Correlated quantization for distributed mean estimation and
  optimization.
\newblock \emph{arXiv preprint arXiv:2203.04925}, 2022.

\bibitem[Vershynin(2018)]{vershynin2018high}
Roman Vershynin.
\newblock \emph{High-dimensional probability: An introduction with applications
  in data science}, volume~47.
\newblock Cambridge university press, 2018.

\bibitem[Warner(1965)]{warner1965randomized}
Stanley~L Warner.
\newblock Randomized response: A survey technique for eliminating evasive
  answer bias.
\newblock \emph{Journal of the American Statistical Association}, 60\penalty0
  (309):\penalty0 63--69, 1965.

\bibitem[Wen et~al.(2017)Wen, Xu, Yan, Wu, Wang, Chen, and Li]{wen2017terngrad}
Wei Wen, Cong Xu, Feng Yan, Chunpeng Wu, Yandan Wang, Yiran Chen, and Hai Li.
\newblock Terngrad: Ternary gradients to reduce communication in distributed
  deep learning.
\newblock \emph{Advances in neural information processing systems}, 30, 2017.

\bibitem[Xu et~al.(2023)Xu, Zhang, Andrew, Choquette-Choo, Kairouz, McMahan,
  Rosenstock, and Zhang]{xu2023federated}
Zheng Xu, Yanxiang Zhang, Galen Andrew, Christopher~A. Choquette-Choo, Peter
  Kairouz, H.~Brendan McMahan, Jesse Rosenstock, and Yuanbo Zhang.
\newblock Federated learning of gboard language models with differential
  privacy, 2023.

\end{thebibliography}
\newpage
\appendix
\onecolumn
\section{Additional Figures}\label{sec:app-additional-figures}
\begin{table}[h!]
    \centering
    {
    \begin{tabular}{c|c|c|c|c|c}
        \multicolumn{6}{c}{F-EMNIST}\\\hline
        \multirow{2}{*}{\makecell{Relative\\Error, $c_0$}} & \multirow{2}{*}{Metric} & \multicolumn{4}{c}{Noise Multiplier ($z$)} \\\cline{3-6}
        & & 0.1 & 0.2 & 0.3 & 0.5 \\\hline\hline
        \multirow{2}{*}{\makecell{\textbf{N/A,}\\\textbf{Genie}}} & \makecell{Validation\\Accuracy} & 81.23 & 79.70 & 77.68 & 74.92 \\
        & \makecell{Compression\\rate, $r$} & 339 & 341 & 1443 & 2391 \\\hline
        \multirow{2}{*}{0.25} & \makecell{Validation\\Accuracy} & 78.87 & 75.53 & 77.01 & 74.97 \\
         & \makecell{Compression\\rate, $r$} & 101 & 212 & 284 & 332 \\\hline
        \multirow{2}{*}{\textbf{0.1}} & \makecell{Validation\\Accuracy} & 81.34 & 78.53 & 77.81 & 75.36 \\
         & \makecell{Compression\\rate, $r$} & 47 & 94 & 117 & 134 \\\hline
         \multirow{2}{*}{0.05} & \makecell{Validation\\Accuracy} & 81.84 & 79.04 & 78.43 & 75.66 \\
         & \makecell{Compression\\rate, $r$} & 24 & 47 & 59 & 68 \\\hline
         \multirow{2}{*}{0.01} & \makecell{Validation\\Accuracy} & 82.51 & 79.59 & 78.67 & 75.56 \\
         & \makecell{Compression\\rate, $r$} & 5 & 9 & 12 & 13 \\\hline
        
    \end{tabular}
    \caption{\textbf{Adapt Norm is stable with respect to choices in the relative error constant $c_0$}. We observe that as this constant is increased so that the compression error is larger with respect to the privacy error, the model's utility degrades and the compression rates obtained increase, and vice versa; this aligns with our expectations as the total error in the FME estimate increases (decreases). However, all choices of constant below $0.1$ lead to viable models with non-trivial compression rates. Results for F-EMNIST with values displaying the mean across 5 runs. Bolded rows represent configurations used in the rest of our experiments.}
    }
    \label{tab:stability-emnist}
\end{table}

\begin{table}[h!]
    \centering
    {
    \begin{tabular}{c|c|c|c|c|c|c}
        \multicolumn{6}{c}{Shakespeare}\\\hline
        \multirow{2}{*}{\makecell{Relative\\Error, $c_0$}} & \multirow{2}{*}{Metric} & \multicolumn{5}{c}{Noise Multiplier ($z$)} \\\cline{3-7}
        & & 0.1 & 0.2 & 0.3 & 0.5 & 0.7 \\\hline\hline
        \multirow{2}{*}{\makecell{\textbf{N/A,}\\\textbf{Genie}}} & \makecell{Validation\\Accuracy} & 55.41 & 53.91 & 52.48 & 50.44 & 48.68 \\
        & \makecell{Compression\\rate, $r$} & 133 & 210 & 342 & 460 & 511 \\\hline
        \multirow{2}{*}{0.25} & \makecell{Validation\\Accuracy} & 55.71 & 53.93 & 52.26 & 49.12 & 46.96 \\
         & \makecell{Compression\\rate, $r$} & 24 & 64 & 117 & 235 & 367 \\\hline
        \multirow{2}{*}{\textbf{0.1}} & \makecell{Validation\\Accuracy} & 55.77 & 54.17 & 52.57 & 49.87 & 47.84 \\
         & \makecell{Compression\\rate, $r$} & 12 & 29 & 49 & 98 & 154 \\\hline
         \multirow{2}{*}{0.05} & \makecell{Validation\\Accuracy} & 55.89 & 54.29 & 52.82 & 50.23 & 48.12 \\
         & \makecell{Compression\\rate, $r$} & 7 & 17 & 27 & 50 & 75 \\\hline
         \multirow{2}{*}{0.01} & \makecell{Validation\\Accuracy} & 55.83 & 54.28 & 52.85 & 50.66 & 48.87 \\
         & \makecell{Compression\\rate, $r$} & 2 & 5 & 8 & 13 & 18 \\\hline
        
    \end{tabular}
    \caption{\textbf{Adapt Norm is stable with respect to choices in the relative error constant $c_0$}. We observe that as this constant is increased so that the compression error is larger with respect to the privacy error, the model's utility degrades and the compression rates obtained increase, and vice versa; this aligns with our expectations as the total error in the FME estimate increases (decreases). However, all choices of constant below $0.1$ lead to viable models with non-trivial compression rates. Results for Shakespere with values displaying the mean across 5 runs. Bolded rows represent configurations used in the rest of our experiments.}
    }
    \label{tab:stability-shapeskeare}
\end{table}

\begin{table}[h!]
    \centering
    {
    \begin{tabular}{c|c|c|c|c|c}
        \multicolumn{6}{c}{SONWP}\\\hline
        \multirow{2}{*}{\makecell{Relative\\Error, $c_0$}} & \multirow{2}{*}{Metric} & \multicolumn{4}{c}{Noise Multiplier ($z$)} \\\cline{3-6}
        & & 0.1 & 0.3 & 0.5 & 0.7 \\\hline\hline
        \multirow{2}{*}{\makecell{\textbf{N/A,}\\\textbf{Genie}}} & \makecell{Validation\\Accuracy} & 23.77 & 22.36 & 22.07 & 21.33 \\
        & \makecell{Compression\\rate, $r$} & 15 & 33 & 45 & 48 \\\hline
        \multirow{2}{*}{0.25} & \makecell{Validation\\Accuracy} & 23.97 & 22.59 & 22.06 & 21.15 \\
         & \makecell{Compression\\rate, $r$} & 2 & 4 & 12 & 16 \\\hline
        \multirow{2}{*}{\textbf{0.1}} & \makecell{Validation\\Accuracy} & 24.03 & 22.70 & 22.12 & 21.31 \\
         & \makecell{Compression\\rate, $r$} & 1 & 3 & 7 & 10 \\\hline
         \multirow{2}{*}{0.05} & \makecell{Validation\\Accuracy} & 23.93 & 22.39 & 22.14 & 21.35 \\
         & \makecell{Compression\\rate, $r$} & 1 & 2 & 4 & 6 \\\hline
         \multirow{2}{*}{0.01} & \makecell{Validation\\Accuracy} & 24.10 & 22.57 & 22.17 & 21.38 \\
         & \makecell{Compression\\rate, $r$} & 1 & 1 & 2 & 3 \\\hline
        
    \end{tabular}
    \caption{\textbf{Adapt Norm is stable with respect to choices in the relative error constant $c_0$}. We observe that as this constant is increased so that the compression error is larger with respect to the privacy error, the model's utility degrades and the compression rates obtained increase, and vice versa; this aligns with our expectations as the total error in the FME estimate increases (decreases). However, all choices of constant below $0.1$ lead to viable models with non-trivial compression rates. Results for SONWP with values displaying the mean across 5 runs. Bolded rows represent configurations used in the rest of our experiments.}
    }
    \label{tab:stability-sonwp}
\end{table}
\FloatBarrier

\section{Missing Details from Section
\ref{sec:fme}
}

In this section, we provide additional details on FME. 

\subsection{Instance optimal FME with bounded norm of the mean}

\paragraph{Optimal error rate with bounded mean norm.} We consider a problem of DP mean estimation when the data is at the server, but under the assumption that  $\norm{\mu}\leq M$ with a known $M$.
First, we investigate the
statistical complexity of the problem, without DP,
showing a lower bound on the error of $\Omega(G^2\min(M^2,1/n))$ -- see \cref{thm:statistical-lower-bound-mean} in Appendix \ref{app:lower-bounds-statistcal-error}. Secondly, in order to understand the complexity under differential privacy constraint, we study the empirical version of the problem wherein for a fixed dataset, the goal is estimate its empirical mean. Our main result is the following.

\begin{theorem}
\label{thm:lower-bound-dp-mean-estimation-bounded-mean}
For any $d,n \in \bbN, G,M>0$, define the instance class,
$    \hat \cP_1(n,d,G,M) = \bc{\bc{z_i}_{i=1}^n: z_i \in \bbR^d, \norm{z_i} \leq G, \norm{\hat \mu(\bc{z_i}_{i=1}^n)} \leq MG}$
For any $\epsilon = O(1)$ and $2^{-\Omega(n)} \leq \delta \leq \frac{1}{n^{1+\Omega(1)}}$, we have,
   \begin{align*}
         &\min_{\cA: \cA \text{ is }(\epsilon,\delta)-\text{DP}} \max_{D \in \hat \cP_1(n,d,G,M)} \mathbb{E}\norm{\cA(D)-\hat \mu(D)}^2 
         = \Theta\br{G^2\, \min\br{M^2,\frac{d\log{1/\delta}}{n^2\epsilon^2}}}
     \end{align*}
\end{theorem}

We convert the lower bound for the empirical problem to the statistical problem, via a known re-sampling trick (see \cite{bassily2019private}, Appendix C).  
The proof of \cref{thm:lower-bound-dp-mean-estimation-bounded-mean} is based on reduction to the (standard) DP mean estimation wherein the above assumption on the bound on norm of the mean is absent, for which the optimal error is known \citep{steinke2015between, kamath2020primer}.
Our result above shows that this additional restriction of $ \|\hat\mu(\{z_i\}_{i=1}^n)\| \leq MG$ does not make the problem \textit{much} easier in terms of achievable error.
In the interesting regime when $G > M \geq G\sqrt{d\log{1/\delta}}/(n\epsilon)$, this yields the (same) optimal error with potentially significantly smaller communication.

\paragraph{Instance-specific tightness.}
 We show that for \textbf{any} dataset with mean $\hat \mu$, the mean-squared error of count-mean sketch (i.e. $R=1$) with added noise of variance $\sigma^2$
 is exactly $\frac{d-1}{PC}\norm{\hat \mu}^2 + \frac{d\sigma^2}{n^2}$ -- see \cref{lem:sketching-instance-tightness} for the statement. Therefore, $\norm{\hat \mu}$ is the only statistic that controls the MSE of this procedure. 
Adapting to the norm of mean thus provides a tight instance-specific communication complexity for this method.

\subsection{Instance optimal FME with bounded tail-norm of the mean} 
\label{app:missing-tail-norm}

\paragraph{Lower bound on error:} 
The error rate in \cref{thm:dp-mean-estimation-upper-bound-2} matches the rate in Eqn. \eqref{eqn:mean-estimation-optimal-rate}, known to be optimal in the worst case, i.e. general values of tail-norms.
However, it may be possible to do better in special cases.
Towards this, we establish lower bounds for (central) DP mean estimation, where the algorithm designer is given $k<d$ and $\gamma$ and is promised that $\norm{\mu_{\ntail{k}}}\leq \gamma$. 
For this setting, we present a lower bound on error of 
$\Omega\br{G^2\min \bc{1, \gamma^2 + \frac{k}{n^2\epsilon^2} + \frac{1}{n}, \frac{d}{n^2\epsilon^2} + \frac{1}{n}}}$.
The first and third term are standard, achieved by outputting zero and our procedure (or simply, by Gaussian mechanism) 
respectively.
The second term is new, which is absent from the guarantees of our procedure.

To prove the lower bound, we first show,
in Theorem \ref{thm:statistical-lower-bound-tail}, in \cref{app:lower-bounds-statistcal-error}, that the statistical error, without DP, is unchanged, $\Omega\br{\min\br{G^2,\frac{G^2}{n}}}$. 
Secondly, in order to establish the fundamental limits under differential privacy, we study the empirical version of the problem, where the goal is to estimate the empirical mean, and give the following result.

\begin{theorem}
\label{thm:dp-mean-estimation-lower-bound-2}
For any $n,d,k \in \bbN, G,M, \gamma\geq 0$, define the instance class
$ \hat \cP_2 (n,d, G, k, \gamma) =\bc{ \bc{z_i}_{i=1}^n: z_i \in \bbR^d, \norm{z_i}\leq G,
     \norm{\hat \mu(D)_{\ntail{k}}}^2
     \leq \gamma^2}$.
For $\epsilon=O(1)$ and $2^{-\Omega(n)} \leq \delta \leq \frac{1}{n^{1+\Omega(1)}}$, we have
\begin{align*}
    &\min_{\cA: \cA \text{ is } (\epsilon,\delta)\text{-DP}}\max_{D \in \cP(n,d,G,k,\gamma)} \mathbb{E}\norm{\cA(D)-\hat \mu(D)}^2= \Omega\br{G^2\min \bc{1, \gamma^2 + \frac{k\log{1/\delta}}{n^2\epsilon^2}, \frac{d\log{1/\delta}}{n^2\epsilon^2}}}
\end{align*}
\end{theorem}

We convert the lower bound for the empirical problem to the statistical problem, via a known re-sampling trick (see \cite{bassily2019private}, Appendix C).  
We provide the proof of \cref{thm:dp-mean-estimation-lower-bound-2} in \cref{app:lower-bounds-dp}.

\paragraph{Achieving the optimal error with small communication:} 
We show that a simple tweak to our procedure underlying \cref{thm:dp-mean-estimation-upper-bound-2} -- adding a Top-$k$ to the unsketched estimate, with exponentially increasing $k$ -- suffices to achieve the error in \cref{thm:dp-mean-estimation-lower-bound-2}. The procedure is adaptive in $k$, but requires prior knowledge of $\gamma$ and is a \textbf{biased} estimator of the mean. In the following, we abuse notation and write $k_{\text{tail}}(\gamma^2;\mu)$ which corresponds to plugging $g(k) = \gamma^2$ in its definition.

\begin{theorem}
\label{thm:dp-mean-estimation-upper-bound-3}
For any $\gamma>0$, 
\cref{alg:dp-mean-estimation-adapt-server_adapt_tail}
with the same parameter settings as in \cref{thm:dp-mean-estimation-upper-bound-2} except with $k_j=2^j$ and $\overline \gamma_j =16\br{\gamma G + \frac{G\sqrt{\log{8\log{d}/\beta}}}{\sqrt{m}}
    +\sqrt{k_j}\frac{\sigma}{n}}  + \tilde \alpha$
satisfies $(\epsilon,\delta)$-DP. 
With probability at least $1-\beta$, the final output satisfies, 
\begin{align*}
    &\norm{\omean - \pmean}^2 
    = \tilde O\br{G^2\br{\gamma^2 + \frac{1}{n} + \frac{ k_{\text{tail}}(\gamma^2;\mu)\log{1/\delta}}{n^2\epsilon^2}}},
\end{align*}
the total per-client communication complexity is $\tilde O \br{k_{\text{tail}}(\gamma^2;\mu)}$
and number of rounds is $\lfloor \log{d} \rfloor$.
\end{theorem}

We provide the proof of \cref{thm:dp-mean-estimation-upper-bound-3} in \cref{app:proofs_upper_bounds}.

\paragraph{Optimal communication complexity:} 
Next, we investigate the communication complexity of multi-round SecAgg-compatible, potentially biased, schemes with prior knowledge of $k$ and $\gamma$ such that $\norm{\mu_{\ntail{k}}}\leq \gamma$. Our main result is the following.
\begin{theorem}
\label{thm:lower-bound-communication-bounded-mean-3}
Let $d,n,k,K\in \bbN, d\geq 2k, \epsilon, \delta, G,\alpha,\gamma>0$. 
Define 
$\cP_2(d,G,\gamma,k) = \bc{\cD \text{ over } \bbR^d: \norm{z}\leq G, z\sim \cD, \text{ and } \norm{\mu(\cD)_{\ntail{k}}}^2\leq \gamma^2}$.
For any $K$, any protocol $\cA$ in the class of $K$-round, SecAgg-compatible protocols (see \cref{sec:multi-round-protocols} for details)
such that its MSE, $\mathbb{E}_{D\sim \cD^n}[\norm{\cA(D) -  \mu}^2] \leq \alpha^2$ for all $\cD \in \cP_2(d,G,\gamma,k)$, there exists a distribution $\cD \in  \cP_2(d,G,\gamma, k)$, such that on dataset $D\sim \cD^n$, w.p. 1, the total per-client communication complexity is
$\Omega(k\log{Gd/k\alpha})$.
\end{theorem}
Plugging in the error $\alpha^2$ from \cref{thm:dp-mean-estimation-upper-bound-3} establishes that complexity complexity of \cref{thm:dp-mean-estimation-upper-bound-3} is tight upto poly-logarithmic factors. 

\paragraph{Instance-specific tightness.}
We show, in \cref{lem:instance-lower-bound-tail-norm} in \cref{app:proofs_upper_bounds}, that for any $P$, $C = \Omega(Pk)$ and $R=\log{d}$, for any dataset with mean $\hat \mu$, the error of the count median-of-means sketch underlying \cref{thm:dp-mean-estimation-upper-bound-3} with added noise of variance $\sigma^2$, with high probability, is $\Omega\br{G^2\norm{\hat \mu_{\ntail{k}}}^2 + \sigma^2k}$.

\paragraph{Discussion on \cref{thm:dp-mean-estimation-upper-bound-2} vs \cref{thm:dp-mean-estimation-upper-bound-3}:} While the error in procedure defined in \cref{thm:dp-mean-estimation-upper-bound-2} may be larger than that
in \cref{thm:dp-mean-estimation-upper-bound-3}, the former has some attractive properties that make it useful in practice.
Specifically, it is unbiased which is desirable in stochastic optimization and it does not require knowledge of any new hyper-parameter $\gamma$, which is unclear how to adapt to.
Further, the additional top $k$ operation, which is the sole difference between the two procedures, does not seem to provide \textit{significant} benefits in our downstream FL experiments.
Consequently, we use the procedure underlying \cref{thm:dp-mean-estimation-upper-bound-2} in our experimental settings, and defer detailed investigation of the method in \cref{thm:dp-mean-estimation-upper-bound-3} to future work.
\section{Proofs of Error Upper Bounds for FME}
\label{app:proofs_upper_bounds}
\subsection{Useful Lemmas}

In this section, we collect and state lemmas that will be used in the proofs of the main results.
\begin{lemma}
\label{lem:kane-nelson-jl}
\cite{kane2014sparser}
For any $i \in [R]$,
CountSketch matrix $S^{(i)} \in \bbR^{\pads\cols \times d}$ with $P =  \Omega\br{\tau^{-1}\log{1/\beta}}$ and $C = \frac{1}{\tau}$
satisfies that, for any $z \in \bbR^{d}$, with probability at least $1-\beta$,
$\br{1-\tau}\norm{z}^2\leq\norm{S^{(i)}z}^2\leq \br{1+\tau}\norm{z}^2$
\end{lemma}

The result below shows that Count-median-of-means sketch preserves norms and heavy hitters.
\begin{lemma}
\label{lem:countsketch-jl-hh}
Let $0 <\alpha,\tau,\beta<1$ and $k,d \in \bbN$ and $k\leq d$.
For $P = \Theta\br{\tau^{-1}\log{2R/\beta}}, C = \max\br{8Pk,\frac{k}{8P\alpha},\frac{1}{\tau}}$ and $R = \lceil 2\log{2d/\beta} \rceil$, the sketch $S$ satisfies that for any $z \in \bbR^d$, with probability at least $1-\beta$,
\begin{enumerate}
    \item \textbf{JL property}:
    \begin{align*}
        \br{1-\tau}\norm{z}^2\leq\norm{S^{(i)}z}^2\leq \br{1+\tau}\norm{z}^2 \ \forall i \in [R]
    \end{align*}
    \item \textbf{$\ell_\infty$ guarantee}:
    Let $\bar z = U(S(z))$, then,
    \begin{align*}
           \br{\bar z_q - z_q}^2 \leq  \frac{\alpha\norm{z_{\text{tail}(k)}}^2}{k} \ \forall q \in [d]
    \end{align*}
   \item  
   \textbf{Sparse recovery}: 
   \begin{enumerate}
       \item Let 
   $\tilde z = \text{Top}_k(\bar z)$
    \begin{align*}
    \norm{\tilde z - z}_2^2 \leq \br{1+7\sqrt{\alpha})}\norm{z_{\text{tail}(k)}}^2
\end{align*}
\item Let 
   $\hat z = \text{Top}_{2k}(\bar z)$
    \begin{align*}
    \norm{\hat z - z}_2^2 \leq \br{1+10\alpha}\norm{z_{\text{tail}(k)}}^2
\end{align*}
   \end{enumerate}
where $\norm{z_{\text{tail}(k)}}
    = \min_{\tilde z: \norm{\tilde z}_0\leq k
    }\norm{z-\tilde z}_2$.
\end{enumerate}
\end{lemma}

\begin{proof}

$P=1$ corresponds to the standard Countsketch based method. For $P>1$, we modify the analysis as follows.
The first part directly follows from the result of \cite{kane2014sparser}, stated as \cref{lem:kane-nelson-jl}, which gave a construction of a sparse Johnson-Lindenstrauss transform based on CountSketch. 

We now proceed to the second part.
Let $H_k \subseteq [d]$ denote the indices of $k$ largest co-ordinates of $z$.
We first consider the estimate of $z_q$ based on one row ($i$-th row), which is give as,
\begin{align*}
    \bar z^{(i)}_q &= \br{{S^{(i)}}^\top S^{(i)}z}_q  = \frac{1}{P}\sum_{j=1}^d \sum_{p=1}^P s^{(i)}_{p}(j) s^{(i)}_{p}(q) \mathbbm{1}\br{h^{(i)}_{p}(j) = h_{p}^{(i)}(q)}z_j \\
    & = z_q + \underbrace{
    \frac{1}{P}\sum_{j=1, j\neq q}^d \sum_{p=1}^P s^{(i)}_{p}(j) s^{(i)}_{p}(q) \mathbbm{1}\br{h^{(i)}_{p}(j) = h_{p}^{(i)}(q)}z_j}_{=:E}
\end{align*}
From moment assumptions, $\mathbb{E}[E]=0$ which gives us that
$\mathbb{E}[ \bar z^{(i)}_q] = z_q$. We now decompose the error into two terms,
\begin{align*}
    E_{H_k} &= \frac{1}{P}\sum_{j\in {H_k}, j\neq q} \sum_{p=1}^P s^{(i)}_{p}(j) s^{(i)}_{p}(q) \mathbbm{1}\br{h^{(i)}_{p}(j) = h_{p}^{(i)}(q)}z_j \\
    E_{T_k} &= \frac{1}{P}\sum_{j\not \in {H_k}, j\neq q} \sum_{p=1}^P s^{(i)}_{p}(j) s^{(i)}_{p}(q) \mathbbm{1}\br{h^{(i)}_{p}(j) = h_{p}^{(i)}(q)}z_j
\end{align*}
By construction,
\begin{align*}
    \mathbb{P}\br{E_{H_k} \neq 0} \leq \mathbb{P}\br{\exists j \in {H_k}, p \in [P]: h^{(i)}_{p}(j) = h^{(i)}_{p}(q)} \leq \frac{Pk}{\cols} \leq \frac{1}{8}
\end{align*}
for $C \geq 8 Pk$. For the other term, we directly compute its second moment,
\begin{align*}
    \mathbb{E}[E_{T_k}^2] & = \mathbb{E}\left[\br{\frac{1}{P}\sum_{j\not \in {H_k}, j\neq q} \sum_{p=1}^P s^{(i)}_{p}(j) s^{(i)}_{p}(q) \mathbbm{1}\br{h^{(i)}_{p}(j) = h_{p}^{(i)}(q)}z_j}^2\right]\\
    & = \frac{1}{P^2} \sum_{j\not \in {H_k}, j\neq q}  \sum_{p=1}^P\mathbb{E}\left[\mathbbm{1}\br{h^{(i)}_{p}(j) = h_{p}^{(i)}(q)}\right]z_j^2 + 0 \\
    & \leq \frac{1}{P\cols}\norm{z_{\text{tail}(k)}}^2
\end{align*}
Therefore for $C \geq \frac{k}{8P\alpha}$, from Chebyshev's inequality, $E_{T_k} \leq \frac{\sqrt{\alpha}\norm{z_{\text{tail}(k)}}}{\sqrt{k}}$ with probability at least $7/8$. 
Combining, for $C = \max\br{\frac{k}{8P\alpha},8Pk}$, with probability at least $3/4$,
\begin{align*}
    \br{\bar z_q^{(i)} - z_q}^2 \leq \frac{\alpha\norm{z_{\text{tail}(k)}}^2}{k}
\end{align*}
With $\rows  =\lfloor 2 \log{1/\beta'} \rfloor$, and $\bar z_q= \text{median}\bc{\bar z_q^{(i)}}_{i=1}^\rows$; using the standard boosting guarantee based on a median of estimates, we have, with probability at least $1-\beta'$, 
\begin{align*}
    \br{\bar z_q - z_q}^2 \leq  \frac{\alpha\norm{z_{\text{tail}(k)}}^2}{k}.
\end{align*}
Setting $\beta' = \frac{\beta}{2d}$ and doing a union bound on all coordinates gives that with probability at least $1-\beta$, for all $ q \in [d]$, 
\begin{align*}
    \br{\bar z_q - z_q}^2 \leq  \frac{\alpha\norm{z_{\text{tail}(k)}}^2}{k}.
\end{align*}
The third item in the Lemma statement is based on converting the 
$\ell_\infty$ to $\ell_2$ guarantee; specifically, instantiating Lemma \ref{lem:top-k-lemma} with $\Delta^2 = \frac{\alpha\norm{z_{\text{tail}(k)}}^2}{k}$ gives us,

\begin{align*}
    \norm{\tilde z - z}^2 &\leq \br{1+5\alpha + 2\sqrt{\alpha}}\norm{z_{\text{tail}(k)}}^2 \\
    & \leq \br{1+ 7\sqrt{\alpha}}\norm{z_{\text{tail}(k)}}^2.
\end{align*} 
Similarly, from Lemma \ref{lem:top-k-lemma}, we get, 
\begin{align*}
    \norm{\hat z - z}^2 \leq \br{1+ 10\alpha}\norm{z_{\text{tail}(k)}}^2,
\end{align*} 
which completes the proof.
\end{proof}

\begin{lemma}
\label{lem:top-k-lemma}
Let $z \in \bbR^d$ and $\bar z \in \bbR^d$ such that $\norm{\bar z-z}_\infty^2 \leq \Delta^2$. Define $\tilde z = \text{Top}_k(\bar z)$ and $\hat z = \text{Top}_{2k}(\bar z)$. Then, 
\begin{align*}
    &\norm{z-\bar z}_2^2   \leq 5k \Delta^2 + \norm{z_{\text{tail}(k)}}^2 + 2\Delta \sqrt{k} \norm{z_{\text{tail}(k)}}\\
    &\norm{z - \hat z}_2^2 \leq 10 k \Delta^2 +  \norm{z_{\text{tail}(k)}}^2 
\end{align*}
\end{lemma}
\begin{proof}
This is a fairly standard fact, though typically not presented in the form above (see, for instance \cite{price20111+}). We give a proof for completeness and its use in many parts of the manuscript.
Let $I$ denote the indices of top $k$ co-ordinates of $z$ and let $\tilde I$ and $\hat I$ denote the top $k$ and top $2k$ co-ordinates of $\bar z$. We proceed with the first part of the lemma. Note that,
\begin{align}
\label{eqn:top-k-main}
    \norm{z- \tilde z}^2 = \norm{z - \bar z_{\tilde I}}^2 = \norm{(z - \bar z)_{\tilde I} }^2 + \norm{z_{\tilde I^c}}^2
\end{align}
where $\tilde I^c$ denotes the complement set of $\tilde I$. 
The first term is bounded as $\norm{(z - \bar z)_{\tilde I} }^2 \leq \abs{\tilde I}\norm{z-\bar z}_\infty^2 = k\Delta^2$. We decompose the second term $\norm{z_{\tilde I^c}}^2$ as follows, 
\begin{align*}
  \norm{z_{\tilde I^c}}^2 =\norm{z_{I \backslash \tilde I}}^2 + \norm{z_{I^c}}^2 - \norm{z_{\tilde I \backslash I}}^2
\end{align*}
Let $a = \max_{i\in I\backslash \tilde I} z_i$  and $b = \min_{i \in \tilde I \backslash I} z_i$. From the $\ell_\infty$ guarantee and the fact that the sets $I$ and $\tilde I$ are indices of top $k$ elements of $z$ and $\tilde z$ respectively, we have that $(a-b)^2 \leq 4\norm{\bar z-z}_\infty^2 \leq 4\Delta^2$. We note consider two cases (a). $b\leq 0$ and $b>0$. In the first case, we have that $a \leq 2\Delta$. Plugging all these simplifications in Eqn. \eqref{eqn:top-k-main}, we get, 
\begin{align*}
    \norm{z-\tilde z}^2 &\leq k \Delta^2 + \norm{z_{I^c}}^2 +\norm{z_{I \backslash \tilde I}}^2 - \norm{z_{\tilde I \backslash I}}^2 \\
    & \leq k \Delta^2 + \norm{z_{I^c}}^2 + \abs{I \backslash \tilde I} a^2 \\
    & \leq k \Delta^2 + \norm{z_{I^c}}^2 + 4k \Delta^2  = \norm{z_{I^c}}^2  + 5k\Delta^2\\
\end{align*}
In the other case (b). $b>0$, we get,
\begin{align*}
    \norm{z-\tilde z}^2 &\leq k \Delta^2 + \norm{z_{I^c}}^2 +\norm{z_{I \backslash \tilde I}}^2 - \norm{z_{\tilde I \backslash I}}^2 \\
    & \leq k \Delta^2 + \norm{z_{I^c}}^2 + \abs{I \backslash \tilde I} a^2 -  \abs{\tilde I \backslash \tilde I}b^2\\
    & = k \Delta^2 + \norm{z_{I^c}}^2 + \abs{I \backslash \tilde I}\br{ a^2 -  b^2}\\
    & \leq k \Delta^2 + \norm{z_{I^c}}^2 + \abs{I \backslash \tilde I} \br{(b+2\Delta)^2 - b^2}\\
    & \leq k \Delta^2 + \norm{z_{I^c}}^2 + 2\Delta \abs{I \backslash \tilde I} b + 4\abs{I \backslash \tilde I}\Delta^2 \\
    & \leq k \Delta^2 + \norm{z_{I^c}}^2 + 2\Delta \sqrt{\abs{I \backslash \tilde I}} \norm{z_{I^c}} + 4k\Delta^2 \\
    & \leq 5k \Delta^2 + \norm{z_{I^c}}^2 + 2\Delta \sqrt{k} \norm{z_{I^c}}
\end{align*}
where the first equality follows since $\abs{I}=\abs{\tilde I}=k$, and the second last inequality follows since $\sqrt{\abs{I \backslash \tilde I}}b \norm{z_{I^c}}$ and $b$ is value of the the minimal element in $I\backslash \tilde I$.

 Finally, note that $\norm{z_{I^c}} = \norm{z_{\text{tail}(k)}}$. Hence, combining the two cases yields the statement claimed in the lemma.

The second part follows analogously. In particular, repeating the initial steps gives us
\begin{align*}
     \norm{z-\hat z}^2 &\leq 2k \Delta^2 + \norm{z_{I^c}}^2 +\norm{z_{I \backslash \hat I}}^2 - \norm{z_{\hat I \backslash I}}^2.
\end{align*}
Defining $a$ and $b$ as before and considering the two cases give us that in the first case (a). $b\leq 0$, we have,
\begin{align*}
     \norm{z-\hat z}^2 &\leq 2k \Delta^2 + \norm{z_{I^c}}^2 +\norm{z_{I \backslash \hat I}}^2\\
      &\leq 2k \Delta^2 + \norm{z_{I^c}}^2 +\abs{I \backslash \hat I}a^2\\
      & \leq  2k \Delta^2 + \norm{z_{I^c}}^2 +4k\Delta^2\\
      & = 6k \Delta^2 + \norm{z_{I^c}}^2.
\end{align*}
In the second case, let $\hat k = \abs{I \backslash \hat I}$; note that $\abs{\hat I\backslash I} =k+\hat k$. Therefore, 
\begin{align*}
     \norm{z-\hat z}^2 &\leq 2k \Delta^2 + \norm{z_{I^c}}^2 +\norm{z_{I \backslash \hat I}}^2 - \norm{z_{\hat I \backslash I}}^2 \\
     & \leq 2k\Delta^2 + \norm{z_{I^c}}^2 + \hat k a^2 - (k+\hat k)b^2 \\
    & \leq 2k\Delta^2 +  \norm{z_{I^c}}^2  + 4\hat k \Delta^2  + 4 \hat k b \Delta- k b^2 \\
    & = -k\br{b - 2\frac{\hat k \Delta}{k}}^2 + \frac{4\hat k^2\Delta^2}{k} +2k\Delta^2 + 4\hat k\Delta^2  +  \norm{z_{I^c}}^2 \\
    & \leq 10 k \Delta^2 +  \norm{z_{I^c}}^2. 
\end{align*}
where in the last inequality, we used that $\hat k \leq k$.
Combining the two cases finishes the proof.
\end{proof}

The following result gives a heavy hitter guarantee for count-median-of-means sketch with noise.
\begin{lemma}
\label{lem:cs-noise-hh}
Let $z \in \bbR^d$, $\rows,\cols,\pads$ be parameters of Count-median-of-means sketch $S$ and let $\xi \sim \cN(0, \sigma^2\bbI)$. Let $\bar z = U\br{S(z)+\xi}$. For $\cols \geq \max\br{8Pk,\frac{k}{8P\alpha}}$ and 
$\rows = \lceil 2\log{2d/\beta}\rceil$ gives us that
with probability at least
$1-\beta$,
\begin{align*}
    \norm{\bar z - z}_\infty^2 \leq
    \frac{2\alpha \norm{z_{\text{tail}(k)}}^2}{k} + \sigma^2
\end{align*}
Let $\tilde z = \text{Top}_k (\bar z)$ and $\hat z = \text{Top}_{2k} (\bar z)$. 
With probability at least $1-\beta$, we have
\begin{align*}
   & \norm{\tilde z - z}_2^2 \leq 
   2\br{1+7\sqrt{\alpha}}\norm{z_{\text{tail}(k)}}^2 + 5k\sigma^2\\
    & \norm{\hat z - z}_2^2  \leq \br{1+20\alpha}\norm{z_{\text{tail}(k)}}^2 + 10k\sigma^2
\end{align*}
\end{lemma}
\begin{proof}
The proof extends the analysis in \cite{pagh2022improved} which was limited to count-sketch $(P=1)$. Specifically, we apply Lemma 3.4 in \cite{pagh2022improved} plugging in an estimate of
one row error obtained from our sketch. In the proof of Lemma \ref{lem:countsketch-jl-hh}, for $\cols \geq \max\br{8Pk,\frac{k}{8P\alpha}}$, for one row estimate $\bar z^{(i)}$, we have that with probability $\geq 3/4$,
\begin{align*}
    \norm{\bar z^{(i)} - z}^2_{\infty} \leq
    \frac{\alpha \norm{z_{\text{tail}(k)}}^2}{k}
\end{align*}
For a fixed coordinate $q$, let $\bar z_q^{(i)}$ denote its estimate from the $i$-th row of the count sketch. We have,
\begin{align*}
    \bar z_q^{(i)} &= 
    \underbrace{z_q +\frac{1}{P}\sum_{j=1, j\neq q}^d \sum_{p=1}^P s^{(i)}_{p}(j) s^{(i)}_{p}(q) \mathbbm{1}\br{h^{(i)}_{p}(j) = h_{p}^{(i)}(q)}z_j}_{=:A} + \underbrace{\frac{1}{\sqrt{P}}\sum_{p=1}^P 
    s^{(i)}_{p}(q) \xi_{p}^i}_{=:B}
\end{align*}
The reason that the second term has $\sqrt{P}$ in the denominator and not $P$ is because the noise is added after sketching, which itself performs division by $\sqrt{P}$ operation.
Now, $A$ is the original non-private estimate and $B$ is the additional noise. 
Since $\xi_p^j \sim \cN(0, \sigma^2)$ and $s_{p}^{(i)}$ are random signs, the random variable 
$B \sim \cN(0, \sigma^2)$.
Applying Lemma 3.4 from \cite{pagh2022improved} gives us that 
for $\Delta^2 > \frac{\alpha\norm{z_{\text{tail}(k)}}^2}{k}$,
the mean estimate $\bar z$, with probability $1-2d\exp{-\frac{\rows\min\br{1,\frac{\Delta^2}{ \sigma^2}}}{2}}$, satisfies
\begin{align*}
    \norm{\bar z^{(i)} - z}^2_{\infty} \leq \Delta^2
\end{align*}
We set $\Delta^2 = \max\br{\frac{2\alpha\norm{z_{\text{tail}(k)}}^2}{k}, \sigma^2}$ -- note that with this setting, the success probability is at least  $1-2d\exp{-\frac{R}{2}}$. Hence, setting $R = \lceil 2\log{2d/\beta}\rceil$ yields the claimed $\ell_\infty$ guarantee with probability $1-\beta$. For the Top $k$ and Top $2k$ guarantees, we apply Lemma \ref{lem:top-k-lemma} with the above $\Delta$, which yields,
\begin{align*}
    \norm{\tilde z- z}^2 & \leq 5k\br{\frac{2\alpha \norm{z_{\text{tail}(k)}}^2}{k} + \sigma^2}+\norm{z_{\text{tail}(k)}}^2 + 2\sqrt{k}\norm{z_{\text{tail}(k)}}\br{\frac{2\sqrt{\alpha \norm{z_{tail}(k)}}}{\sqrt{k}} + \sigma} \\
    & = \br{1+14\sqrt{\alpha}}\norm{z_{\text{tail}(k)}}^2 + 5k\sigma^2 + 2\sqrt{k} \norm{z_{\text{tail}(k)}}\sigma \\
    & \leq 2\br{1+7\sqrt{\alpha}}\norm{z_{\text{tail}(k)}}^2 + 5k\sigma^2,
\end{align*}
where the last inequality follows from AM-GM inequality.
Similarly, from Lemma \ref{lem:top-k-lemma}, we get
\begin{align*}
     \norm{\hat z- z}^2 \leq \br{1+20\alpha}\norm{z_{\text{tail}(k)}}^2 + 10k\sigma^2,
\end{align*}
which finishes the proof.
\end{proof}

In the following, we give a lower bound on the error of count median-of-means sketch with noise, for any instance.

\begin{lemma}
\label{lem:instance-lower-bound-tail-norm}
Let
$\rows,\cols,\pads$ be parameters of Count-median-of-means sketch $S$ and let $\xi_i = [\xi^1_i, \xi^2_i, \cdots, \xi^R_i]^\top$, where $\xi_i^j \sim \cN(0, \sigma^2\bbI_{PC})$ for all $i \in [n], j \in [R]$.
Consider any dataset $D=\bc{z_1,z_2,\cdots, z_n}$ and let $\tilde  \mu = \text{Top}_k\br{U\br{\frac{1}{n}\sum_{i=1}^nS(z_i)+\xi_i}}$. For 
$\cols \geq 100Pk$,
and $R=\lceil 800\log{2d/\beta}\rceil$
 with probability at least $1-\beta$,
\begin{align*}
    \norm{\tilde \mu - \hat \mu}^2 \geq  \frac{\norm{\hat \mu_{\text{tail}(k)}}^2}{2} +
    0.125\sigma^2k
\end{align*}
\end{lemma}
\begin{proof}
Since the sketching operation is linear, it suffices to show the result for sketching the mean of the dataset, $\hat \mu: = \frac{1}{n}\sum_{i=1}^n z_i$.
For a fixed coordinate $q$, let $\bar  \mu_q^{(i)}$ denote its estimate from the $i$-th row of the count-median-of-means sketch. We have,
\begin{align*}
    \br{\bar  \mu_q^{(i)} - \hat  \mu_q}^2 &= \br{
    \underbrace{\frac{1}{P}\sum_{j=1, j\neq q}^d \sum_{p=1}^P s^{(i)}_{p}(j) s^{(i)}_{p}(q) \mathbbm{1}\br{h^{(i)}_{p}(j) = h_{p}^{(i)}(q)} \hat \mu_j}_{=:A} + \underbrace{\frac{1}{\sqrt{P}}\sum_{p=1}^P 
    s^{(i)}_{p}(q) \xi_{p}^i}_{=:B}}^2 \\
    &= A^2 +B^2 + 2AB \geq B^2 - 2\abs{A}\abs{B}
\end{align*}
Since $\xi_p^j \sim \cN(0, \sigma^2)$ and $s_{p}^{(i)}$ are random signs, the random variable 
$B \sim \cN(0, \sigma^2)$. 
Using the CDF table for the normal distribution,
we have that with 
\begin{align*}
    \mathbb{P}\br{B^2 \geq 4\sigma^2} =2\mathbb{P}_{X\sim \cN(0,1)}\br{X\geq 1}\leq 0.04
\end{align*}
Similarly,
\begin{align*}
        \mathbb{P}\br{B^2\leq 0.25\sigma^2} 
    = 1-   \mathbb{P}\br{B^2> 0.25\sigma^2}
    =1-2\mathbb{P}_{X\sim \cN(0,1)}\br{X> 0.5} \leq 0.4
\end{align*}
With $C\geq \max\br{100Pk, \frac{k}{8P\alpha}}$ as in the proof of Lemma \ref{lem:countsketch-jl-hh}, $\abs{A}\leq  \frac{\sqrt{\alpha}\norm{\hat \mu}_{\text{tail}(k)}}{\sqrt{k}}$ with probability at least $0.99$. Hence, with probability at least $0.55$, we have that
\begin{align*}
    \br{\bar \mu_q^{(i)} - \hat \mu_q}^2 \geq 0.25\sigma^2 -\frac{ 4    \sigma\sqrt{\alpha} \norm{\hat \mu_{\text{tail}(k)}}}{\sqrt{k}} =: \varepsilon^2
\end{align*}
We now argue amplification for the median: Define $I_q^i = \mathbbm{1}\br{\hat \mu_q \leq \bar \mu_q^i \leq \hat \mu_q + \varepsilon}$. Note that $\mathbb{E}I_q^i \leq 0.45$. Then,
\begin{align*}
    \mathbb{P}\br{\hat \mu_q \leq \bar \mu_q \leq \hat \mu_q + \varepsilon} \leq \mathbb{P}\br{\sum_{i=1}^RI_q^i \geq \frac{R}{2}} \leq \mathbb{P}\br{\sum_{i=1}^R\br{I_q^i -\mathbb{E}I_q^i}  \geq 0.05R} \leq \exp{-\frac{R}{800}} = \frac{\beta}{2d}
\end{align*}
by setting of $R$. Now, similarly, we have 
\begin{align*}
    \mathbb{P}\br{\hat \mu_q \geq \bar \mu_q \geq \hat \mu_q - \varepsilon} \leq \frac{\beta}{2d}
\end{align*}
Combining both gives us that with probability at least $1-\beta$, for all $q \in [d]$
\begin{align}
\label{eqn:countsketch-lower-bound}
 \br{\bar \mu_q  - \hat \mu_q}^2 \geq \varepsilon^2
\end{align}

Now, let 
$\tilde  \mu = \text{Top}_k\br{U\br{\frac{1}{n}\sum_{i=1}^nS(z_i)+\xi_i}}$
and let $\tilde I$ be the set of coordinates achieving its Top $k$. We have,
\begin{align*}
    \norm{\tilde  \mu- \hat \mu}^2  = \norm{ \tilde \mu_{\tilde I} - \hat \mu_{\tilde I} + \hat \mu_{\tilde I^c}}^2 = \norm{\br{\tilde \mu - \hat \mu}_{\tilde I}}^2 + \norm{\hat \mu_{{\tilde I}^c}}^2 \geq \norm{\br{\tilde \mu - \hat \mu}_{\tilde I}}^2 + \norm{\hat \mu_{\text{tail}(k)}}^2
\end{align*}

For the other term, from Eqn. \eqref{eqn:countsketch-lower-bound}, we have
\begin{align*}
    \norm{\br{\tilde \mu - \hat \mu}_{\tilde I}}^2 \geq k\varepsilon^2 
\end{align*}
Combining, we get,
\begin{align*}
    \norm{\tilde \mu - \hat \mu}^2 \geq  \norm{\hat \mu_{\text{tail}(k)}}^2 + k\varepsilon^2 =  \norm{\hat \mu_{\text{tail}(k)}}^2 
    + 0.25\sigma^2k
    - 4    \sigma\sqrt{k}\sqrt{\alpha} \norm{\hat \mu_{\text{tail}(k)}}
\end{align*}
Consider two cases: 

\paragraph{Case 1}: $\norm{ \hat \mu_{\text{tail}(k)}} \geq 8\sigma\sqrt{k}\sqrt{\alpha}$. Note that
\begin{align*}
    \norm{ \hat \mu_{\text{tail}(k)}}^2  -  4    \sigma\sqrt{k}\sqrt{\alpha} \norm{ \hat \mu_{\text{tail}(k)}} \geq \frac{\norm{ \hat \mu_{\text{tail}(k)}}^2}{2} \iff \norm{ \hat \mu_{\text{tail}(k)}} \geq 8\sigma\sqrt{k}\sqrt{\alpha}
\end{align*}
In this case, the bound becomes, 
\begin{align*}
    \norm{\tilde \mu -  \hat \mu}^2 \geq  \frac{\norm{ \hat \mu_{\text{tail}(k)}}^2}{2} + 0.25\sigma^2k
\end{align*}

\paragraph{Case 2}: $\norm{ \hat \mu_{\text{tail}(k)}} \leq \frac{0.125\sigma\sqrt{k}}{4\sqrt{\alpha}}$. Note that
\begin{align*}
  0.25\sigma^2k -  4\sigma\sqrt{k}\sqrt{\alpha} \norm{ \hat \mu_{\text{tail}(k)}} \geq  0.125\sigma^2k \iff \norm{ \hat \mu_{\text{tail}(k)}} \leq \frac{0.125\sigma\sqrt{k}}{4\sqrt{\alpha}}
\end{align*}
In this case, the bound becomes, 
\begin{align*}
    \norm{\tilde  \mu -  \hat \mu}^2 \geq  \norm{ \hat \mu_{\text{tail}(k)}}^2 + 0.125\sigma^2k
\end{align*}
We now set $\alpha$ to make sure the two cases exhaust all the possibilities. In particular, 
\begin{align*}
    8\sigma\sqrt{k}\sqrt{\alpha} = \frac{0.125\sigma\sqrt{k}}{4\sqrt{\alpha}} \iff \alpha = \frac{1}{128}
\end{align*}
Combining, this gives us, 
\begin{align*}
    \norm{\tilde \mu -  \hat \mu}^2 \geq  \frac{\norm{ \hat \mu_{\text{tail}(k)}}^2}{2} + 0.125\sigma^2k
\end{align*}
which completes the proof.
\end{proof}

In the following, we compute the mean-squared error of count-mean sketch with noise, for any instance.
\begin{lemma}
\label{lem:sketching-instance-tightness}
For the count-median-of-means sketching operation described in Section \ref{sec:adaptivity_tail_norm}, with $R=1$ and $\xi_i \sim \cN(0, \sigma^2\bbI_{PC}), i \in [n]$, we have that for any dataset $D=\bc{z_i}_{i=1}^n$ with $z_i\in \bbR^d$
\begin{align*}
    \mathbb{E}\norm{U\br{\frac{1}{n}\sum_{i=1}^n S(z_i)+\xi_i} -z}^2  = \frac{d-1}{PC} \norm{\hat \mu(D)}^2 + d\sigma^2
\end{align*}
\end{lemma}
\begin{proof}
Since the sketching operation is linear, it suffices to show the result for sketching the mean of the dataset, $\hat \mu: = \frac{1}{n}\sum_{i=1}^n z_i$. Further, since $R=1$, the sketching and unsketching operation simplifies and we get,
\begin{align*}
    \mathbb{E}\norm{\br{S^{(1)}}^\top\br{ S^{(1)}\hat \mu+\xi} - \hat \mu}^2 &=  \mathbb{E}\norm{\br{S^{(1)}}^\top S^{(1)}\hat \mu - \hat \mu}^2  + \mathbb{E}\norm{\br{S^{(1)}}^\top S^{(1)} \xi}^2
\end{align*}
For the first term, we look at the error in every coordinate -- recall $h^{(1)}_{j}: [d]\rightarrow [C]$ denotes the bucketing hash function and $C$ and $P$ denotes the number of columns and rows of the Count-mean sketch. We have
\begin{align*}
    \mathbb{E}\br{\br{\br{S^{(1)}}^\top S^{(1)}\hat \mu}_j - \hat \mu_j}^2 &= \frac{1}{P^2}\sum_{j=1}^P \sum_{q=1, q\neq j}^d\mathbb{E} \mathbbm{1}\br{h^{(1)}_{j}(q)=j} \hat \mu_q^2 \\
    &= \frac{1}{PC}\sum_{q=1, q\neq j}^d \hat \mu_q^2
\end{align*}
where in the above, we use that $\mathbb{P}\br{h^{(1)}_{j}(q)=j} = \frac{1}{C}$.
Hence, 
\begin{align*}
     \mathbb{E}\norm{\br{S^{(1)}}^\top S^{(1)}\hat \mu - \hat \mu}^2  = \frac{d-1}{PC}\norm{\hat \mu}^2
\end{align*}
For the other term, by direct computation, we have that,
\begin{align*}
    \mathbb{E}\norm{\br{S^{(1)}}^\top S^{(1)} \xi}^2 &= \sum_{j=1}^d \mathbb{E}\br{\br{S^{(1)}}^\top S^{(1)} \xi}_j^2\\
    &=\sum_{j=1}^d \mathbb{E}\br{\sum_{q=1}^d\br{\br{S^{(1)}}^\top S^{(1)}}_{jq} \xi_q}^2\\
    &=\sum_{j=1}^d \mathbb{E}\br{\br{S^{(1)}}^\top S^{(1)}}_{jj}^2 \xi_j^2\\
    &=\sum_{j=1}^d \mathbb{E}\br{\br{S^{(1)}}^\top S^{(1)}}_{jj}^2 \mathbb{E}\xi_j^2\\
   &=  d\sigma^2 
\end{align*}
where the above uses the fact that the sketching matrix and the Gaussian noise vector are independent and the the variance of the diagonal entries of $\br{S^{(1)}}^\top S^{(1)}$ are $1$.
Combining the above yields the statement in the lemma.
\end{proof}

The following result shows that count median-of-means sketch, (even) with added noise, is an unbiased estimator.
\begin{lemma}
\label{lem:unbiased-count-median-of-means}
For the count-median-of-means sketching and unsketching operations described in Section \ref{sec:adaptivity_tail_norm},
for any $P,R,C\geq 1$, $\sigma^2\geq 0$, any $z\in \bbR^d$, we have that $$\mathbb{E}[U(S(z) + \xi)] = z$$ where $\xi = [\xi^{(1)},\xi^{(2)}, \cdots \xi^{(R)}]^\top$, $\xi^{(i)} \sim \cN(0, \sigma^2\bbI_{PC})$ for $i\in [R]$.
\end{lemma}
\begin{proof}
We have that $j$-th co-ordinate is,
$U(S(z) + \xi)_j = \text{Median}\br{\bc{\br{{{S^{(i)}}^\top\br{S^{(i)}z + \xi^{(i)}}}}_j}_{i=1}^R}$.
Let $\overline z^{(i)} = \bc{{{S^{(i)}}^\top\br{S^{(i)}z + \xi^{(i)}}}}$ and let $\overline z = U(S(z) + \xi)$. Note that    $\overline z_j = \text{Median}\br{\bc{\overline z^{(1)}_j, \overline z^{(2)}_j, \cdots \overline z^{(R)}_j}}$.
Further, the $j$-th co-ordinate of $ \bar z^{(i)}$ is,
\begin{align*}
    \bar z_j^{(i)} &= 
    \underbrace{z_j +\frac{1}{P}\sum_{q=1, q\neq j}^d \sum_{p=1}^P s^{(i)}_{p}(q) s^{(i)}_{p}(j) \mathbbm{1}\br{h^{(i)}_{p}(q) = h_{p}^{(i)}(j)}z_q}_{=:A^{(i)}_j} + \underbrace{\frac{1}{\sqrt{P}}\sum_{p=1}^P 
    s^{(i)}_{p}(j) \xi_{p}^{(i)}}_{=:B^{(i)}_j}
\end{align*}
Hence, 
\begin{align*}
    \overline z_j &= \text{Median}\br{\bc{\overline z^{(1)}_j, \overline z^{(2)}_j, \cdots \overline z^{(R)}_j} } = \text{Median}\br{\bc{z_j +  A^{(1)}_j + B^{(1)}_j , z_j +  A^{(2)}_j + B^{(2)}_j, \cdots, z_j +  A^{(R)}_j + B^{(R)}_j} } \\
    & =  z_j + \text{Median}\br{\bc{A^{(1)}_j + B^{(1)}_j , A^{(2)}_j + B^{(2)}_j, \cdots,  A^{(R)}_j + B^{(R)}_j} }
\end{align*}
Thus, $\mathbb{E}[\overline z_j] =  z_j + \mathbb{E}[\text{Median}(\{A^{(1)}_j + B^{(1)}_j , A^{(2)}_j + B^{(2)}_j, \cdots,  A^{(R)}_j + B^{(R)}_j\})]$.
Observe that the random variables $A^{(i)}_j$ and $B^{(i)}_j$ and thus $A^{(i)}_j+B^{(i)}_j$ are symmetric about zero by construction. 
Therefore, 
\begin{align*}
    &\mathbb{E}[\text{Median}(\{A^{(1)}_j + B^{(1)}_j , A^{(2)}_j + B^{(2)}_j, \cdots,  A^{(R)}_j + B^{(R)}_j\})] \\
    &= \mathbb{E}[\text{Median}(\{-A^{(1)}_j - B^{(1)}_j , -A^{(2)}_j - B^{(2)}_j, \cdots,  -A^{(R)}_j - B^{(R)}_j\})]\\
    & = - \mathbb{E}[\text{Median}(\{A^{(1)}_j + B^{(1)}_j , A^{(2)}_j + B^{(2)}_j, \cdots,  A^{(R)}_j + B^{(R)}_j\})]
\end{align*}
Hence, $\mathbb{E}[\text{Median}(\{A^{(1)}_j + B^{(1)}_j , A^{(2)}_j + B^{(2)}_j, \cdots,  A^{(R)}_j + B^{(R)}_j\})] = 0$. Repeating this for all co-ordinates completes the proof.
\end{proof}

In the following, we show that count median-of-means sketch, with added noise, is an unbiased estimator, even when its size is estimated using data. 
\begin{lemma}
\label{lem:unbiased-sketch-size-data}
Let $S:\bbR^d \rightarrow \bbR^{R\times PC}$ and $U:\bbR^{R\times PC} \rightarrow \bbR^d$ be the count-median-of-means sketching and unsketching operations described in Section \ref{sec:adaptivity_tail_norm}.
Let $z  \in \bbR^d$ be a random variable with mean $\mu \in \bbR^d$.
For any functions $f_1, f_2, f_3$ with range $\bbN$, random variable $\tilde \xi, \sigma^2 \geq 0$,  such that sketch-size $P = f_1(z,\tilde \xi), R = f_2(z,\tilde \xi), C=f_3(z,\tilde \xi)$, we have that
$$\mathbb{E}[U(S(z) + \xi)] = \mu$$ where $\xi = [\xi^{(1)},\xi^{(2)}, \cdots \xi^{(R)}]^\top$, $\xi^{(i)} \sim \cN(0, \sigma^2\bbI_{PC})$ for $i\in [R]$.
\end{lemma}
\begin{proof}
The proof follows simply from the law of total expectation. We have that,
\begin{align*}
    \mathbb{E}[U(S(z) + \xi)] = \mathbb{E}_{z,\tilde \xi}[ \mathbb{E}_{S, \xi}[U(S(z) + \xi) | \tilde \xi, z] = \mathbb{E}_{z,\tilde \xi}[z| \tilde \xi, z] = \mu
\end{align*}
where the second equality follows from Lemma \ref{lem:unbiased-count-median-of-means}.
\end{proof}

We state the guarantee for AboveThreshold mechanism with the slight modification that we want to stop when the query output is below (as opposed to above) a threshold $\gamma$.
\begin{lemma}
(\cite{dwork2010differential}, Theorem 3.4)
\label{lem:above-threshold}
Given a sequence of $T$ $B$-sensitive queries $\bc{q_t}_{t=1}^T$, dataset $D$ and a threshold $\gamma$, the AboveThreshold mechanism guarantees that with probability at least $1-\beta$,
\begin{enumerate}
    \item If the algorithm halts at time $t$, then $q_t(D) \leq \gamma + \frac{8B\br{\log{T} + \log{2/\beta}}}{\epsilon}$.
    \item If the algorithm doesn't halt at time $t$, then $q_t(D) \geq \gamma - \frac{8B\br{\log{T} + \log{2/\beta}}}{\epsilon}$.
\end{enumerate}
\end{lemma}

\subsection{Proof of Theorem \ref{thm:dp-mean-estimation-upper-bound-1}}

We start with the privacy analysis. Note that in the first round, only the second sketch is used, and in the second round, only the first sketch is used. In both cases, the clipping operation ensures that the sensitivity is bounded.
To elaborate, in the first round, as in \cite{chen2022fundamental}, the sensitivity of $\nu_1 = \frac{1}{n}\sum_{c=1}^nQ_1^{(c), {\text{clipped}}}$ is at most $\frac{2B}{n}$. Similarly, in the second round, let $\tilde \nu_2$ and $\tilde \nu_2'$ be the norm estimates on neighbouring datasets. The sensitivity is bounded as,
\begin{align*}
    \abs{\norm{\tilde \nu_2}_2 -  \norm{\tilde \nu_2'}_2} \leq \norm{\tilde \nu_2-\tilde \nu_2'}_2 = \frac{1}{n}\norm{\text{clip}_B(\tilde Q_2^{(1))} - \text{clip}_B(\tilde Q_2^{'(1)})}_2 \leq \frac{2B}{n}
\end{align*}

Hence, using the guarantees of Gaussian and Laplace mechanism \cite{dwork2014algorithmic}, with the stated noise variances, the first two rounds satisfy $\br{\frac{\epsilon}{2},\delta}$ and $\br{\frac{\epsilon}{2},0}$-DP respectively. Finally, applying standard composition, we have that the algorithm satisfies $(\epsilon,\delta)$-DP.

The unbiasedness claim follows from linearity of expectation and  \cref{lem:unbiased-sketch-size-data} wherein $\tilde \xi$ is the random second sketch $\tilde S$ used to set $C$ in the sketch size.

We now proceed to the utility analysis.
Let $\hat \mu_1$ and $\hat \mu_2$ denote that empirical means of clients data sampled in the first and second rounds. From concentration of empirical mean of i.i.d bounded random vectors (see Lemma 1 in \cite{jin2019short}), we have that with probability at least $1-\frac{\beta}{4}$, for $j \in \bc{1,2}$, 
\begin{align}
\label{eqn:thm-one-vec-concentration}
    \norm{\hat \mu_j - \mu}^2 \leq \frac{2B^2\log{16/\beta}}{n}
\end{align}

Now, from linearity and $\ell_2$-approximation property of CountSketch with the prescribed sketch size (see Lemma \ref{lem:countsketch-jl-hh}),
we have that with probability $1-\frac{\beta}{4}$
\begin{align}
\label{eqn:norm-estimate-approx}
    \bar n_1^2 = \norm{\tilde S_1\br{\emean_1}}^2 \in \left[\frac{1}{2},\frac{3}{2}\right]\norm{\emean_1}^2.
\end{align}

Define $\bar n_1 := \norm{\tilde \nu_1}$, $\hat n_1 = \text{clip}_B(\bar n_1) +\text{Laplace}(\tilde \sigma)$, and $\tilde n_1:= \hat n_1 + \overline \gamma$.
Note that from the setting of $B$ and the above $\ell_2$-approximation guarantee, with probability at least $1-\frac{\beta}{4}$, no clipping occurs.
Further, from concentration of Laplace random variables (see Fact 3.7 in \cite{dwork2014algorithmic}), we have that with probability at least $1-\beta/4$,

\begin{align}
\label{eqn:thm-one-concentration-of-norm-estimate}
\abs{\hat n_1 -\overline n_1} \leq \tilde \sigma\log{8/\beta} = \frac{2B\log{8/\beta}}{n\epsilon}
\end{align}

Therefore, with probability at least $1-\frac{3\beta}{4}$, we have,
\begin{align}
\nonumber
    \tilde n_1 = \hat n_1 + \overline \gamma &\geq \overline  n_1 - \abs{\overline n_1 - \hat n_1} + \overline \gamma \geq \frac{1}{2}\norm{\hat \mu_1} - \frac{2B\log{8/\beta}}{n\epsilon} + \overline \gamma \\
    & \geq \frac{1}{2}\norm{\hat \mu_2} - \frac{1}{2}\norm{\hat \mu_1 - \hat \mu_2}  - \frac{2B\log{8/\beta}}{n\epsilon} + \overline \gamma \geq \frac{1}{2}\norm{\hat \mu_2}
    \label{eqn:norm-adapt-n2-bound}
\end{align}

where the last inequality follows from the setting of $\overline \gamma$.
We now decompose the error of the output $\overline \mu :=\omean_2$ as,
\begin{align*}
    \norm{\overline \mu - \mu}^2 \leq 2\norm{\overline \mu - \hat \mu_2}^2 + 2\norm{\hat \mu_2 - \mu}^2 = 2\norm{\overline \mu - \hat \mu_2}^2 + \frac{4B^2\log{16/\beta}}{n},
\end{align*}
where the last inequality holds with probability at least $1-\beta/4$ from Eqn. \eqref{eqn:thm-one-vec-concentration}.
Let $\xi_2 \sim \cN(0,\sigma^2\bbI_{PC})$ be the Gaussian noise added to the sketch.
For the first term above, 
\begin{align}
\nonumber
    \norm{\overline \mu - \mu_2}^2 \leq \norm{U_2(\nu_2) - \hat \mu_2}^2 &= \norm{U_2\br{\frac{1}{n}\sum_{c=1}^n \text{clip}_B(Q_2^{(c)})+\xi_2}-\hat \mu_2}^2 \\ 
    \nonumber
   &= \norm{U_2\br{\frac{1}{n}\sum_{c=1}^n Q_2^{(c)}+\xi_2}-\hat \mu_2}^2
\end{align}
where the last equality follows from the $\ell_2$-approximation property of count sketch (Lemma \ref{lem:countsketch-jl-hh}) which, with the setting of $B$, ensures that no clipping occurs with probability at least $1-\beta/4$.

Now, since we are only using one row of the sketch, the sketching and unsketching operations simplify to yield,
\begin{align*}
    \norm{\overline \mu - \mu}^2 &= \norm{\br{S_2}^{\top}\br{S_2\hat \mu_2 +\xi_2}-\hat \mu_2}^2.
\end{align*}

We 
now use Lemma \ref{lem:sketching-instance-tightness} to get, \begin{align*}
   &\mathbb{E}\norm{\overline \mu - \hat \mu_2}^2 = \mathbb{E}\left[
    \frac{d-1}{PC_2}\norm{\hat \mu_2}^2\right] + \frac{d\sigma^2}{n^2} \\
   &= \frac{d-1}{P}\mathbb{E}\left[
    \frac{\norm{\hat \mu_2}^2}{C_2} \mathbb{P}\br{C_2\geq \min\br{\frac{n^2\epsilon^2 }{\log{1/\delta}},nd}\frac{\norm{\hat \mu_2}^2}{2PG^2}}
    +\frac{\norm{\hat \mu_2}^2}{C_2} \mathbb{P}\br{C_2> \min\br{\frac{n^2\epsilon^2 }{\log{1/\delta}},nd}\frac{\norm{\hat \mu_2}^2}{2PG^2}}
    \right] + \frac{d\sigma^2}{n^2}  \\
    & \leq \frac{G^2}{n}+ \frac{dG^2\log{1/\delta}}{n^2\epsilon^2} + d\beta G^2 + \frac{d\sigma^2}{n^2} \\
    & = O\br{\frac{G^2}{n}+\frac{G^2\log{1/\delta}}{n^2\epsilon^2}} 
\end{align*}

where the first inequality follows from the setting of $C_2$ and Eqn \eqref{eqn:norm-adapt-n2-bound} which gives us that $C_2
\geq
\min\br{\frac{n^2\epsilon^2 }{\log{1/\delta}},nd}\frac{\norm{\hat \mu_2}^2}{2PG^2}$ with probability at least $1-\frac{3\beta}{4}$ and that $C_2, P\geq 1$; the last inequality follows from setting of $\beta \leq
\frac{\log{1/\delta}}{n^2\epsilon^2}$.

The communication complexity is $ \tilde P \tilde C + PC_2$.
Note that 
$C_2 =\max\br{\min\br{\frac{n^2\epsilon^2}{\log{1/\delta}},nd}\frac{\tilde n_1^2}{G^2P}, 2}$
and,
\begin{align*}
    \tilde n_1 &=\hat n_1+ \overline\gamma \leq \overline n_1 + \abs{\overline n_1 - \hat n_1}+ \overline\gamma \\
    & \leq \frac{3}{2}\norm{\hat \mu_1} + \frac{2B\log{8/\beta}}{n\epsilon} + \overline \gamma \\
    &\leq \frac{3}{2}\norm{\mu} + \frac{3}{2}\norm{\hat \mu_1 - \mu} + \frac{2B\log{8/\beta}}{n\epsilon} + \overline \gamma \\
    &\leq \frac{3}{2}MG + \frac{3\sqrt{B\log{16/\beta}}}{\sqrt{n}} + \frac{2B\log{8/\beta}}{n\epsilon} + \overline \gamma \leq \frac{3}{2}MG + 4\overline \gamma
\end{align*}
where the first and third inequality holds with probability at least $1-\frac{3\beta}{4}$ from Eqn. \eqref{eqn:thm-one-vec-concentration}
\eqref{eqn:norm-estimate-approx} and
\eqref{eqn:thm-one-concentration-of-norm-estimate}, and the last inequality follows from setting of $\overline \gamma$.

This gives a total communication complexity of $  \max\br{4P, 16\min\br{\frac{n^2\epsilon^2 }{\log{1/\delta}},nd}
\frac{M^2G^2+\overline \gamma^2}{G^2}}$ with probability at least $1-\beta$. Plugging in the values of $P$ and $\overline \gamma$ gives the claimed statement.

\subsection{Proof of \texorpdfstring{\cref{thm:dp-mean-estimation-upper-bound-2}}{}}
We start with the privacy analysis. There are two steps in each round of interaction which accesses data: sketching the vectors and error estimate. For the first access, from the clipping operation, the $\ell_2$ sensitivity of each row of the combined sketch $\nu_j = \frac{1}{n}\sum_{c=1}^{n}\text{clip}_B(Q^{(c)}_j)$ is bounded by $\frac{2B}{n}$. 
We apply the Gaussian mechanism guarantee in terms of  R\'enyi Differential Privacy (RDP) \cite{mironov2017renyi} together with composition over rows and number of rounds. 
Converting the RDP guarantee to approximate DP guarantee gives us that
with the stated noise variance, the quantity $\bc{\nu_{j}}_{j=1}^{\lfloor\log{d}\rfloor}$ satisfies $\br{\frac{\epsilon}{2},\delta}$-DP. For the second, we compute the $\ell_2$ sensitivity of error estimate as follows. Given two neighbouring datasets $D$ and $D'$ such that w.l.o.g. the datapoint $z_1$ in $D$ is replaced by $z_1'$ in $D'$. 
Let $\tilde \nu_j$ and $\tilde \nu_j$ be the quantities for $D$ and $D'$ respectively. Since $\omean_j$ is private, we fix it and compute sensitivity as, 
\begin{align*}
    \abs{\norm{\tilde S_j(\omean_j) - \bar \nu_j} - \norm{\tilde S_j(\omean_j) - \bar \nu_j'}} &\leq \norm{\bar \nu_j - \bar \nu_j'} \\
    & \leq  \frac{1}{n}\sum_{i=1}^{\tilde R}\norm{\text{clip}_B(\tilde S_j^{(i)} (z_1)) - \text{clip}_B(\tilde S_j^{(i)} (z_1'))} \\
    & \leq \frac{2B\tilde R}{n} = \frac{2B}{n}
\end{align*}
where the first and second steps follow from triangle inequality, the third from the fact that clipped vectors have norm at most $B$ and the last from setting of $\tilde R$.
From the AboveThreshold guarantee (Lemma \ref{lem:above-threshold}),
 the prescribed settings of noise added to the threshold, of standard deviation $\tilde \sigma$ and to the query $\overline e_j$, of standard deviation $2\tilde \sigma$,
 imply the error estimation steps satisfies $\br{\frac{\epsilon}{2},0}$-DP. 
 We remark that in our case, the threshold $\overline \gamma_j$ changes for $j$-th query, but in standard AboveThreshold, the threshold is fixed. However, note that $\overline \gamma_j = \overline\gamma \text{(some  fixed value)} + 16\sqrt{k_j}\sigma \text{(changing)}$. This changing part of the threshold can be absorbed in the query itself, without changing its sensitivity, thereby reducing it to standard AboveThrehsold with fixed threshold.
 Finally, combining the above DP guarantees using basic composition of differential privacy gives us $(\epsilon,\delta)$-DP.

The unbiasedness claim follows from linearity of expectation and  \cref{lem:unbiased-sketch-size-data} wherein $\tilde \xi$ is the random second sketch $\tilde S$ used to set $C$ in the sketch size.

We now proceed to the utility proof.  
The proof consists of two parts. First, we show that when the algorithm stops, it guarantees that the error of the output is small. Then, we give a high probability bound on the stopping time.

We start with the first part.
Recall that $\mu$ is the true mean and let $\hat \mu_j$ denote the empirical mean of the cohort selected in step $j$ of the algorithm.
We first decompose the error into statistical and empirical error as follows; for any $j \in \lfloor \log{d} \rfloor$
\begin{align}
\label{eqn:proof-mean-estimation-decompostion}
    \norm{\bar \mu - \mu}^2  \leq 2 \norm{\hat \mu_j -  \mu}^2 + 2\norm{\bar \mu - \hat \mu_j}^2
\end{align}

We bound the first term $\norm{\hat \mu_j - \mu}^2$ by standard concentration arguments (see Lemma 1 in \cite{jin2019short}).
Specifically, with probability at least $1-\frac{\beta}{4}$, for all $j \in \lfloor \log{d} \rfloor$, we have,
\begin{align}
\label{eqn:vector-concentration}
    \norm{\hat \mu_j -  \mu}^2 \leq \frac{2B^2\log{8\log{d}/\beta}}{n}
\end{align}

We now bound the second term in Eqn. \eqref{eqn:proof-mean-estimation-decompostion}.
Let $\hat e_j$ be the error in sketching with $S_j$, defined as, $\hat e_j = \norm{\omean_j - \emean_j}$.
Note that $\bar e_j$, defined in Algorithm \ref{alg:dp-mean-estimation-adapt-server_adapt_tail}, is an estimate of $\hat e_j$. Specifically, fixing the random $\bc{S_j}_{j}$, from linearity and $\ell_2$-approximation property of CountSketch with the prescribed sketch size (see Lemma \ref{lem:countsketch-jl-hh})
we have that with probability $1-\frac{\beta}{4}$, for all $j \in \lfloor \log{d} \rfloor$, we have,
\begin{align}
\label{eqn:error-estimate-approx}
    \bar e_j^2 = \norm{\tilde S_j\br{\omean_j} - \tilde \nu_j}^2 = \norm{\tilde S_j\br{\omean_j - \hat \mu_j}}^2 \in \left[\frac{1}{2},\frac{3}{2}\right]\norm{\omean_j - \emean_j}^2 = \left[\frac{1}{2},\frac{3}{2}\right]\hat e_j^2
\end{align}
Let $\hat j$ be the guess on which the algorithm stops.
Using the utility guarantee of AboveThreshold mechanism 
(Lemma \ref{lem:above-threshold}),
with probability at least $1-\frac{\beta}{4}$,
\begin{align}
\label{eqn:above-th-error}
    \bar e_{\hat j} \leq \overline \gamma_{\hat j} +\tilde \alpha
\end{align}
where $\tilde \alpha = \frac{32B\br{\log{\lfloor \log{d} \rfloor} + \log{8/\beta}}}{n\epsilon}$
are as stated in the Theorem statement.
Combining Eqn. \eqref{eqn:error-estimate-approx} and Eqn. \eqref{eqn:above-th-error}, we have that when the algorithm halts,
with probability at least $1-\frac{\beta}{4}$,
\begin{align*}
    \bar e_{\hat j}^2 \leq \br{\overline \gamma_{\hat j}+\tilde \alpha}^2 \implies \hat e_{\hat j}^2 \leq 3\br{\overline \gamma_{\hat j}+\alpha}^2
\end{align*}
To compute the error of the output $\overline \mu$, from Eqn. \eqref{eqn:proof-mean-estimation-decompostion}, and above, we have that with probability at least $1-\beta$,
\begin{align}
\label{eqn:error-bound-to-start-proof-of-thm3}
    \norm{\bar \mu - \mu}^2 &\leq \frac{16G^2\log{8\log{d}/\beta}}{n}+ 6\br{\overline \gamma_{\hat j}+\tilde \alpha}^2 \\
    \label{eqn:error-bound-thm2}
& 
= O\br{\frac{G^2\log{8\log{d}/\beta}}{n} +
\frac{dG^2\log{1/\delta}\br{\log{8d/\beta}}^3}{n^2\epsilon^2}+
\frac{G^2\br{\log{\lfloor \log{d} \rfloor} + \log{8/\beta}}^2}{n^2\epsilon^2}}
\end{align}

We now give a high-probability bound on the stopping time.
Given a vector $z \in \bbR^d$, define $\norm{z_\text{tail(k)}}_2
    = \min_{\tilde z: \norm{\tilde z}_0\leq k
    }\norm{z-\tilde z}_2$.
The error is bounded as,
\begin{align}
    \hat e_j^2  &= \norm{\text{Top}_{{k_j}}\br{U_j
    \br{\frac{1}{n}\sum_{c=1}^n     \text{clip}_{B}(Q_j^{(c)})}} 
    - \emean_j}^2 \\
    &= \norm{U_j
    \br{\frac{1}{n}\sum_{c=1}^n     Q_j^{(c)}} 
    - \emean_j}^2 \\
    & \leq  d\norm{U_j(S_j(\emean_j) + \xi_j) - \emean_j}_\infty^2  \\
    &\leq 
    32\frac{d}{k_j}\norm{\br{\hat \mu_j}_{\text{tail}({k_j})}}^2 +
    \frac{10d\sigma^2}{n^2}
\end{align}
where the second equality holds using the JL property in Lemma \ref{lem:countsketch-jl-hh} which shows that with probability at least $1-\frac{\beta}{4}$,
norms of all vectors are preserved upto a relative tolerance $\in [\frac{1}{2},\frac{3}{2}]$, hence no clipping is done in all rounds.
The bound on the second term follows from the sketch recovery guarantee in Lemma \ref{lem:cs-noise-hh} (with $\alpha =1$ in the lemma).

Let $\gamma^2 = \max\br{\frac{1}{n},\frac{\sigma^2d}{n^2G^2}}$
and $g(k) = \frac{\gamma^2}{d} k-\frac{8\log{8\log{d}/\beta}}{n}$. 
By construction, our strategy 
uses guesses $\underline j$ and $\overline j$ such that $2^{\underline j} \leq k_{\text{tail}}(g(k);\mu) \leq 2^{\overline j} $, and $k_{\overline j} = 2^{\overline j}  = 2\cdot 2^{\underline j} \leq 2 k_{\text{tail}}(g(k);\mu)$.

Let $I$ be the set of indices corresponding tail$(k_{\text{tail}}(g(k);\mu))$ of $\mu$ 
i.e. the $d-k_{\text{tail}}(g(k);\mu)$ smallest coordinates of $\mu$  
and let $\mu_I$ be a vector such that $(\mu_I)_i = \mu_i$ if $i \in I$, otherwise $(\mu_I)_i=0$.
From monotonicity of the errors, we have that 
\begin{align*}
      \norm{\br{\hat \mu_j}_{\text{tail}(k(\overline j)}}^2\leq \norm{\br{\hat \mu_j}_{\text{tail}(\abs{I})}}^2
    \leq \norm{\br{\hat \mu_j}_{I}}^2 &\leq 2\norm{\br{\emean_j}_I - \mu_I}^2 + 2\norm{\pmean_{\text{tail}(k_{\text{tail}}(g(k);\mu))}}^2 \\ 
    & \leq 2\norm{\br{\emean_j -  \pmean}_I}^2 + 
    2\norm{\pmean_{\text{tail}(k_{\text{tail}}(g(k);\mu))}}^2 \\
    & \leq 2\norm{\emean_j -  \mu}^2 + 
    2\norm{\pmean_{\text{tail}(k_{\text{tail}}(g(k);\mu))}}^2 \\
    & \leq \frac{4B^2\log{8\log{d}/\beta}}{n} + 
    2\norm{\pmean_{\text{tail}(k_{\text{tail}}(g(k);\mu))}}^2  \\
    & \leq
    \frac{2\gamma^2G^2 k_{\text{tail}}(g(k);\mu)}{d}
\end{align*}
where the second last inequality follows from vector concentration bound in Eqn. \eqref{eqn:vector-concentration} which holds with probability at least $1-\frac{\beta}{4}$ for all $j \in \lfloor \log{d} \rfloor$, and the last inequality holds from the definition of $g(k)$.

From a union bound over the success of sketching and error estimation and using Eqn. \eqref{eqn:error-estimate-approx}, 
with probability at least $1-\beta$, 
\begin{align}
    \nonumber
    \hat e_{\overline j}^2 \leq 
    32\frac{d}{k_{\overline j}}\norm{\br{\hat \mu_j}_{\text{tail}(k(\overline j))}}^2 + \frac{10\sigma^2d}{n^2}
    \implies \bar e_{\overline j}^2 &\leq \frac{3}{2}\hat e_{\overline j}^2 \\
    \nonumber
    &\leq 48\frac{d}{k_{\overline j}} \norm{\br{\hat \mu_j}_{\text{tail}(k(\overline j))}}^2  + \frac{15d\sigma^2}{n^2} \\
    \label{eqn:bound-on-error}
    & \leq
    96\gamma^2G^2
    + \frac{15 d\sigma^2}{n^2} \leq 121 \gamma^2 G^2
\end{align}

Finally, from the Above Threshold guarantee (using contra-positive of second part of Lemma \ref{lem:above-threshold}), since 
$\bar e_{\overline j} \leq 
\overline \gamma -\tilde \alpha = 
15\gamma G$,
with probability at least $1-\frac{\beta}{4}$, it halts.

The communication complexity now follows by the setting of the two sketch sizes.
The total communication for sketches $\bc{S_j}_j$ is
\begin{align}
    \label{eqn:comm-one}
    \sum_{j=1}^{\overline j} \rows \pads \cols_j &\leq 
    \sum_{j=1}^{\overline j} \br{8RP^2 +\frac{R}{4}}2^j
    = \Theta\br{k_{\text{tail}}(g(k);\mu)\log{8d\log{d}/\beta}\text{log}^2\br{16\log{8d\log{d}/\beta}\log{d}/\beta}}
\end{align}
where we use that $\overline j\leq 2k_{\text{tail}}(g(k);\mu)$, with holds with probability at least $1-\beta$, as argued above.
Similarly, the total communication for sketches $\bc{\tilde S_j}_j$, is 
\begin{align}
  \label{eqn:comm-two}
    \sum_{j=1}^{\overline j} \tilde \rows \tilde \cols \tilde \pads = \overline j 2\tilde P^2 
    = \Theta\br{\log{k_{\text{tail}}(g(k);\mu)}\log{4d\log{d}/\beta}}
\end{align}

The total communication is the sum of the two terms in Eqn. \eqref{eqn:comm-one} and \eqref{eqn:comm-two} which is dominated by the first term.
This completes the proof.

\subsection{Proof of \texorpdfstring{\cref{thm:dp-mean-estimation-upper-bound-3}}{}}

The privacy analysis follows as in \cref{thm:dp-mean-estimation-upper-bound-2}.
However, we note that in our application of AboveThreshold here, the threshold $\overline \gamma_j$ changes for $j$-th query, but in standard AboveThreshold, the threshold is fixed. 
In our case, $\overline \gamma_j = \overline\gamma \text{(\textit{some  fixed value})} + 16\sqrt{k_j}\sigma \text{(\textit{changing})}$ -- this \textit{changing} part of the threshold can be absorbed in the query itself, without changing its sensitivity, thereby reducing it to standard AboveThrehsold with fixed threshold.

We now proceed to the utility proof, which again consists of two parts.
First, we show that when the algorithm stops, it guarantees that the error of the output is small. Secondly, we give a high probability bound on the stopping time.

We start with the first part.
The proof is identical to that of \cref{thm:dp-mean-estimation-upper-bound-3} up to Eqn. \eqref{eqn:error-bound-to-start-proof-of-thm3}.
Recall that $\mu$ is the true mean and $\hat \mu_j$ denotes the empirical mean of the cohort selected in step $j$ of the algorithm.
Let ${k_j} = 2^j$ and $\hat e_j$ be the error in sketching with $S_j$, defined as, $\hat e_j = \norm{\omean_j - \emean_j}$.
Further, let $\hat j$ be the guess on which the algorithm stops.
From Eqn. \eqref{eqn:error-bound-to-start-proof-of-thm3} in \cref{thm:dp-mean-estimation-upper-bound-3},
With probability at least $1-\beta$,
\begin{align}
\nonumber
    \norm{\bar \mu - \mu}^2 &\leq \frac{16G^2\log{8\log{d}/\beta}}{n}+ 6\br{\overline \gamma_{\hat j}+2\tilde \alpha}^2 \\
    \label{eqn:error-bound-thm3}
   & \leq \frac{16G^2\log{8\log{d}/\beta}}{n}+ 6\br{16\br{\gamma G+ \frac{G\sqrt{\log{8\log{d}/\beta}}}{\sqrt{n}}
+\sqrt{k_{\hat j}}\frac{\sigma}{n}}+2\tilde \alpha}^2 
\end{align}
We now give a high probability bound on $\hat j$, where the algorithm stops, which we will plug-in in the above bound.
Given a vector $z \in \bbR^d$, define $\norm{z_\text{tail(k)}}_2
    = \min_{\tilde z: \norm{\tilde z}_0\leq k
    }\norm{z-\tilde z}_2$.
The error is bounded as,
\begin{align*}
    \hat e_j^2  &= \norm{\text{Top}_{{k_j}}\br{U_j
    \br{\frac{1}{n}\sum_{c=1}^n     \text{clip}_B(Q_j^{(c)})}} 
    - \emean_j}^2 \\
    &= \norm{\text{Top}_{{k_j}}\br{U_j
    \br{\frac{1}{n}\sum_{c=1}^n     Q_j^{(c)}}}
    - \emean_j}^2 \\
    & =  \norm{\text{Top}_{{k_j}}\br{U_j(S_j(\emean_j) + \xi_j)} - \emean_j}^2  \\
    &\leq 
    32\norm{\br{\hat \mu_j}_{\text{tail}({k_j})}}^2 +
    \frac{10\sigma^2{k_j}}{n^2}
\end{align*}
where the second equality holds using the JL property in Lemma \ref{lem:countsketch-jl-hh} which shows that with probability at least $1-\frac{\beta}{4}$,
norms of all vectors are preserved upto a relative tolerance $\in [\frac{1}{2},\frac{3}{2}]$, hence no clipping is done in all rounds.
The bound on the second term follows from the sketch recovery guarantee in Lemma \ref{lem:cs-noise-hh} (with $\alpha =1$ in the lemma).

By construction, our strategy 
uses guesses $\underline j$ and $\overline j$ such that $2^{\underline j} \leq k_{\text{tail}}(\gamma^2;\mu) \leq 2^{\overline j} $, and $k_{\overline j} = 2^{\overline j}  = 2\cdot 2^{\underline j} \leq 2 k_{\text{tail}}(\gamma^2;\mu)$.

Let $I$ be the set of indices corresponding tail$(k_{\text{tail}}(\gamma^2;\mu))$ of $\mu$ i.e. the $d-k_{\text{tail}}(\gamma^2;\mu)$ smallest coordinates of $\mu$  and let $\mu_I$ be a vector such that $(\mu_I)_i = \mu_i$ if $i \in I$, otherwise $(\mu_I)_i=0$.
We have that 
\begin{align*}
      \norm{\br{\hat \mu_j}_{\text{tail}(k(\overline j)}}^2\leq \norm{\br{\hat \mu_j}_{\text{tail}(\abs{I})}}^2
    \leq \norm{\br{\hat \mu_j}_{I}}^2 &\leq 2\norm{\br{\emean_j}_I - \mu_I}^2 + 2\norm{\pmean_{\text{tail}(k_{\text{tail}}(\gamma^2;\mu))}}^2 \\ 
    & \leq 2\norm{\br{\emean_j -  \pmean}_I}^2 + 
    2\norm{\pmean_{\text{tail}(k_{\text{tail}}(\gamma^2;\mu))}}^2 \\
    & \leq 2\norm{\emean_j -  \mu}^2 + 
    2\norm{\pmean_{\text{tail}(k_{\text{tail}}(\gamma^2;\mu))}}^2 \\
    & \leq \frac{4B^2\log{8\log{d}/\beta}}{n} + 
    2\norm{\pmean_{\text{tail}(k_{\text{tail}}(\gamma^2;\mu))}}^2
\end{align*}
where the second last inequality follows from vector concentration bound in Eqn. \eqref{eqn:vector-concentration} which holds with probability at least $1-\frac{\beta}{4}$ for all $j \in \lfloor \log{d} \rfloor$.

From a union bound over the success of sketching and error estimation and using Eqn. \eqref{eqn:error-estimate-approx}, 
with probability at least $1-\beta$, 
\begin{align}
    \nonumber
    \hat e_{\overline j}^2 \leq 
    32\norm{\br{\hat \mu_j}_{\text{tail}(k(\overline j))}}^2 + \frac{10\sigma^2k_{\overline j}}{n^2}
    \implies \bar e_{\overline j}^2 &\leq \frac{3}{2}\hat e_{\overline j}^2 \\
    \nonumber
    &\leq 48 \norm{\br{\hat \mu_j}_{\text{tail}(k(\overline j))}}^2  + \frac{15k_{\overline j}\sigma^2}{n^2} \\
    \nonumber
    &\leq  48\norm{\br{\hat \mu_j}_{\text{tail}(\abs{I})}}^2 + 
    \frac{15k_{\overline j}\sigma^2}{n^2} \\
    \label{eqn:bound-on-error-thm3}
    & \leq \frac{196B^2\log{8\log{d}/\beta}}{n} + 
    96\norm{\pmean_{\text{tail}(k_{\text{tail}}(\gamma^2;\mu))}}^2 + 
    \frac{15k_{\overline j}\sigma^2}{n^2}
\end{align}
This gives us that
\begin{align*}
    \bar e_{\overline j}^2 \leq  \br{16\br{\gamma G+ \frac{G\sqrt{\log{8\log{d}/\beta}}}{\sqrt{n}} +
    \sqrt{k_{\overline j}}\frac{\sigma}{n}}}^2
\end{align*}

Finally, from the Above Threshold guarantee (using contra-positive of second part of Lemma \ref{lem:above-threshold}), since 
$\bar e_{\overline j} \leq  \overline \gamma_j - \tilde \alpha =
16\br{\gamma + \frac{G\sqrt{\log{8\log{d}/\beta}}}{\sqrt{n}}
+\sqrt{k_{\overline j}}\frac{\sigma}{n}}
$,
with probability at least $1-\frac{\beta}{4}$, it halts. Thus, $k_{\hat j} \leq k_{\overline j} \leq 2 k_{\text{tail}}(\gamma^2;\mu)$. Plugging this in Eqn. \eqref{eqn:error-bound-thm3} gives the claimed error bound.

The communication complexity now follows by the setting of the two sketch sizes.
The total communication for sketches $\bc{S_j}_j$ is
\begin{align}
    \label{eqn:comm-one-thm3}
    \sum_{j=1}^{\overline j} \rows \pads \cols_j &\leq 
    \sum_{j=1}^{\overline j} \br{8RP^2 +\frac{R}{4}}k_j
    = \Theta\br{k_{\text{tail}}(\gamma^2;\mu)\log{8d\log{d}/\beta}\text{log}^2\br{16\log{8d\log{d}/\beta}\log{d}/\beta}}
\end{align}
where we use that $\overline j\leq 2k_{\text{tail}}(\gamma^2;\mu)$, with holds with probability at least $1-\beta$, as argued above.
Similarly, the total communication for sketches $\bc{\tilde S_j}_j$, is 
\begin{align}
  \label{eqn:comm-two-thm3}
    \sum_{j=1}^{\overline j} \tilde \rows \tilde \cols \tilde \pads = \overline j 2\tilde P^2 
    = \Theta\br{\log{k_{\text{tail}}(\gamma^2;\mu)}\log{4d\log{d}/\beta}}
\end{align}

The total communication is the sum of the two terms in Eqn. \eqref{eqn:comm-one-thm3} and \eqref{eqn:comm-two-thm3} which is dominated by the first term.
This completes the proof.
\section{Proofs of Error Lower Bounds for FME}
\label{app:proofs_lower-bounds_error}
\subsection{Non-private statistical error lower bounds}
\label{app:lower-bounds-statistcal-error}

We first recall some notation: for a probability distribution $\cD$, let $\mu(\cD): = \mathbb{E}_{z\sim \cD}[z]$ denote its mean, when it exists.
\begin{theorem}
\label{thm:statistical-lower-bound-mean}
Let $d\in \bbN, G$ and $M$ such that $0<M\leq 1$, define the instance class, 
\begin{align*}
    \cP_1(d,G,M) = \bc{\text{Probability distribution } \cD \text{ over }  \bbR^d : \norm{z}\leq G \text{ for } z\sim \cD \text{ and }
    \norm{\mu(\cD)}
    \leq MG}
\end{align*}
Then, for any $n \in \bbN$, we have,
\begin{align*}
    \min_{\cA: \br{\bbR^d}^n \rightarrow \bbR^d} \max_{\cD \in  \cP_1(d,G,M)} \underset{D\sim \cD^n, \cA}{\mathbb{E}}\norm{\cA(D) - \mu(\cD)}^2 = \Theta\br{G^2\min\br{M^2,\frac{1}{n}}}
\end{align*}
\end{theorem}
\begin{proof}
The upper bound is achieved by the best of the following two estimators: $M^2G^2$ for the trivial estimator which outputs ``zero" regardless of the problem instance, and $\frac{G^2}{n}$ for the sample mean estimator.

For the lower bound, first note that it suffices to consider $d=1$ since the right hand side doesn't explicit depend on $d$ and one can always embed a one-dimensional instance in any higher dimensions simply by appending zero co-ordinates to the data points. 
Our lower bound uses the (symmetric) Bernoulli mean estimation problem, the proof of which is classical (see Example 7.7 in \cite{duchi2016lecture}). However, in its standard formulation, there is no constraint on the mean that $\norm{\mu(\cD)}\leq MG$. Consequently, we modify the proof to incorporate it and we present it (almost) in entirety, below.

As is standard in minimax lower bounds, we reduce the estimation problem to (binary) hypothesis testing:
 let $\cZ = \bc{-G,G}$ and consider two distributions $\cD_1$ and $\cD_2$ over it defined as follows,
\begin{align*}
    \mathbb{P}_{z\sim \cD_1}\br{z=G} = \frac{1+\tau}{2} \qquad   \mathbb{P}_{z\sim \cD_2}\br{z=G} = \frac{1-\tau}{2} 
\end{align*}
where $\tau\geq0$ is a parameter to be set later. Note that $\abs{\mu(\cD_1)} = \abs{\mu(\cD_2)} = \tau G$. Note that since $\cD_1, \cD_2 \in  \cP_1(1,G,M)$, this gives us the  constraint that $\tau \leq \frac{MG}{G} = M$. 

The minimax risk is lower bounded as,
\begin{align*}
       \min_{\cA: \br{\bbR^d}^n \rightarrow \bbR^d} \max_{\cD \in  \cP_1(d,G,M)} \underset{D\sim \cD^n, \cA}{\mathbb{E}}\norm{\cA(D) - \mu(\cD)}^2 &\geq 
        \min_{\cA: \br{\bbR^d}^n \rightarrow \bbR^d} \max_{\cD \in \bc{\cD_1,\cD_2}} \underset{D\sim \cD^n, \cA}{\mathbb{E}}\norm{\cA(D) - \mu(\cD)}^2 \\
        & \geq \min_{\cA: \br{\bbR^d}^n \rightarrow \bbR^d} \underset{\cD \sim \text{Unif}\bc{\cD_1,\cD_2}}{\mathbb{E}} \underset{D\sim \cD^n, \cA}{\mathbb{E}}\norm{\cA(D) - \mu(\cD)}^2\\
        &\geq \min_{\Psi: \br{\bbR^d}^n \rightarrow \bc{\cD_1,\cD_2}} G^2\tau^2 \underset{\cD\sim \text{Unif}(\cD_1,\cD_2), D\sim\cD^n}{\mathbb{P}}\br{\Psi(D) \neq \cD} \\
        & = \frac{G^2\tau^2}{2}\br{1-\text{TV}(\cD_1^n, \cD_2^n)^2}
\end{align*}
where $\text{TV}(\cD_1^n, \cD_2^n)$ denotes the total variation distance between distributions.
In the above, the third inequality is the reduction from estimation to (Bayessian) testing (see Proposition 7.3 in \cite{duchi2016lecture}) with $\Psi$ being the test function: herein, the nature first chooses $\cD \sim \text{Unif}\bc{\cD_1,\cD_2}$, and conditioned on this choice, we observe $n$ i.i.d. samples from $\cD$, and the goal is to infer $\cD$. Further, the last equality follows from Proposition 2.17 in \cite{duchi2016lecture}.

We now upper bound total variation distance by KL divergence using Bretagnolle–Huber inequality \cite{bretagnolle1978estimation} -- this is the key step which differs from the proof in the standard setup wherein a weaker bound based on Pinsker's inequality suffices. 
We get,
\begin{align*}
    \text{TV}(\cD_1^n, \cD_2^n)^2 \leq
    1- \text{exp}\br{-\text{KL}\br{\cD_1^n\Vert \cD_2^n}}
     &=     1- \text{exp}\br{-n\text{KL}\br{\cD_1\Vert \cD_2}}\\
     &=1- \text{exp}\br{-n\tau \text{log} \br{\frac{1+\tau}{1-\tau}}}\\
     & \leq 1-\text{exp}\br{-3n\tau^2}
\end{align*}
where the first equality uses the chain rule for KL divergence, the second equality follows from direct computation and the last inequality holds for $\tau \leq \frac{1}{2}$ -- note that, in our setup above, we have the constraint that 
$\tau \leq M$,
which isn't necessarily violated with the assumption $\tau \leq \frac{1}{2}$. This gives us that,
\begin{align*}
       \min_{\cA: \br{\bbR^d}^n \rightarrow \bbR^d} \max_{\cD \in  \cP_1(d,G,M)} \underset{D\sim \cD^n, \cA}{\mathbb{E}}\norm{\cA(D) - \mu(\cD)}^2 &\geq \frac{G^2\tau^2}{2}\text{exp}\br{-3n\tau^2}
\end{align*}
Finally, setting $\tau^2 = \min\br{\frac{M^2}{2}, \frac{1}{3n}}$ yields that,
\begin{align*}
       \min_{\cA: \br{\bbR^d}^n \rightarrow \bbR^d} \max_{\cD \in  \cP_1(d,G,M)} \underset{D\sim \cD^n, \cA}{\mathbb{E}}\norm{\cA(D) - \mu(\cD)}^2 &\geq \frac{G^2\min\br{\frac{M^2}{2}, \frac{1}{3n}}}{2}\text{exp}\br{-1
       } \geq \frac{G^2}{12}\min\br{M^2, \frac{2}{3n}},
\end{align*}
which completes the proof.
\end{proof}

\begin{theorem}
\label{thm:statistical-lower-bound-tail}
Let $k,d\in \bbN$ such that $k<d$, $G,\gamma>0$ such that $\gamma\leq 1$. Define the instance class
\begin{align*}
    \cP_2(d,G,\gamma,k) = \bc{\text{Probability distribution } \cD \text{ over } \bbR^d: \norm{z}\leq G \text{ for }z\sim \cD, \text{ and } \norm{\mu(\cD)_{\ntail{k}}}^2\leq \gamma^2}
\end{align*}
Then, for any $n \in \bbN$, we have,
\begin{align*}
    \min_{\cA: \br{\bbR^d}^n \rightarrow \bbR^d} \max_{\cD \in \cP_2(d,G,k,\gamma)} \underset{D\sim \cD^n, \cA}{\mathbb{E}}\norm{\cA(D) - \mu(\cD)}^2 = \Theta\br{\min\br{G^2,\frac{G^2}{n}}}
\end{align*}
\end{theorem}

Before presenting the proof, we note that in the above statement, we exclude $k=d$. This is because, for $k=d$, $\mu(\cD)_{\ntail{k}} = G\mu(\cD)$ and hence $\cP_2(d,G,\gamma,k) = \cP_1(d,G,\gamma)$. The optimal rate then follows from Theorem \ref{thm:statistical-lower-bound-mean}.

\begin{proof}
The upper bound follows from the guarantees in the standard setting of mean estimation with bounded data, ignoring the additional structure of bound on tail norm. Specifically, the bound $G^2$ is achieved from outputting zero, and $\frac{G^2}{2}$ is achieved by the sample mean estimator.

For the lower bound, we will establish the result for $\gamma=0$, and the main claim will follow since $P_2(d,G,k,\gamma) \subseteq P_2(d,G,k,0)$ for any $\gamma\geq0$. Further, assume that $d\geq 2$, otherwise, the claim follows from Theorem \ref{thm:statistical-lower-bound-mean}.

Suppose for contradiction that there exists $G>0$, $d,n,k\in \bbN$ with $k<d$ such that there exists an algorithm $\cA: \br{\bbR^d}^n \rightarrow \bbR^d$ with the guarantee, 
\begin{align*}
    \max_{\cD \in \cP_2(d,G,k,0)} \underset{D\sim \cD^n, \cA}{\mathbb{E}}\norm{\cA(D) - \mu(\cD)}^2 = o\br{\min\br{G^2,\frac{G^2}{n}}}
\end{align*}

Now, consider an instance $\tilde \cD \in \cP_1(1,G,G)$, defined in Theorem \ref{thm:statistical-lower-bound-mean}.
We will use Algorithm $\cA$ to break the lower bound in Theorem \ref{thm:statistical-lower-bound-mean} hence establishing a contradiction.

Given $\tilde D \sim \tilde \cD^n$, we simply append $d-1$ ``zero" co-ordinates to every data point in $\tilde D$ -- let $D$ denote this constructed dataset. Note that every data point in $D$ can be regarded as a sample from a product distribution $\cD = \tilde \cD \times (\cD_0)^{d-1}$, where distribution $\cD_0$ is a point mass at zero. Further, note that $\cD \in \cP_2(d,G,k,0)$. Applying algorithm $\cA$ yields, 
\begin{align*}
     \underset{D\sim \cD^n, \cA}{\mathbb{E}}\norm{\cA(D) - \mu(\cD)}^2 = o\br{\min\br{G^2,\frac{G^2}{n}}}
\end{align*}
However, note that $\mu(\cD) = [\mu(\tilde\cD), 0, \cdots, 0]^\top$. Hence, post-processing $\cA$ by only taking its first two co-ordinates defining $\tilde \cA$ as $\tilde \cA(\tilde D) : = [\br{\cA(D)}_1,\br{\cA(D)}_2]^\top$ gives us that $\norm{\cA(D)-\mu(\cD)}^2 \geq \norm{\tilde \cA(\tilde D) - \mu(\tilde \cD)}^2$. Plugging this above, we get,
\begin{align*}
  \max_{\tilde \cD \in \cP_1(1,G,G)}
     \underset{\tilde D\sim \tilde \cD^n, \cA}{\mathbb{E}}\norm{\tilde \cA(\tilde D) - \mu(\tilde \cD)}^2 = o\br{\min\br{G^2,\frac{G^2}{n}}}
\end{align*}
This contradicts the statement in Theorem \ref{thm:statistical-lower-bound-mean} establishing the claimed lower bound.
\end{proof}

\subsection{Lower Bounds on Error under Differential Privacy}
\label{app:lower-bounds-dp}

\begin{proof}[Proof of \cref{thm:lower-bound-dp-mean-estimation-bounded-mean}]
The upper bound is established by the known guarantees of the best of the following procedures: output zero for $M^2G^2$ bound and Gaussian mechanism for $\frac{G^2d\log{1/\delta}}{n^2\epsilon^2}$ bound. For the lower bound, we consider three cases (a). $M\geq 1$, (b). $M\leq \frac{4}{n}$ and (c). $\frac{4}{n}< M< 1$. The first case $M\geq 1$ is trivial wherein we simply ignore the bound $M$ on norm of mean and apply known lower bounds for DP-mean estimation with bounded data  (see \cite{steinke2015between, kamath2020primer}). Specifically, define the following instance class,
\begin{align*}
    \hat \cP\br{n,d, G}= \bc{\bc{z_1,z_2, \cdots z_n}:z_i\in \bbR^d, \norm{z_i}\leq G}
 \end{align*}
     The above instance class is the standard setting of mean estimation with bounded data for which it is known (see \cite{steinke2015between, kamath2020primer}) that for $\epsilon=O(1)$ and $2^{-\Omega(n)} \leq \delta \leq \frac{1}{n^{1+\Omega(1)}}$, we have
     \begin{align}
     \label{eqn:dp-mean-estimation-lower-bound}
         \min_{\cA: \cA \text{ is }(\epsilon,\delta)-\text{DP}} \max_{D \in \hat \cP(n,d,G)} \mathbb{E}\norm{\cA(D)-\hat \mu(D)}^2 = \Theta\br{\min\bc{G^2,\frac{G^2d\log{1/\delta}}{n^2\epsilon^2}}}
\end{align}

Using the fact that for $M\geq 1$, we have, $\hat \cP_1(n,d,G,M) = \hat \cP(n,d,G)$, and hence,
 \begin{align*}
         \min_{\cA: \cA \text{ is }(\epsilon,\delta)-\text{DP}} \max_{D \in \hat \cP_1(n,d,G,M)} \mathbb{E}\norm{\cA(D)-\hat \mu(D)}^2 &=
          \min_{\cA: \cA \text{ is }(\epsilon,\delta)-\text{DP}} \max_{D \in \hat \cP(n,d,G)} \mathbb{E}\norm{\cA(D)-\hat \mu(D)}^2\\
          & = \Omega\br{\min\bc{G^2,\frac{G^2d\log{1/\delta}}{n^2\epsilon^2}}} 
         \\&= \Omega\br{\min\bc{M^2G^2,\frac{G^2d\log{1/\delta}}{n^2\epsilon^2}}}
     \end{align*}

For the second case $M\leq \frac{4}{n}$, suppose to contrary that there exists $d,n\in \bbN$,  $G, M>0$ satisfying $M\leq \frac{4}{n}$ and $\epsilon=O(1), 2^{-\Omega(n)} \leq \delta \leq \frac{1}{n^{1+\Omega(1)}}$ such that there exists an $(\epsilon, \delta)$-DP algorithm $\cA$ with the following guarantee,
\begin{align*}
    \max_{D \in \hat \cP_1(n,d,G,M)}\mathbb{E}\norm{\cA(D)-\hat \mu(D)}^2 = o\br{\min\bc{M^2G^2,\frac{G^2d\log{1/\delta}}{n^2\epsilon^2}}}
\end{align*}
We will use the above algorithm to break a known lower bound, hence implying a contradiction.
Consider the instance class
$\hat \cP\br{1,d,MG}$, i.e. $n=1$, defined above.
From the DP-mean estimation lower bound (Eqn. \eqref{eqn:dp-mean-estimation-lower-bound}), we have that for $\epsilon=O(1)$ and $\delta = \Theta(1)$,
  \begin{align}
  \nonumber
         \min_{\cA: \cA \text{ is }(\epsilon,\delta)-\text{DP}} \max_{D \in \hat \cP(1,d,M)} \mathbb{E}\norm{\cA(D)-\hat \mu(D)}^2 &=
         \nonumber\Theta\br{\min\bc{M^2G^2,\frac{M^2G^2d\log{1/\delta}}{\epsilon^2}}}\\
         \nonumber
         &=\Theta\br{G^2\min\bc{M^2,\frac{M^2d}{\epsilon^2}}}\\
         \nonumber
         & = \Omega\br{G^2\min\bc{M^2, M^2d}} \\
         & = \Omega(M^2G^2)
         \label{eqn:dp-lowerbound-case-two}
\end{align}
where the second and third equality follows from conditions on $\epsilon$ and $\delta$.

We now modify $\cA$ to solve DP mean estimation for datasets in 
$\hat \cP\br{1,d,MG}$
. Towards this, given 
$D \in \hat \cP\br{1,d,MG}$
i.e. $D=\bc{z}$, construct dataset $\tilde D$ of $n$ samples as follows: $\tilde D=\bc{\frac{nz}{4}, \vzero, \vzero, \cdots \vzero}$. 
It is easy to see that $\tilde D \in \hat \cP_1(n,d,G,M)$ since $M\leq \frac{4}{n}$.
Applying $(\epsilon,\delta)$-DP $\cA$ to $\tilde D$ gives us that,
\begin{align*}
    \mathbb{E}\norm{\cA(\tilde D)-\hat \mu(\tilde D)}^2 = o\br{G^2\min\bc{M^2,\frac{d\log{1/\delta}}{n^2\epsilon^2}}}
\end{align*}

We now post-process $\cA(\tilde D)$ to obtain a solution for $\cA(\tilde D)$. Firstly, define $\tilde \delta = \frac{1}{2}-\delta$. With probability $\tilde \delta$, output $z$ as the mean of dataset $\tilde D$, in the other case, output $\cA(\tilde D)$; call this output $\bar \mu$. Note that in the first case, we don't preserve privacy, whereas the other case satisfies $(\epsilon,\delta)$-DP by DP property of $\cA$ and post-processing. Hence, overall, the reduced method satisfies $(\epsilon, \frac{1}{2})$-DP. We remark that this is done since since the usual lower bounds (Eqn. \eqref{eqn:dp-mean-estimation-lower-bound}) assumes that $\delta=\Theta(1)$ as therein $n=1$.
The accuracy is, 
\begin{align*}
    \mathbb{E}\norm{\bar \mu - \hat \mu(D)}^2 \leq \delta \mathbb{E}\norm{\cA(\tilde D)-\hat \mu(\tilde D)}^2 &= o\br{G^2\min\bc{M^2, \frac{d\delta \log{1/\delta}}{n^2\epsilon^2}}} \\
    & = o\br{G^2\min\bc{M^2, \frac{d}{n^2\epsilon^2}}} \\
    & =  o\br{G^2\min\bc{M^2, \frac{d}{n^2}}} \\
    & = o(M^2G^2)
\end{align*}
where the second equality follows since $\delta\log{1/\delta}\leq 1$ for $2^{-\Omega(n)} \leq \delta \leq \frac{1}{n^{1+\Omega(1)}}$, third follows since $\epsilon=O(1)$ and in the final equality, we use the assumption that $M\leq \frac{4}{n}$. This contradicts the lower bound in Eqn. \eqref{eqn:dp-lowerbound-case-two}.

We now proceed to the final case, $\frac{4}{n}<M<1$ wherein we will proceed as in case two.
Suppose to contrary that there exists $d,n\in \bbN$,  $G, M>0$ satisfying $\frac{4}{n}<M<1$ and $\epsilon=O(1), 2^{-\Omega(n)} \leq \delta \leq \frac{1}{n^{1+\Omega(1)}}$ such that there exists an $(\epsilon, \delta)$-DP algorithm $\cA$ with the following guarantee,
\begin{align*}
    \max_{D \in \hat \cP_1(n,d,G,M)}\mathbb{E}\norm{\cA(D)-\hat \mu(D)}^2 = o\br{G^2\min\bc{M^2,\frac{d\log{1/\delta}}{n^2\epsilon^2}}}
\end{align*}

Define $N := \lceil \frac{Mn}{2} \rceil$ and consider the instance class $\hat \cP(N,d,G)$, as defined above. Again, from the DP-mean estimation lower bound, we have that for $\epsilon=O(1)$ and $2^{-\Omega(N)} \leq \delta \leq \frac{1}{N^{1+\Omega(1)}}$,
  \begin{align}
         \min_{\cA: \cA \text{ is }(\epsilon,\delta)-\text{DP}} \max_{D \in \hat \cP(N,d,G)} \mathbb{E}\norm{\cA(D)-\hat \mu(D)}^2 
         &=\Theta\br{\min\bc{G^2,\frac{G^2d\log{1/\delta}}{N^2\epsilon^2}}}
         \label{eqn:dp-lowerbound-case-three}
\end{align}

We now modify $\cA$ to solve DP mean estimation for datasets in $\hat \cP(N,d,G)$. Towards this, given $D \in \hat \cP(N,d,G)$, construct dataset $\tilde D$ of $n$ samples by appending $n-N$ samples of $\vzero \in \bbR^d$ to $D$. Note that the norm of data points in $\tilde D$ is upper bounded by $G$ and $\norm{\hat \mu(\tilde D)} = \frac{N}{n}\norm{\hat \mu(D)} \leq \frac{N}{n} G \leq \frac{MG}{2}+ \frac{G}{2n}\leq MG$. Hence, we have that $\tilde D \in \hat \cP_1(n,d,G,M)$. Applying $(\epsilon,\delta)$-DP algorithm $\cA$ to $\tilde D$ gives us that, 
\begin{align*}
    \mathbb{E}\norm{\cA(\tilde D)-\hat \mu(\tilde D)}^2 = o\br{G^2\min\bc{M^2,\frac{d\log{1/\delta}}{n^2\epsilon^2}}}
\end{align*}
We now post-process $\cA(\tilde D)$ to obtain a solution for dataset $D$. 
With probability $\frac{1}{N^{1+\Omega(1)}} -\delta$, output $\hat \mu(D)$ as the solution, in the other case, output $\frac{n}{N}\cA(\tilde D)$; call this solution $\bar \mu$. As in the previous case, in the first case, we don't preserve privacy while the second case satisfies $(\epsilon,\delta)$-DP from DP property of $\cA$ and post-processing. Hence, the combined method satisfies $\br{\epsilon, \frac{1}{N^{1+\Omega(1)}}}$-DP. Further, the accuracy guarantee is, 
\begin{align*}
    \mathbb{E}\norm{\bar \mu - \hat \mu(D)}^2 \leq \delta \mathbb{E}\norm{\frac{n}{N}\cA(\tilde D) - \hat \mu(D)}^2 &\leq\delta \frac{n^2}{N^2}\mathbb{E}\norm{\cA(\tilde D)-\hat \mu(\tilde D)}^2 \\
    &= o\br{\delta\frac{n^2G^2}{N^2}\min\bc{M^2, \frac{\log{1/\delta}}{n^2\epsilon^2}}}\\
    &= o\br{\frac{n^2G^2}{N^2}\min\bc{\delta M^2, \frac{\delta\log{1/\delta}}{n^2\epsilon^2}}}\\
  &= o\br{G^2\min\bc{1, \frac{\log{N}}{N^2\epsilon^2}}}
\end{align*}
where the last equality follows from the setting of $N$ and from $\delta\log{1/\delta}\leq 1\leq \log{N}$ as $N\geq 2$ when $Mn>4$. This contradicts the lower bound in Eqn. \eqref{eqn:dp-lowerbound-case-three} for $\delta = \frac{1}{N^{1+\Omega(1)}}$. Combining the three cases finishes the proof.
\end{proof}

\begin{proof}[Proof of \cref{thm:dp-mean-estimation-lower-bound-2}]
The upper bound is obtained by the best of the following procedures: output zero for $G^2$ term, our proposed method 
\cref{thm:dp-mean-estimation-upper-bound-2}
for $\gamma^2 + \frac{kG^2\log{1/\delta}}{n^2\epsilon^2}$ term, and $\frac{dG^2\log{1/\delta}}{n^2\epsilon^2}$ from Gaussian mechanism. For the lower bound, we first define the following instance classes, 
\begin{align*}
    &\hat \cP_1(n,d,G,M) = \bc{\bc{z_1,z_2,\cdots z_n}: z_i \in \bbR^d, \norm{z_i} \leq G, \norm{\hat \mu(\bc{z_i}_{i=1}^n)} \leq MG}\\
     &\hat \cP\br{n,d, G}= \bc{\bc{z_1,z_2, \cdots z_n}:z_i\in \bbR^d, \norm{z_i}\leq G}
     \end{align*}
     The second instance class is the standard setting of mean estimation with bounded data and it is known (see \cite{steinke2015between, kamath2020primer}) that for $\epsilon=O(1)$ and $2^{-\Omega(n)} \leq \delta \leq \frac{1}{n^{1+\Omega(1)}}$, we have
     \begin{align}
     \label{eqn:dp-mean-estimation-lower-bound-proof2}
         \min_{\cA: \cA \text{ is }(\epsilon,\delta)-\text{DP}} \max_{D \in \hat \cP(n,d,G)} \mathbb{E}\norm{\cA(D)-\hat \mu(D)}^2 = \Theta\br{\min\bc{G^2,\frac{G^2d\log{1/\delta}}{n^2\epsilon^2}}}
\end{align}
For the first instance class, we showed a lower bound in \cref{thm:lower-bound-dp-mean-estimation-bounded-mean}.
We will now use these two results to construct a reduction argument for the lower bound claimed in the theorem statement. Suppose to contrary that there exists $k,d,n\in \bbN$, $G,\gamma>0$ and $\epsilon=O(1)$, $2^{-\Omega(n)} \leq \delta \leq \frac{1}{n^{1+\Omega(1)}}$ such that there exists an algorithm $\cA$ with the following guarantee, 
\begin{align*}
    \max_{D \in \cP(n,d,G,k,\gamma)} \mathbb{E}\norm{\cA(D)-\hat \mu(D)}^2= o\br{\min \bc{ G^2, \gamma^2G^2 + \frac{kG^2\log{1/\delta}}{n^2\epsilon^2}, \frac{dG^2\log{1/\delta}}{n^2\epsilon^2}}}
\end{align*}

Consider instance classes $\hat \cP_1\br{n,d-k,\frac{G}{\sqrt{2}},\gamma}$ and $\hat\cP\br{n,k, \frac{G}{\sqrt{2}}}$. Given datasets $D_1 \in \hat \cP_1\br{n,d-k,\frac{G}{\sqrt{2}},\gamma}$ and $D_2 \in \hat \cP\br{n,k, \frac{G}{\sqrt{2}}}$, we will use Algorithm $\cA$ to solve DP mean estimation for both these datasets. Specifically, construct dataset $D$ by stacking samples from $D_2$ to corresponding samples from $D_1$. Note that each sample in $D$ has norm bounded by $G$. Further, $\hat \mu(D) = [\hat \mu(D_1), \hat \mu(D_2)]^\top$. Therefore, we have,
\begin{align*}
    \norm{\hat \mu(D)_{\text{tail}(k)}} = G\norm{\hat \mu(D)_{\ntail{k}}}
    = \min_{\tilde \mu: \tilde \mu \text{ is }k \text{-sparse}}\norm{\hat \mu(D) - \tilde \mu} \leq \norm{\hat \mu(D)-[\vzero,\hat \mu(D_2)]^\top} = \norm{\hat \mu(D_1)} \leq \gamma G
\end{align*}
We therefore have that $D \in \cP(n,d,G,k,\gamma)$. Apply $(\epsilon,\delta)$-DP $\cA$ to get $\cA(D) = [\bar \mu_1, \bar \mu_2]$. The accuracy guarantee of $\cA$ gives us that, 
\begin{align*} 
&\mathbb{E}\norm{\cA(D)- \hat \mu(D)}^2= o\br{\min \bc{ G^2, \gamma^2G^2 + \frac{kG^2\log{1/\delta}}{n^2\epsilon^2}, \frac{dG^2\log{1/\delta}}{n^2\epsilon^2}}} \\
& \implies \mathbb{E}\norm{\bar \mu_1 - \hat \mu(D_1)}^2 +\mathbb{E}\norm{\bar \mu_2 - \hat \mu(D_2)}^2 = o\br{\min \bc{ G^2, \gamma^2G^2 + \frac{kG^2\log{1/\delta}}{n^2\epsilon^2}, \frac{dG^2\log{1/\delta}}{n^2\epsilon^2}}} 
\end{align*}
Note that both the outputs $\mu_1$ and $\mu_2$ satisfy $(\epsilon,\delta)$-DP from DP guarantee of $\cA$ and post-processing property. From the above, we get the following accuracy guarantees,
\begin{align}
\label{eqb:dp-lower-bound-concat-one}
& \mathbb{E}\norm{\bar \mu_1 - \hat \mu(D_1)}^2 = o\br{\min \bc{ G^2, \gamma^2G^2 + \frac{kG^2\log{1/\delta}}{n^2\epsilon^2}, \frac{dG^2\log{1/\delta}}{n^2\epsilon^2}}} \\
\label{eqb:dp-lower-bound-concat-two}
& \mathbb{E}\norm{\bar \mu_2 - \hat \mu(D_2)}^2 = o\br{\min \bc{ G^2, \gamma^2G^2 + \frac{kG^2\log{1/\delta}}{n^2\epsilon^2}, \frac{dG^2\log{1/\delta}}{n^2\epsilon^2}}} 
\end{align}
We now consider three cases, (a). $k>d/2$, (b). $k\leq d/2 $ and $\gamma^2<\frac{k\log{1/\delta}}{n^2\epsilon^2}$ and (c). $k\leq d/2$ and
$\gamma^2\geq \frac{k\log{1/\delta}}{n^2\epsilon^2}$. For the first case (a). $k>d/2$, from Eqn. \eqref{eqb:dp-lower-bound-concat-one}, we get, 
\begin{align*}
    \mathbb{E}\norm{\bar \mu_2 - \hat \mu(D_2)}^2 &= o\br{\min \bc{ G^2, \gamma^2G^2+ \frac{kG^2\log{1/\delta}}{n^2\epsilon^2}, \frac{dG^2\log{1/\delta}}{n^2\epsilon^2}}} \\
    &= o\br{\min \bc{ G^2, G^2+ \frac{kG^2\log{1/\delta}}{n^2\epsilon^2}, \frac{dG^2\log{1/\delta}}{n^2\epsilon^2}}} \\
    & = o\br{\min \bc{ G^2, \frac{kG^2\log{1/\delta}}{n^2\epsilon^2}}}
\end{align*}
where the second equality uses $\gamma^2\leq 1$ and the third equality uses that $d<2k$.

For the second case  $k\leq d/2 $ and $\gamma^2<\frac{k
\log{1/\delta}}{n^2\epsilon^2}$, we again use Eqn. \eqref{eqb:dp-lower-bound-concat-one} to get,
\begin{align*}
    \mathbb{E}\norm{\bar \mu_2 - \hat \mu(D_2)}^2 &= o\br{\min \bc{ G^2, \gamma^2G^2+ \frac{kG^2\log{1/\delta}}{n^2\epsilon^2}, \frac{dG^2\log{1/\delta}}{n^2\epsilon^2}}} \\
    &= o\br{\min \bc{ G^2, \frac{2kG^2\log{1/\delta}}{n^2\epsilon^2}, \frac{dG^2\log{1/\delta}}{n^2\epsilon^2}}} \\
    & = o\br{\min \bc{ G^2, \frac{kG^2\log{1/\delta}}{n^2\epsilon^2}}}
\end{align*}
where the second equality uses  $\gamma^2<\frac{kG^2\log{1/\delta}}{n^2\epsilon^2}$ and the third equality uses $2k\leq d$.

Combining cases (a). and (b). together contradicts the lower bound for DP mean estimation, mentioned in Eqn. \eqref{eqn:dp-mean-estimation-lower-bound-proof2}.
 For the third case $k\leq d/2$ and $\gamma^2\geq \frac{k\log{1/\delta}}{n^2\epsilon^2}$, from Eqn. \eqref{eqb:dp-lower-bound-concat-two}, we get 
\begin{align*}
    \mathbb{E}\norm{\bar \mu_1 - \hat \mu(D_1)}^2 &=
    o\br{\min \bc{ G^2, \gamma^2G^2 + \frac{kG^2\log{1/\delta}}{n^2\epsilon^2}, \frac{dG^2\log{1/\delta}}{n^2\epsilon^2}}} \\
    &=o\br{\min \bc{ G^2, 2\gamma^2G^2, \frac{dG^2\log{1/\delta}}{n^2\epsilon^2}}}  \\
    & = o\br{\min \bc{\gamma^2G^2, \frac{(k + (d-k))G^2\log{1/\delta}}{n^2\epsilon^2}}} \\
    & = o\br{\min \bc{\gamma^2G^2,\frac{ 2(d-k)G^2\log{1/\delta}}{n^2\epsilon^2}}} \\
     & = o\br{\min \bc{\gamma^2G^2,\frac{ (d-k)G^2\log{1/\delta}}{n^2\epsilon^2}}} 
\end{align*}

where the second equality uses that
$\gamma^2\geq \frac{k\log{1/\delta}}{n^2\epsilon^2}$,
and the second last equality uses that $k\leq d/2 \implies k \leq d-k$. This contradicts the lower bound for DP mean estimation in \cref{thm:lower-bound-dp-mean-estimation-bounded-mean} since $D_2 \in \hat \cP_1\br{n,d-k,\frac{G}{\sqrt{2}}, \gamma}$. Combining all the above cases finishes the proof.
\end{proof}

\section{Proofs of Communication Lower Bounds for FME}
\label{app:proofs_lower-bounds_communication}

\label{sec:multi-round-protocols}
\paragraph{Multi-round protocols with SecAgg:}
We set up some notation for multi-round protocols.
For $K\in \bbN$, a $K$-round protocol $\cA$ consists of a sequence of encoding schemes, denoted as $\bc{\cE_t}_{t=1}^K$ and a decoding scheme $\cU$. The encoding scheme in the $t$-th round, $\cE_t: \bc{\cO_j}_{j<t} \times \bbR^d \rightarrow \cO_t$ where
$\cO_j$  denote its output space in the $j$-th round.
Since we are operating under the SecAgg constraint, the set $\cO_t$ is identified with a finite field over which SecAgg operates.
and let $b_t:=\log{\abs{\cO_t}}$ be the number of bits used to encode messages in the $t$-th round.
For dataset $D=\bc{z_i}_{i=1}^n$ distributed among $n$ clients, we recursively define outputs of SecAgg, $O_1 = \text{SecAgg}(\bc{\cE_1(z_i)}_{i=1}^n)$ and $O_t = \text{SecAgg}\left(\left\{\cE_t\br{\bc{O_j}_{j<t},z_i}\right\}_{i=1}^n\right)$.
Finally, after $K$ rounds, the decoding scheme, denoted as $\cU: \cO_1 \times \cO_2 \times \cdots \times \cO_K \rightarrow \bbR^d$, outputs $\cU\br{\bc{O_t}_{t\leq K}}=:\cA(D)$.  The total per-client communication complexity is $\sum_{t=1}^K b_t$.

Given a set $\cM$ in a normed space $(\cX,\norm{\cdot})$, and $\gamma>0$, we use $\cN\br{\cP,\gamma,\norm{\cdot}}$
and $\cN^{\text{ext}}\br{\cP,\gamma,\norm{\cdot}}$
to denote its covering
and exterior covering numbers at scale $\gamma$ -- see Section 4.2 in \cite{vershynin2018high} for definitions.

\begin{proof}[Proof of \cref{thm:lower-bound-communication-bounded-mean-1}]
This follows uses the one-round compression lower bound for unbiased schemes, Theorem 5.3 in \cite{chen2022fundamental}.
Specifically, \cite{chen2022fundamental} showed that for any
given $d, M$, for any unbiased encoding and decoding scheme with $\ell_2$ error at most $\alpha$, there exists a data point $z \in \bbR^d$ with $\norm{z}\leq M$ such that its encoding requires $\Omega\br{\frac{dM^2}{\alpha}}$ bits. 
We first extend the same lower bound to multi-round protocols. Assume to the contrary that there exists a $K>0$ and a $K$-round protocol with encoding and decoding schemes $\bc{\cE_i}_{i=1}^K$ and $\cU$ respectively, such that the total size of messages communicated, $\sum_{i=1}^K b_i = o\br{\frac{dM^2}{\alpha}}$ bits. Now, note that the set of possible encoding schemes are range (or communication) restricted, but otherwise arbitrary. Thus, we can construct a one round encoding scheme which simulates the $K$ round scheme $\bc{\cE_i}_{i=1}^K$. Specifically, it first encodes the data point using $\cE_1$ to get a message size $b_1$, then uses $\cE_2$ on the output and data to get a message of size $b_2$, and so on, until it has used all the $K$ encoding schemes. Finally, it concatenates all messages and sends it to the server, which decodes it using $\cU$. The total size of the messages thus is $\sum_{i=1}^K b_i =o\br{\frac{dM^2}{\alpha}}$, which contradicts the result in \cite{chen2022fundamental}, and hence establishes the same lower bound for multi-round schemes. 
We now extend the compression lower bound by reducing FME under SecAgg to compression.
We simply define the probability distribution as supported on a single data point $z$, so that all clients have the same data point. Due to SecAgg constraint, the decoder simply observes the average of messages i.e the same number of bits as communicated by each client. Hence, if there exists a multi-round scheme SecAgg-compatible scheme for FME for this instance, with total per-client communication of $o\br{\frac{dM^2}{\alpha}}$ bits, then we can simulate it to solve the (one client) compression problem with the same communication complexity, implying a contradiction.
Finally, plugging in the
optimal error $\alpha^2$ as $\min\br{M^2,\frac{G^2}{n}+\frac{G^2\log{1/\delta}}{n^2\epsilon^2}}$, where optimality is established in \cref{thm:lower-bound-dp-mean-estimation-bounded-mean}, gives us the claimed bound.
\end{proof}

\begin{proof}[Proof of \cref{thm:lower-bound-communication-bounded-mean-2}]
It suffices to show that a one-round i.e. $K=1$, SecAgg-compatible, unbiased, 
communication protocol cannot be adaptive i.e have
optimal error and 
optimal communication complexity simultaneously for all values of the norm of the mean of instance.
Note that such a 
protocol $\cA$ gets as inputs $n,d,G,\epsilon,\delta$, and outputs an encoding scheme $\cE: \bbR^d \rightarrow \cO$, and decoding scheme $\cU: \cO \rightarrow \bbR^d$ where 
$\cO$ is the output space of the encoding scheme and can be identified with the finite field over which SecAgg operates.
In the proof of \cref{thm:lower-bound-communication-bounded-mean-1},
we showed that the communication complexity for optimal error, even under knowledge of $M=\norm{\mu(\cD)}$, is $b=
\Omega\br{\min\br{d,\frac{M^2n^2\epsilon^2}{G^2\log{1/\delta}}, \frac{M^2nd}{G^2}}}$ bits.
To show that, we proved a compression lower bound -- 
for any protocol $\cA$, there exists a distribution $\cD$, supported on a single point $z$ such that $\norm{\mu(\cD)}\leq M$, and that $\cE(z)$ has size at least $b$ bits w.p. 1. 
We modify this, by defining a distribution $\tilde \cD$ supported on $\bc{\vzero,z}$ such that $\mathbb{P}_{\tilde \cD}(z) = \frac{\ln{n}}{n}$. Note that the norm of the mean shrinks to $\frac{M\log{n}}{n} \ll M$, and hence the optimal communication complexity for such a distribution is smaller.
By direct computation, it follows that upon sampling $n$ i.i.d. points from $\tilde \cP$, with probability at least $1-\frac{1}{n}$, there is at least one client with data point $z$.
However, the encoding function $\cE$, oblivious to the knowledge of the 
data points in other clients, 
still produces an encoding of $z$ in $b$ bits, and thus fails with probability at least $1-\frac{1}{n}$.
\end{proof}

\begin{theorem}
\label{thm:lower-bound-communication-covering}
For any $K>0$, set $\cM\subset \bbR^d$ and $K$-round encoding schemes $\bc{\cE_i}_{i=1}^K$ and decoding scheme $\cU$, if the following holds,
\begin{align*}
    \max_{\cD:  \mu(\cD) \in \cM} \mathbb{E}_{\bc{\cE_i},\cU, D\sim \cD^n}\norm{\cU\br{\bc{\cA_{\cE_i}(D)}_{i\leq K}} - \mu(\cD)}^2 \leq \alpha^2,
\end{align*}
then the total communication $\sum_{i=1}^K b_i \geq \log{\cN^{\text{ext}}(\cM,\alpha, \norm{\cdot}_2)}$
\end{theorem}
\begin{proof}
The proof follows the covering argument in \cite{chen2022fundamental}. 
Given a dataset $D=\bc{z_i}_{i=1}^n$ distributed among $n$ cleints, define outputs of SecAgg recursively, in the multi-round scheme as, $O_1 = \text{SecAgg}(\bc{\cE_1(z_i)}_{i=1}^n)$ and $O_t = \text{SecAgg}\left(\left\{\cE_t\br{\bc{O_j}_{j<t},z_i}\right\}_{i=1}^n\right)$.
Consider the set
\begin{align*}
   \cC =  \bc{\mathbb{E}_{\cU}\cU(\bc{C_j}_{j\leq K}): C_j \in \cO_j}
\end{align*}
Note that since $\abs{\cO_j}\leq 2^{b_j}$, we have that $\abs{\cC} \leq 2^{\sum_{j=1}^Kb_j}$.

For each value of $\mu(\cD) \in \cM$, 
we simply pick a distribution $\cD$ supported on the singleton $\bc{\mu(\cD)}$.
Now, from the premise in the Theorem statement and Jensen's inequality, we get that,
\begin{align*}
 &\max_{\cD:  \mu(\cD) \in \cM} \min_{\text{randomness in }\bc{\cE_i}_i}\norm{\mathbb{E}_{\cU}\left[\cU\br{\bc{O_i}_{i\leq K}}\right] - \mu(\cD)}^2
    \\& \leq \max_{\cD: \mu(\cD) \in \cM} \mathbb{E}_{\bc{\cE_i}_i}\norm{\mathbb{E}_{\cU}\left[\cU\br{\bc{O_i}_{i\leq K}}\right] - \mu(\cD)}^2\\
     & \leq \max_{\cD: \mu(\cD) \in \cM} \mathbb{E}_{\bc{\cE_i}_i,\cU}\norm{\cU\br{\bc{O_i}_{i\leq K}} - \mu(\cD)}^2 \leq \alpha^2,
\end{align*}

This gives us a mapping, which for each $ \mu(\cD)$, gives a point, which is at most $\alpha^2$ away from $\hat \mu(D)$ in $\norm{\cdot}_2^2$ -- this is a $\alpha$ exterior covering of $\cM$ with respect $\norm{\cdot}_2$. However, every point in the above cover lies in  $\cC$. Therefore, 
\begin{align*}
    \sum_{i=1}^K b_i \geq \log{C} \geq \log{\cN^{\text{ext}}(\cM,\alpha, \norm{\cdot}_2)}.
\end{align*}
which completes the proof.
\end{proof}

\begin{proof}[Proof of \cref{thm:lower-bound-communication-bounded-mean-3}]
Note the for any distribution $\cD \in \cP_2(d,G,\gamma,k)$, its mean $\mu(D)$ lies in $\cM=\bc{z \in \bbR^d: \norm{z}_2\leq G \text{ and } \norm{z_{\text{tail}(k)}}\leq \gamma}$.
The claim follows simply from \cref{thm:lower-bound-communication-covering} by plugging in the the covering number of
$\cM$.
Towards this, define $\tilde \cM :=\bc{z \in \bbR^d: \norm{z}_2\leq G \text{ and } \norm{z_{\text{tail}(k)}}\leq 0}$, which is simply the 
set of $k$-sparse vectors in $d$ dimensions, bounded in norm by $G$. 
Now, since 
$\cM\supseteq \tilde \cM$, 
from monotonicity property of exterior covering numbers (see Section 4.2 in \cite{vershynin2018high}), we have
$\cN^{\text{ext}}\br{\tilde \cM,\alpha,\norm{\cdot}_2} \geq \cN^{\text{ext}}\br{\cM,\alpha,\norm{\cdot}_2}$.
Further, we relate the exterior and (non-exterior) covering numbers as 
$\cN^{\text{ext}}\br{\tilde \cM,\alpha,\norm{\cdot}_2} \geq \cN\br{\tilde \cM,2\alpha,\norm{\cdot}_2}$ (see Section 4.2 in \cite{vershynin2018high}). 
We now compute 
$\cN\br{\tilde \cM,2\alpha,\norm{\cdot}_2}$
using the standard volume argument. Give a set, let $\text{Vol}_k$ denote its $k$-dimensional (Lebesque) volume. We have that $\text{Vol}_k(\tilde \cM) = \binom{d}{k} \text{Vol}_k(\bbB^{k}_G) = \binom{d}{k} G^kC_k$, where $\bbB^{k}_G$ is a ball in $k$ dimensions of radius $G$, and $C_k$ is a $k$-dependent constant. Further, note that $\binom{d}{k} \geq \br{\frac{d}{k}}^k$ for $d\geq 2k$.
Now, since sum of volumes of all balls in the cover should be at least as large as this volume, we get,
\begin{align*}
  & \br{\frac{d}{k}}^k G^k C_k  \leq \text{Vol}_k(\tilde \cM) \leq \cN\br{\tilde \cM,2\alpha,\norm{\cdot}_2} \text{Vol}_{k}(\bbB^{k}_{2\alpha}) = \cN\br{\tilde \cM,2\alpha,\norm{\cdot}_2} C_k\br{2\alpha}^k \\
  & \implies \cN\br{\tilde \cM,2\alpha,\norm{\cdot}_2} \geq \br{\frac{d G}{2k\alpha}}^k.
\end{align*}
Finally, using \cref{thm:lower-bound-communication-covering} plugging in the derived lower bound gives the stated guarantee.
\end{proof}

\section{Empirical Details}\label{sec:app-setup}
In all cases, we use central DP with fixed $\ell_2$ clipping denoted as $\eta$. We tune the server learning rate for each noise multiplier on each dataset, with al other hyperparemters held fixed. On all datasets, we select $n\in[100,1000]$ and train for a total $R=1500$ rounds.

F-EMNIST has $62$ classes and $N=3400$ clients with a total of $671,585$ training samples. Inputs are single-channel $(28,28)$ images. 
On F-EMNIST, the server uses momentum of $0.9$ and $\eta=0.49$ with the client using a learning rate of $0.01$ without momentum and a mini-batch size of $20$. We use $n=100$ clients per round.
Our optimal server learning rates are $\{0.6, 0.4, 0.2, 0.1, 0.08\}$ for noise multipliers in $\{0.1, 0.2, 0.3, 0.5, 0.7\}$, respectively.

The Stack Overflow (SO) dataset is a large-scale text dataset based on responses to questions asked on the site Stack Overflow. The are over $10^8$ data samples unevenly distributed across $N=342,477$ clients. We focus on the next word prediction (NWP) task: given a sequence of words, predict the next words in the sequence.
On Stack Overflow, the server uses a momentum of $0.95$ and $\eta=1.0$ with the client using a learning rate of $1.0$ without momentum and limited to 256 elements per client. Our optimal server learning rates are $\{1.0, 1.0, 0.5, 0.5\}$ for noise multipliers in $\{0.1, 0.3, 0.5, 0.7\}$, respectively. We use $n=1000$ clients per round.

Shakespeare is similar to SONWP but is focused on character prediction and instead built from the collective works of Shakespeare, partitioned so that each client is a speaking character with at least two lines. There are $N=715$ characters (clients) with $16,068$ training samples and $2,356$ test samples.
On Shakespeare, the server uses a momentum of $0.9$ and $\eta=1.85$; the client uses a learning rate of $5.0$ and mini-batch size of $4$. Our optimal server learning rates are $\{0.1,0.1,0.08,0.06,0.04,0.03,\}$ for noise multipliers in $\{0.05, 0.1, 0.2, 0.3, 0.5, 0.7\}$, respectively. We use $=100$ clients per round.

\subsection{Model Architectures}
\begin{figure}[H]
    \centering
\begin{Verbatim}
Model: "Large Model: 1M parameter CNN"
_________________________________________________________________
 Layer (type)                Output Shape              Param #   
=================================================================
 Conv2D                    (None, 26, 26, 32)          320       
                                                                 
 MaxPooling2D              (None, 13, 13, 32)          0         
                                                                                                            
 Conv2D                    (None, 11, 11, 64)          18496     
                                                                 
 Dropout                   (None, 11, 11, 64)          0         
                                                                 
 Flatten                   (None, 7744)                0         
                                                                 
 Dense                     (None, 128)                 991360    
                                                                 
 Dropout                   (None, 128)                 0         
                                                                 
 Dense                     (None, 62)                  7998      
                                                                 
=================================================================
Total params: 1,018,174
Trainable params: 1,018,174
Non-trainable params: 0
_________________________________________________________________
\end{Verbatim}
    \caption{F-EMNIST model architecture.}
    \label{fig:large-model}
\end{figure}
\begin{figure}[H]
    \centering
\begin{Verbatim}
Model: "Stack Overflow Next Word Prediction Model"
_________________________________________________________________
 Layer (type)                Output Shape              Param #   
=================================================================
 InputLayer                (None, None)                0         
                                                                 
 Embedding                 (None, None, 96)            960384    
                                                                 
 LSTM                      (None, None, 670)           2055560   
                                                                 
 Dense                     (None, None, 96)            64416     
                                                                 
 Dense                     (None, None, 10004)         970388    
                                                                 
=================================================================
Total params: 4,050,748
Trainable params: 4,050,748
Non-trainable params: 0
_________________________________________________________________
\end{Verbatim}
    \caption{Stack Overflow Next Word Prediction model architecture.}
    \label{fig:sonwp-model}
\end{figure}
\begin{figure}[H]
    \centering
\begin{Verbatim}
Model: "Shakespeare Character Prediction Model"
_________________________________________________________________
 Layer (type)                Output Shape              Param #   
=================================================================
Embedding                  (None, 80, 8)               720       
                                                                 
LSTM                       (None, 80, 256)             271360    
                                                                 
LSTM                       (None, 80, 256)             525312    
                                                                 
Dense                      (None, 80, 90)              23130     
                                                                 
=================================================================
Total params: 820,522
Trainable params: 820,522
Non-trainable params: 0
_________________________________________________________________
\end{Verbatim}
    \caption{Shakespeare character prediction model architecture.}
    \label{fig:shakespeare-model}
\end{figure}
\section{DP and Sketching Empirical Details}\label{app:DP}
\paragraph{Noise multiplier to $\mathbf{\varepsilon}$-DP}
We specify the privacy budgets in terms of the noise multiplier $z$, which together with the clients per round $n$, total clients $N$, number of rounds $R$, and the clipping threshold completely specify the trained model $\varepsilon$-DP. Because the final $\mathbf{\varepsilon}$-DP values depend on the sampling method: e.g., Poisson vs. fixed batch sampling, which depends on the production implementation of the FL system, we report the noise multipliers instead. Using~\citet{mironov2017renyi}, our highest noise multipliers roughly correspond to $\varepsilon=\{5,10\}$ using $\delta=1/N$ and privacy amplification via fixed batch sampling for SONWP and F-EMNIST.

\paragraph{Sketching} We display results in terms of the noise multiplier which fully specifies the $\varepsilon$-DP given our other parameters ($n$, $N$, and $R$). We use a count-mean sketch as proposed in~\citet{chen2021communication} which compresses gradients to a sketch matrix of size ($t,w$)=($length$,$width$). We find empirically that this had no major difference compared with our proposed median-of-means sketch. We use $t=15$ as used in~\citet{chen2022fundamental}. We use this value for all our experiments and calculate the $width=d/(r * length)$ where $gradient \in \mathcal{R}^d$ and $r$ is the compression rate. The full algorithms for decoding and encoding can be found in~\citet{chen2022fundamental} but we note that we do not normalize by the sketch length and instead normalize by its square root, which gives a (approximately) preserved sketched gradient norm.

\section{Adapt Tail: Empirical Results and Challenges in FL}
\label{app:adapt-tail}
In this section, we describe some \textit{natural} mappings of Adapt Tail for FME to FO, as well as discuss some challenges encountered in FL experiments.

The client protocol
again remains the same with $z_c$ replaced by  model updated via local training. 
For the server-side procedure, as in the case of Adapt Norm (\cref{sec:fme-to-fo-norm}), the key idea is to plug-in our FME procedure, \cref{alg:dp-mean-estimation-adapt-server_adapt_tail}, in the  averaging step of the FedAvg algorithm, yielding \cref{alg:dp-mean-estimation-adapt-server_adapt_experiments_adapt_tail}.
Further rather than using $\log{d}$ interactive rounds to estimate the $k$-th tail norm, we use a \textit{stale estimates} computed from prior rounds as in our Adapt Norm approach. 
Consequently, we do not use the AboveThreshold mechanism (in \cref{alg:dp-mean-estimation-adapt-server_adapt_tail}), but rather account for composition for noise added for adaptivity.
This means we do not noise the threshold $\overline \gamma$, and add Gaussian (instead of Laplace) noise to error.
Further, $c_0$ is a constant which determines the fraction of error relative to DP error, that we want to tolerate is at most, say  $10\%$ as we use across all experiments, including these, in this paper.

The final proposed change yields a few variants of our method.
Instead of doubling the sketch size at every round, we update it as follows, $$C_{j+1} = \br{1+\eta \text{sign}(\tilde e_j - \tilde \gamma_j)}C_j$$
where we set $\eta=0.2<1$.
Setting $\eta=1$ recovers our doubling scheme; in general, this means that
instead of resetting to our initial (small) sketch size for every FME instance, we increase/decrease the sketch size based on if the error $\tilde e_j$ is smaller/larger than the target $\overline \gamma$.
While $\text{sign}(\tilde e_j - \tilde \gamma_j)$ is a convenient choice to measure the distance of error from the target threshold , we also tried a few other options. For instance, we can use the (absolute) error with exponential and linear updates, $C_{j+1} = \br{1+\eta (\tilde e_j - \tilde \gamma_j)}C_j$ and $C_{j+1} = C_j+\lfloor\eta (\tilde e_j - \tilde \gamma_j)\rfloor$, as well as use relative, as opposed to absolute error.

\begin{algorithm}[h]
\caption{Adapt Tail FL
}

\begin{algorithmic}[1]

\REQUIRE Sketch sizes $\ssize_1 = RPC_1$ and $\sssize = \tilde R \tilde P \tilde C$
noise multiplier $\sigma$, model dimension $d$,
a constant $c_0$,
rounds $K$, $\eta$
\FOR{$j=1$ to $K$}
\STATE Select $n$ random clients
and 
broadcast sketching operators $S_j$ and $\tilde S_j$ of sizes 
$\ssize_j = RPC_j$ and $\sssize$.
\STATE {\small $\nu_j=\! \text{SecAgg}(\{Q^{(c)}_{j}\}_{c=1}^n)+\cN\left(0,\frac{\sigma^2B^2}{0.9n} \bbI_{PC}\right)$}, \\ $\tilde \nu_j \!= \!
             \text{SecAgg}(\{\tilde Q^{(c)}_{j}\}_{c=1}^n)$
              {, where  $Q^{(c)}_{j} \!\!\gets\!\! [\mathsf{clip}_{B}(S^{(i)}_{j}(z^{(j)}_c))]_{i=1}^R$, 
$\tilde Q^{(c)}_j\!\! \gets [\mathsf{clip}_{B}(\tilde S^{(i)}_{j}(z^{(j)}_{c}))]_{i=1}^{\tilde R}$ } 
             
\STATE   \textbf{Unsketch DP mean}:$\omean_j = 
U_{j}(\nu_j)$
        \STATE \textbf{Second sketch}:
        $\bar e_j = \norm{\tilde S_j(\omean_j) - \tilde \nu_j}$
        \STATE $\overline \gamma = c_0 \frac{\sqrt{d}\sigma B}{\sqrt{0.9}n}+
\frac{2\sigma B}{\sqrt{0.1}n}
$
 \STATE 
 $\tilde e_j = \bar e_j + \cN(0, \sigma^2/0.1)$
 \STATE $C_{j+1} = \br{1+\eta \text{sign}(\tilde e_j - \overline \gamma)}C_j$
    \ENDFOR
\end{algorithmic}

\label{alg:dp-mean-estimation-adapt-server_adapt_experiments_adapt_tail}
\end{algorithm}

\paragraph{On AdaptTail's performance for federated optimization.} 
We found that the above mappings of \textbf{Adapt Tail}, did not do well in our experiments. In particular, we found that this method almost always led to a severe overestimate of the attainable compression rates, which correspondingly hurt the utility of the final model. Often, we found this deterioration was drastic, even beyond $\Delta=20\%$. 

Next, we discuss our observations and our hypotheses for why this method did not perform well in practice.
 First, recall that the logic behind our mapping from adaptive procedure for FME to FO is that to use it to solve intermediate FME problems arising from FedAvg. Thus, all our theoretical guarantees are for the FME problem. We find the the Adpat Tail procedure indeed finds very high compression rates, within small additional error (e.g., our chosen $10\%$) in FME. However, these compression rates are often too large to result in a reasonable accuracy for the learning task.

From this, and contrasted with the empirical success we observe with the Adapt Norm approach, we believe the most plausible issue is that the federated mean estimation error may not be the best proxy metric. Though we find very favorable compression rates with the Adapt Norm method by also building on this, this method does achieve worse compression-utility tradeoffs than the (potentially unattainable, due to computation requirements) genie.

\end{document}